\documentclass[hidelinks,onefignum,onetabnum]{siamart251104}

\usepackage{lipsum}
\usepackage{amsfonts}
\usepackage{graphicx}
\usepackage{epstopdf}
\usepackage{algorithmic}
\usepackage{xcolor}
\usepackage{cuted}
\usepackage{bm}
\usepackage{bbm}
\usepackage{soul}
\usepackage{enumitem}

\ifpdf
  \DeclareGraphicsExtensions{.eps,.pdf,.png,.jpg}
\else
  \DeclareGraphicsExtensions{.eps}
\fi

\newsiamremark{remark}{Remark}
\newsiamremark{hypothesis}{Hypothesis}
\crefname{hypothesis}{Hypothesis}{Hypotheses}
\newsiamthm{claim}{Claim}
\newsiamremark{fact}{Fact}
\crefname{fact}{Fact}{Facts}

\headers{Data-Driven Model Reduction Using WELDNet}{B. Dahal, J. Cheng, H. Liu, R. Lai, W. Liao}

\title{Data-Driven Model Reduction Using WELDNet: Windowed autoEncoders for Learning Dynamics\thanks{
Preprint
\funding{Biraj Dahal, Jiahui Cheng and Wenjing Liao are partially supported by National Science Foundation under the
NSF DMS 2145167 and the U.S. Department of Energy under the DOE SC0024348. Hao Liu was partially supported by HKRGC ECS 22302123, HKRGC GRF 12301925 and Guangdong and Hong Kong Universities ``1+1+1'' Joint Research Collaboration Scheme UICR0800008-24. Rongjie Lai's research is supported in party by NSF DMS-2401297. }}
}

\author{Biraj Dahal\thanks{School of Mathematics, Georgia Institute of Technology, Atlanta, GA, USA \email{bdahal6@gatech.edu}.}
\and Jiahui Cheng\thanks{Meta, USA.} 
\and Hao Liu\thanks{Department of Mathematics, Hong Kong Baptist University, Hong Kong, China.} 
\and Rongjie Lai\thanks{Department of Mathematics, Purdue University, West Lafayette, IN, USA.} 
\and Wenjing Liao\thanks{School of Mathematics, Georgia Institute of Technology, Atlanta, GA, USA.}}

\usepackage{amsopn}
\usepackage{mathrsfs}
\usepackage{custom}

\newcommand{\Lip}{\mathrm{Lip}}

\newcommand{\NN}{{\rm NN}}

\DeclareMathOperator*{\argmin}{arg\,min}

\ifpdf
\hypersetup{
  pdftitle={Data-Driven Model Reduction using WeldNet: Windowed autoEncoders for Learning Dynamics},
  pdfauthor={B. Dahal, J. Cheng, H. Liu, R. Lai, W. Liao}
}
\fi

\begin{document}

\maketitle

\begin{abstract}
Many problems in science and engineering involve time-dependent, high dimensional datasets arising from complex physical processes, which are costly to simulate. In this work, we propose \textbf{WeldNet}: \textbf{W}indowed \textbf{E}ncoders for \textbf{L}earning \textbf{D}ynamics, a data-driven nonlinear model reduction framework to build a low-dimensional surrogate model for complex  evolution systems. Given time-dependent training data, we split the time domain into multiple overlapping windows, within which nonlinear dimension reduction  is performed by auto-encoders to capture latent codes. Once a low-dimensional representation of the data is learned, a propagator network is trained to capture the evolution of the latent codes in each window, and a transcoder is trained to connect the latent codes between adjacent windows. 
The proposed windowed decomposition significantly simplifies propagator training by breaking long-horizon dynamics into multiple short, manageable segments, while the transcoders ensure consistency across windows.
In addition to the algorithmic framework, we develop a mathematical theory establishing the representation power of WeldNet under the manifold hypothesis, justifying the success of nonlinear model reduction via deep autoencoder-based architectures. 
Our numerical experiments on various differential equations indicate that WeldNet can capture nonlinear latent  structures and their underlying dynamics, outperforming both traditional projection-based approaches and recently developed nonlinear model reduction methods.
\end{abstract}

\begin{keywords}
model reduction, manifold learning, scientific machine learning, dynamical systems
\end{keywords}
\begin{MSCcodes}
68Q25, 68R10, 68U05
\end{MSCcodes}

\section{Introduction}
\label{sec:introduction}
Many real-world applications in science and engineering involve large-scale, complex, and costly data simulations or inversions of physical processes. However, the high-dimensional nature of these models often creates overwhelming demands on computational resources. Model reduction plays a crucial role in addressing this challenge, which helps to reduce the data dimension and problem size \cite{hesthaven2016certified,benner2017model,benner2015survey,volkwein2011model,rozza2008reduced,brunton2022data,quarteroni2015reduced}. 

Linear model reduction techniques for differential equations and dynamical systems have been well-established in  literature \cite{benner2015survey,benner2017model,volkwein2011model}. Classical projection-based model reduction approaches have demonstrated great success when
the underlying model is linear and low-dimensional. 
Representative projection-based model reduction methods include Proper Orthogonal Decomposition (POD) \cite{berkooz1993proper,pinnau2008model,volkwein2011model},  reduced-basis techniques \cite{prud2002reliable,rozza2008reduced}, the Principal Component Analysis (PCA) approach \cite{moore1981principal}, rational interpolation \cite{gugercin2008h_2,baur2011interpolatory}, Galerkin projection \cite{holmes2012turbulence,rowley2004model}, etc. Most projection-based model reduction methods rely on projecting high-dimensional models onto a low-dimensional linear subspace.

In real-world applications, many objects exhibit low-dimensional nonlinear structures \cite{tenenbaum2000global,roweis2000nonlinear}. Simple transformations, such as translations or rotations, place these objects on low-dimensional nonlinear manifolds, which cannot be efficiently captured by linear subspaces.
 Meanwhile, many physical processes are inherently nonlinear, including fluid dynamics, nonlinear optical phenomena, and shallow water wave propagation.
 When linear model reduction methods are applied to reduce evolutionary equations, their optimal performance is quantified by the Kolmogorov $n$-width \cite{pinkus2012n}. 
 For diffusion-dominated problems, the Kolmogorov $n$-width decays rapidly as the subspace dimension increases, enabling linear methods to achieve significant success with well-established justification \cite{bachmayr2017kolmogorov,binev2011convergence}. Unfortunately, many differential equations, particularly advection-dominated problems, display a slow decay of the Kolmogorov $n$-width. In these cases, linear model reduction methods require a sufficiently large reduced dimensionality to achieve acceptable accuracy \cite{ohlberger2015reduced}, which demonstrates the limitations of linear model reduction methods in handling  nonlinear structures.

The limitations of linear methods drive the development of nonlinear model reduction techniques, which aim to leverage low-dimensional nonlinear structures in data and models \cite{coifman2008diffusion,crosskey2017atlas,ye2024nonlinear}.   As deep learning has gained popularity in recent years, deep learning methods have been increasingly applied to address nonlinear model reduction \cite{otto2019linearly,wang2018model,lee2020model,fresca2021comprehensive,gonzalez2018deep,fresca2021comprehensive,si2024latent,xiao2024ld,zhang2025coefficient}. 
%
Many works  utilize  autoencoder \cite{belkin2003laplacian} for dimension reduction, and then learn the unknown physical process on latent variables by a neural network \cite{hesthaven2018non,o2022derivative,wang2019non,otto2019linearly,wang2018model,lee2020model,fresca2021comprehensive,gonzalez2018deep,fresca2021comprehensive,franco2023deep,liu2025generalization}. However, when the solution manifold of evolutionary equations has evolved for a long time, a global  dimension reduction may not well capture the effective latent parameters at all time points. 
 
In this work, we propose Windowed Encoders for Learning Dynamics with Neural Networks (WeldNet) for nonlinear model reduction. WeldNet uses encoder networks to perform nonlinear dimension reduction and learn the evolution operation in the encoded domain. The time domain is divided into several sequentially overlapping subintervals called ``windows''. In each window, we train an autoencoder to learn the latent codes and a propagator to evolve the latent codes in time. Transcoder networks are trained to connect two adjacent windows on their overlap, allowing an initial condition to be encoded, evolved to the terminal time, and then decoded. After training, WeldNet can be used to approximate the trajectory at any time given an initial condition. Figure \ref{fig:welddiagram} diagrams our WeldNet method of trajectory learning for time-dependent data.

\begin{figure}
    \centering
    \includegraphics[width=0.9\linewidth]{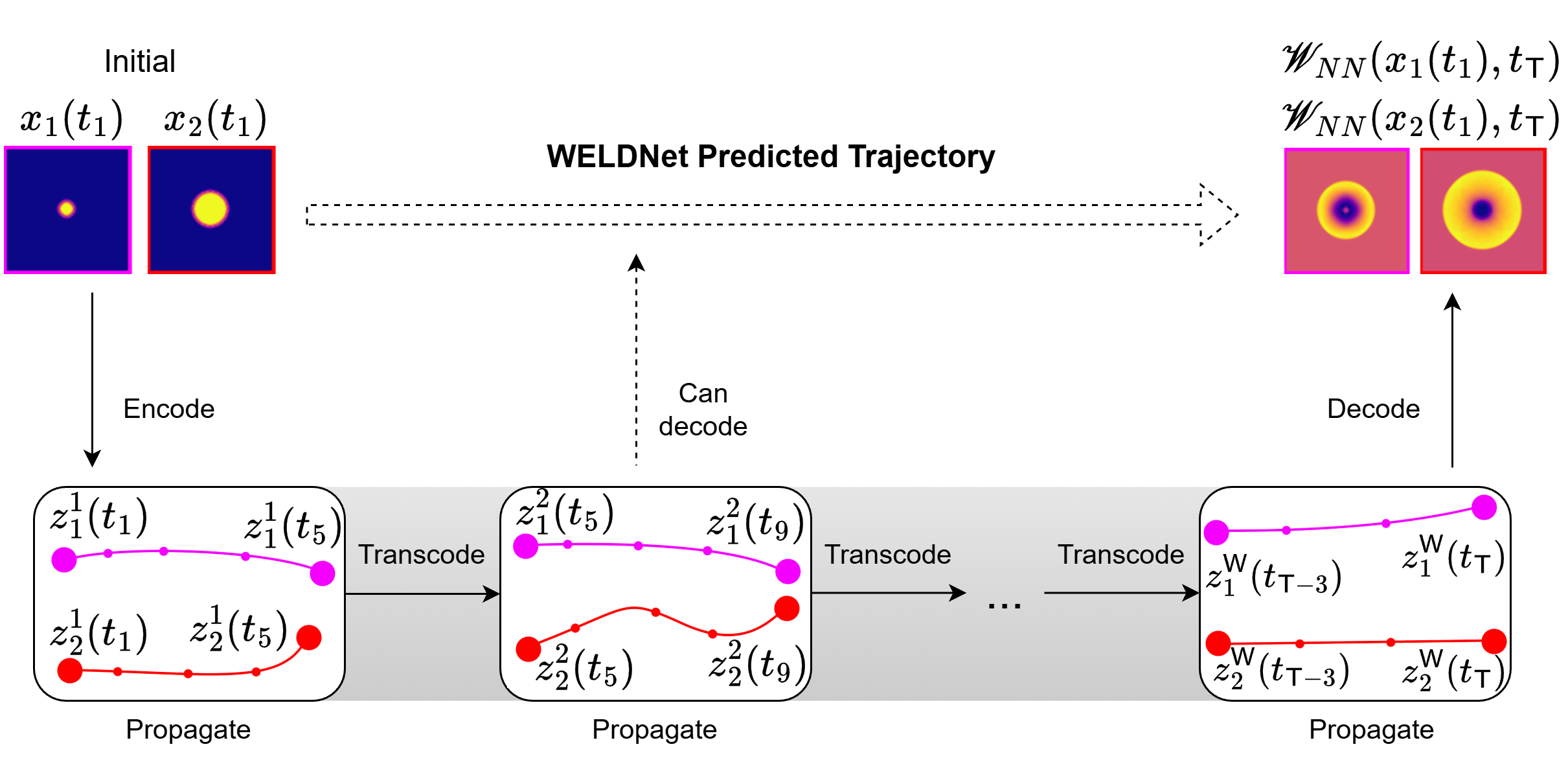}
   \caption{WeldNet: Initial conditions are encoded, propagated within windows, transcoded between windows, and decoded. In this example, there are $\sfW$ windows and $\sfT$ time steps. $z^i_j(t_k)$ denotes the latent space representation of $x_j(t_k)$ according to window $i$.}
     \label{fig:welddiagram}
\end{figure}
     
\begin{figure}
    \centering
    \includegraphics[width=0.7\linewidth]{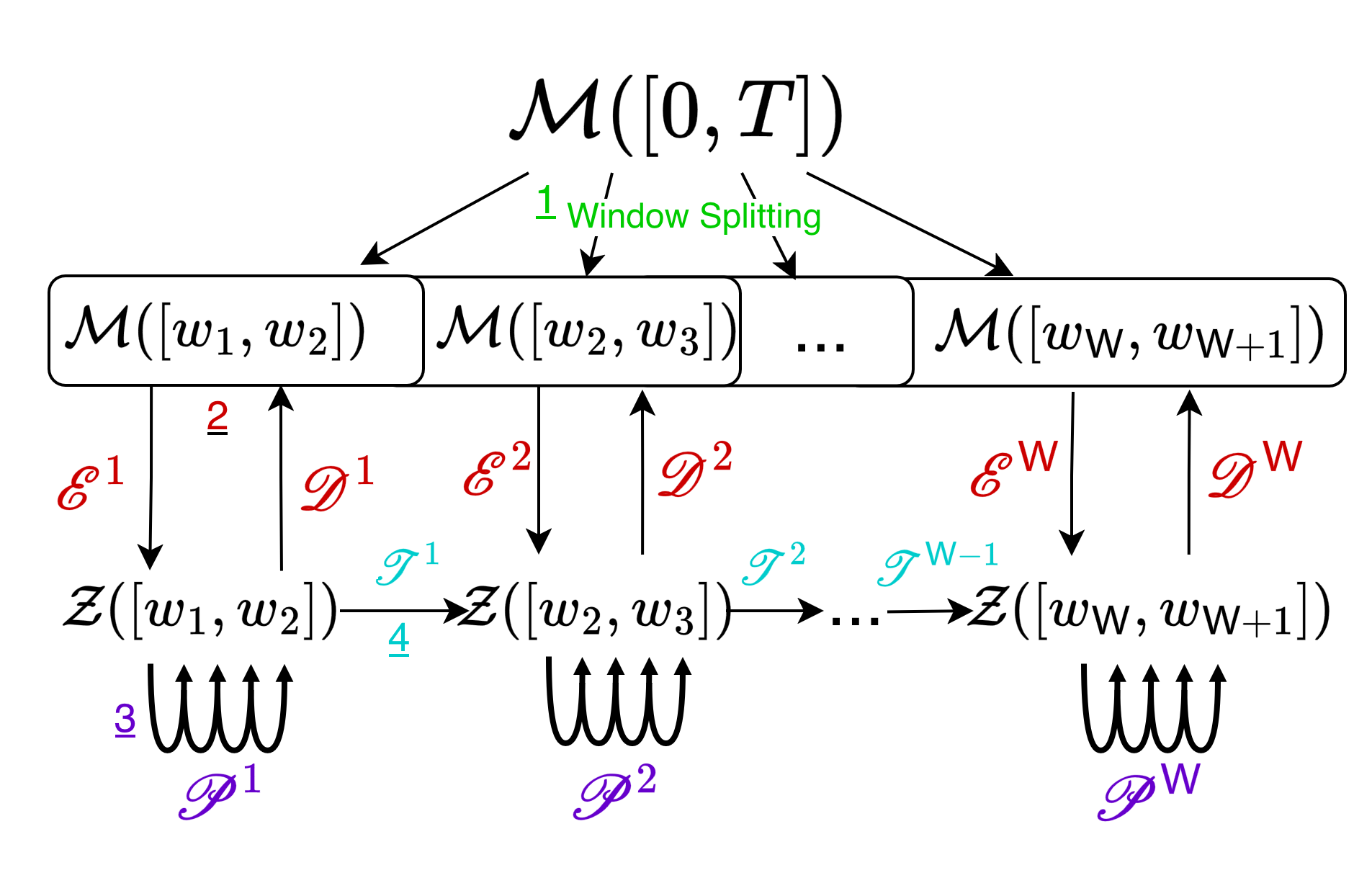}
 \caption{Components of WeldNet: 1) Window splitting, 2) Autoencoder, 3) Propagator, 4) Transcoder.}
 \label{fig:weldnotation}
\end{figure}

The separation between dimension reduction and trajectory learning allows the exploitation of existing nonlinear low-dimensional evolutionary structures in the data. 
Figure \ref{fig:weldnotation} indicates the notation for each component of WeldNet. We consider the trajectory manifold of an evolutionary  process, denoted by  $\cM([0, T])$, which collects all evolution trajectories in the time interval $[0,T]$. We split this time interval into $\sfW$ windows. For each window $i \in \{1, 2, \dots, \sfW\}$, $\sE^i$ and $\sD^i$ denote the encoder and decoder networks, $\sP^i$ denotes the propagator network for the evolution of latent codes, and $\sT_i$ (for $i < \sfW$) denotes the transcoder network from the $i$th window to the $(i+1)$th window. 

We validate WeldNet on a few examples, including the Burgers' equation, transport equation,  Korteweg–De Vries (KdV) equation and 2D shallow-water equation. 
Furthermore, this paper provides a mathematical theory on the representation power of WeldNet to justify the success of nonlinear model reduction by deep learning under manifold hypothesis. 
Suppose the evolution operator satisfies a regularity assumption and all trajectories of this evolutionary process lie on a low-dimensional manifold. We prove that, WeldNet can approximate the trajectories of this evolutionary process up to arbitrary accuracy, if the encoder, decoder and propagator network architectures are properly set up. Our theory justifies the representation power of WeldNet for a large class of evolutionary processes, and provides a theoretical foundation for nonlinear model reduction using auto-encoder-based methods.


The main contributions of this paper are: 
\begin{enumerate}
    \item \textbf{WeldNet framework.} We propose a windowed autoencoder--propagator architecture for nonlinear model reduction. The method performs dimension reduction and latent-space trajectory learning by dividing the time domain into overlapping windows connected via transcoder networks.
    
    \item \textbf{Efficient trajectory prediction.} WeldNet enables end-to-end approximation of system trajectories by encoding an initial condition, evolving the latent representations across windows, and decoding the result at any desired time.

    \item \textbf{Theoretical foundation.} Under the manifold hypothesis, we establish a representation theory to show that WeldNet can express  evolutionary processes with low-dimensional structures to arbitrarily high accuracy when the encoder, decoder, and propagator networks are appropriately designed.
    
    \item \textbf{Empirical validation.} The effectiveness of WeldNet is demonstrated on several nonlinear PDEs, including the Burgers', transport, KdV, and 2D shallow-water equations. 
\end{enumerate}









\textbf{Organization.}
We first introduce our WeldNet model in Section \ref{sec:method}. A representation theory of WeldNet is presented in Section \ref{sec:theory}, and comprehensive numerical experiments are given in Section \ref{sec:numerical}. Finally, we conclude in Section \ref{sec:conclusion}.

\textbf{Notations.} For any $n \in \mathbb{N}$, we denote $[n] = \{1, \dots, n\}$. For $\bfx\in \bR^n$, we denote $\|\bfx\|_{\bR^n}=\|\bfx\|_2$ and $\| \bfx \|_\infty = \max_{i \in [n]} x_i$ where $x_i$ is the $i$th component of $\bfx$. For any $\bfx \in \bR^n$, we denote the ReLU function $\sigma: \bR^n \rightarrow \bR^n$ by $\sigma(\bfx) = (\max(x_i, 0))_{i=1}^n$. For any function $f : A \rightarrow B$ defined on sets $A$, $B$, we denote $\|f\|_{L^{\infty}(A;\, B)}= \sup_{\bfx\in A}\|f(\bfx)\|_B$ where $\|\cdot\|_B$ is a norm on $B$. Given a finite set $A$, we denote the cardinality of $A$ by $|A|$.
We use $\bigcirc_{k=1}^{K } f_{k}$ to denote the composition $f_K\circ f_{K-1} \circ \cdots \circ f_1$. In particular, $\bigcirc_{k=1}^{K } f$  denotes the composition of $f$ for $K$ times.

\section{WeldNet for Model Reduction}
\label{sec:method}
In this section, we present our WeldNet model, which operates on time-dependent trajectory data collected from an evolutionary  process.
In science and engineering applications, the solution trajectory in an evolutionary process often depends on few parameters in the initial condition or in the evolution equation \cite{zeng2024autoencoders,lee2020model}.
When the trajectories of this evolutionary  process lie on a low-dimensional manifold, WeldNet parameterizes the trajectory manifold by a low-dimensional latent code and builds a surrogate evolutionary model in the latent space.

We consider an evolutionary  process in $\bR^D$, whose initial states can be (locally) parameterized by a small number of parameters. Suppose  the initial states are supported on a set $\cM(0) \subseteq \bR^D$ We are interested in learning an evolutionary process driven by an unknown continuous \textbf{time-evolution operator} $\sF : \cM(0) \times [0, T] \rightarrow \bR^D$.
For simplicity, we denote 
$$\bfx(t) = \sF(\bfx(0),t).$$

We denote $\cM(t) = \sF(\cM(0), t) = \{\sF(\bfx(0), t):\ \bfx(0) \in \cM(0)\}$ as the initial data evolved by $t$ time units, and $ \cM([a, b]) = \{\cM(t): \ t\in[a,b]\}$ denotes the collection of state manifold in the time interval $[a,b]$. We call the set $\cM(0)$ the initial manifold and the set $\cM([0, T])$ the \textbf{trajectory manifold}. We will assume that $\cM([0, T])$ is a $d$-dimensional Riemannian manifold embedded in $\bR^D$. For convenience, we extend the domain of the time-evolution operator $\sF$ to cover all of $\cM([0, T])$ and all relevant times, such that for all $\bfx(s) = \sF(\bfx(0), s) \in \cM(s)$, we assume that the dynamic satisfies  $\sF(\bfx(s), t) = \sF(\bfx(0), s+t) = \bfx(s+t)$ for $t \in [0, T-s]$.

In addition, we extend the domain of $\sF$ to all of $\cM([0, T]) \times [0, T]$ in the following way: if $\bfx \in \cM(s)$ then define $\sF(\bfx, t) = \sF(\bfx, \max(T-s, t))$. Note that this extension preserves the Lipschitz constant.

\subsection{Data Collection}
\label{sec:setup}
In applications, one can measure the trajectory of the above evolutionary process and collect trajectory data for multiple initial states. Suppose $N$ initial states are randomly sampled from a probability measure $\rho(0)$ supported  on $\cM(0)$: 
\[ \{\bfx_1(0), \dots, \bfx_N(0)\} \iid \rho(0). \]
Let $0 = t_1 < t_2 < \dots < t_\sfT = T$ be a time grid, denoted $\bT = \{t_1, \dots, t_\sfT\}$. {In this paper, we consider an equally spaced time grid with time spacing $\Delta t$.}  After measuring the trajectory of these $N$ initial states at this time grid,  we can collect the dataset
\begin{align*} 
&\{\bfx_n\}_{n=1}^N := \{(\bfx_n(t_k))_{k=1}^\sfT\}_{n=1}^N \in \bR^{N \times \sfT \times D}, \text{ where } \bfx_n(t_k) = \sF(\bfx_n(0), t_k). 
\end{align*}
Given this dataset, our goal is to construct a low-dimensional surrogate model $\sW_{model} : \cM(0) \times \bT \rightarrow \bR^D$  such that 
\[ \sW_{model}(\bfx(0), t_k) \approx \sF(\bfx(0), t_k), \ \forall \bfx(0) \in \cM(0), \, t_k\in \bT. \]
After training, one can predict the solution trajectory for a new initial condition sampled from $\rho(0)$ by evolving the surrogate model in the low-dimensional latent space.

\subsection{WeldNet Training}
\label{sec:weldnettraining}
Training a WeldNet model involves four stages, shown in Figure \ref{fig:weldnotation}: 1) Window splitting, 2) Autoencoder training, 3) Propagator training, 4) Transcoder training. We provide details of each stages below. We also provide an overview of notations is provided in Table \ref{table:thmnotations}.

\setlength{\tabcolsep}{4pt} 
\begin{table}[]
\centering
\begin{tabular}{|l|l|}
\hline
\textbf{Notation} & \textbf{Explanation}                         \\ \hline
$[w_i, w_{i+1}]$  & Time range in window $i$                     \\ \hline
$\bT$           &  Time grid: $ \{t_1, \dots, t_\sfT\}$ \\
\hline
$\bT^i$ & $ \bT \cap [w_i, w_{i+1}]$
\\
\hline
$\bm{T}_i$        & Size of $\bT \cap (w_i, w_{i+1}]$            \\ \hline
$\sF$             & Evolution operator                           \\ \hline
  
$\sE^i_*/\sD^i_*$ & Window $i$'s oracle encoder/decoder         \\ \hline
$\sP^i_*$ & Window $i$'s oracle propagator $=\sE^i_* \circ \sF(\cdot, \Delta t) \circ \sD^i_*$         \\ \hline
$\sT^i_*$ & Window $i$'s oracle transcoder $=\sE^{i+1}_* \circ \sE^i_*$         \\ \hline
$\sE^i_\NN/\sD^i_\NN/\sP^i_\NN/\sT^i_\NN$ & Window $i$'s oracle encoder/decoder/propagator/transcoder       \\ \hline
$\cM([a, b])$     & Trajectory manifold from time $a$ to 
$b$     \\ \hline
$\cZ([a, b])$     & Latent space from time $a$ to $b$            \\ \hline
\end{tabular}
\caption{Notation overview.}
\label{table:thmnotations}
\end{table}

{\bf 1)  Window splitting.} A WeldNet model with $\sfW$ windows divides the time domain into $\sfW$ sequentially overlapping windows, with the $i$-th window denoted by $[w_i, w_{i+1}]$ and $\bigcup_{i=1}^{\sfW} [w_i, w_{i+1}] = [0, T]$. Suppose $\{w_i\}_{i=1}^{\sfW + 1}\subset \{t_k\}_{k=1}^{\sfT}$, i.e., the endpoints of each window are on the time grid. 

{\bf 2) Autoencoder and Propagator training.}  WeldNet first trains $\sfW$ autoencoders and $\sfW$ propagators, where the $i$th autoencoder aims to learn a representation of the manifold $\cM([w_i, w_{i+1}])$, and the $i$th propagator will predict the displacement required to evolve the latent code from one time discretization point to the subsequent one in latent manifold. 

Fixing a latent space dimension $d \in \mathbb{N}$, We denote the domain of the decoders and propagators, also known as the \textbf{latent space}, by $\cZ([0, T]) = [0,1]^d \times [0, T]$. We also use the notation $\cZ(t) = [0,1]^d \times \{t\}$ and notation $\cZ([a,b]) = [0,1]^d \times [a,b]$. Note that we always incorporate time in our latent codes to track the time information. 

The autoencoder reconstruction loss for the $i$th window is
\begin{equation}
\textstyle
\textbf{L}^i_\text{ae}(\sE^i, \sD^i) = \frac{1}{N|\bT \cap [w_i, w_{i+1}]|} \sum_{n=1}^N \sum_{t_k \in \bT \cap [w_i, w_{i+1}]} \left\| \sD^i(\sE^i(\bfx_n(t_k))) - \bfx_n(t_k) \right \|_{\bR^D}^2,
\label{eq:optae}
\end{equation}
where  $(\sE^i, \sD^i)$ is an encoder-decoder pair in the $i$th window. The propagator loss for the $i$th window (denoted $\textstyle\textbf{L}^i_\text{prop}(\sE^i, \sP^i)$) is 
\begin{equation}
\frac{1}{N|\bT \cap [w_i, w_{i+1})|} \sum_{n=1}^N \sum_{t_k \in \bT \cap [w_i, w_{i+1})}\left|\sP^i (\sE^i\left(\bfx_n(t_k)\right)) - \sE^i\left(\bfx_n(t_{k+1})\right) \right|^2_{\bR^{d+1}},
\label{eq:optprop}
\end{equation}
where $\sP^i$ denotes an one-step propagator for the latent code in the $i$th window.

We train the encoder, decoder, and propagator together using the following objective function:
\begin{equation}
\textstyle
(\sE^i_\NN, \sD^i_\NN, \sP^i_\NN) \in \underset{\sE \in \mathcal{F}^{\sE}_{\rm NN}, \sD \in \mathcal{F}^{\sD}_{\rm NN}, \sP \in \mathcal{F}^{\sP}_{\rm NN}}{\argmin} \textbf{L}^i_\text{ae}(\sE^i, \sD^i) + \lambda \textbf{L}^i_\text{prop}(\sE^i, \sP^i),
\label{eq:aeproploss}
\end{equation}
 where $\mathcal{F}^{\sD}_{\rm NN}$, $\mathcal{F}^{\sE}_{\rm NN}$, and $\cF^\sP_{\NN}$ are network classes for decoder, encoder, and propagator, respectively, and $\lambda > 0$ is a hyperparameter. In this work, we use $\lambda = 0.1$ for all experiments. Figure \ref{fig:step2loss} diagrams the autoencoder and propagator losses (the loss is the mean squared error between applying the functions in the red path and applying the functions in the blue path).

\begin{figure}
  \begin{minipage}[t]{0.28\textwidth}
    \centering
    \fbox{\includegraphics[width=0.7\textwidth]{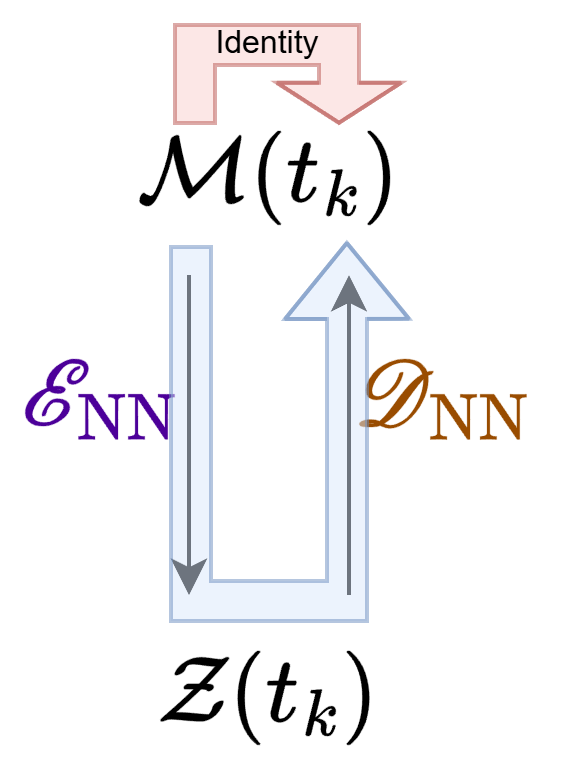}}
  \end{minipage}
  \hfill
  \begin{minipage}[t]{0.68\textwidth}
    \centering
    \fbox{\includegraphics[width=0.94\textwidth]{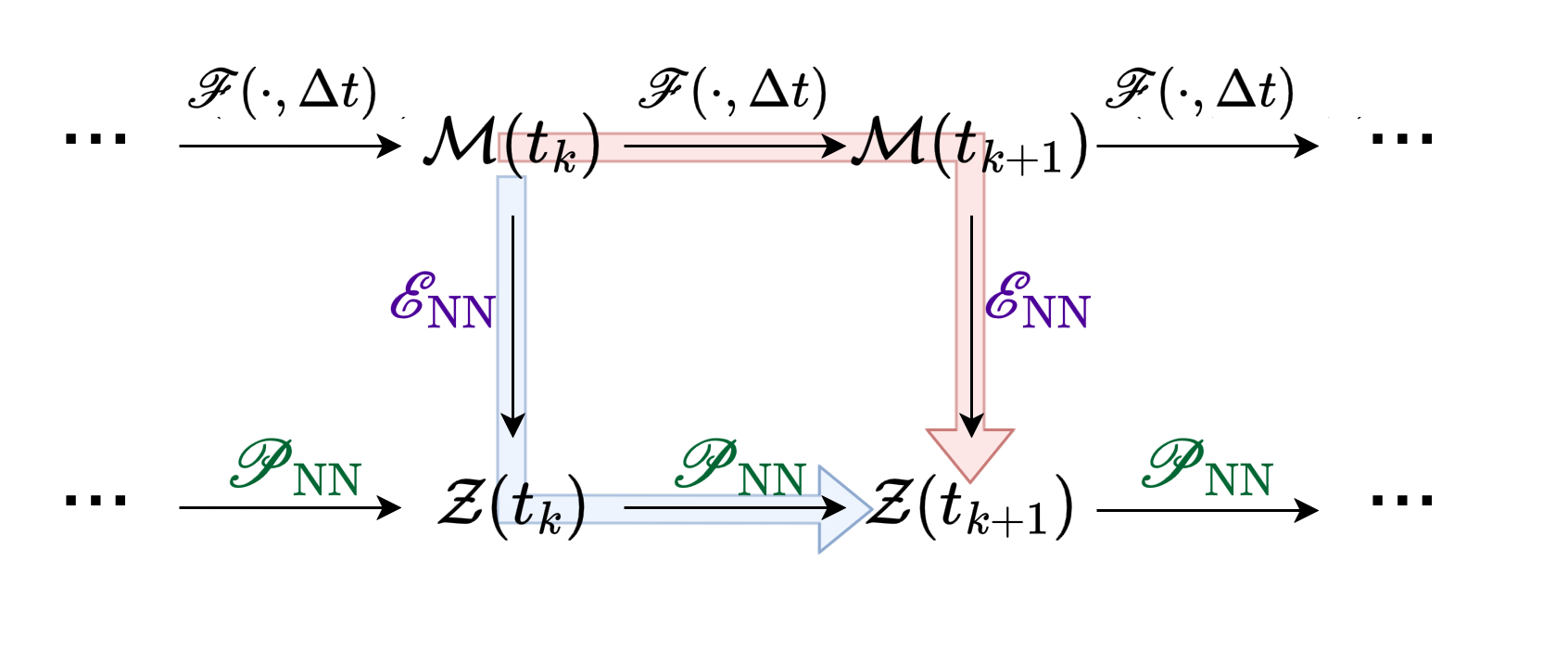}}
  \end{minipage}
  \caption{Diagram of the autoencoder loss (left) and propagator loss (right), where the loss is the MSE between the output of the red arrow and the blue arrow.}
  \label{fig:step2loss}
\end{figure}


{\bf 3) Propagator finetuning.} After we train the autoencoder and propagator together, we then finetune the propagator in order to reduce the accumulation of error that occurs when applying $\sP^i$ to propagate a latent code over multiple time steps. Specifically, we freeze the encoder and train the propagator with the objective (where $k(w_i)$ is the time index of $w_i$, i.e. the first time step in window $i$): 
 
\begin{equation}
\sP^i_\NN = \underset{\sP^i \in \mathcal{F}^{\sP}_{\rm NN}}{\argmin} \frac{1}{N\bm{T}_i} \sum_{n=1}^N \sum_{s=1}^{\bm{T}_i}\left\|\left(\bigcirc_{k=1}^s \sP^i\right) (\sE^i\left(\bfx_n(t_{k(w_i)})\right)) - \sE^i\left(\bfx_n(t_{k(w_i) + s})\right) \right\|^2_{\bR^{d+1}} ,
\label{eq:propfinetune}
\end{equation}
where we denote $\bm{T}_i = |\bT \cap(w_i, w_{i+1}]|$.

The loss for the finetuning of propagator network is illustrated in the Figure \ref{fig:propaccum}.

\begin{figure}[h!]
    \centering
    \includegraphics[width=0.91\textwidth]{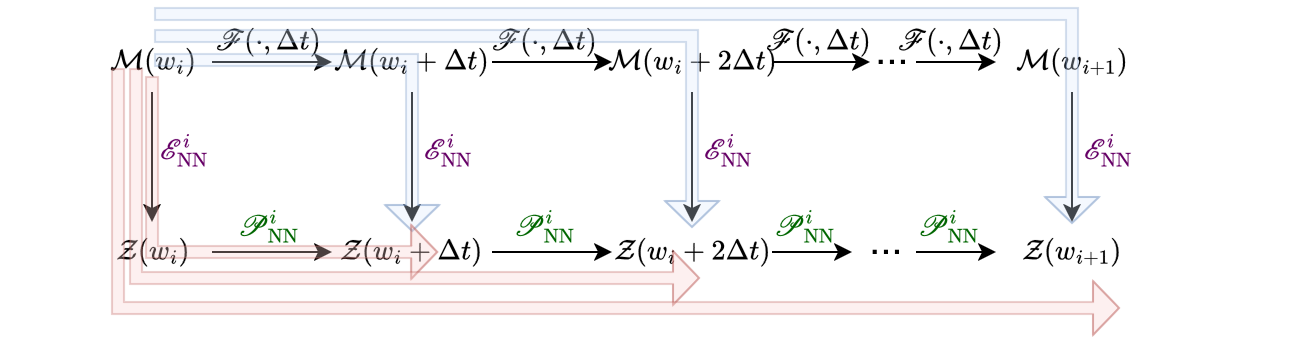}
    \caption{Diagram for propagator accumulation loss. The loss is the average of the mean squared error between the blue and red arrows that point to the same symbol.}
    \label{fig:propaccum}
\end{figure}

{\bf 4) Transcoder training.} We have trained autoencoders and propagators on each time window separately. In order to connect the windows, we train transcoder networks on the overlap between the windows. The goal of the $i$th transcoder is to connect the codes at the end of window $i$ with the codes at the beginning of window $i+1$. To train the transcoder, we will propagate codes from the beginning of window $i$ to the end. Specifically, we use the objective $\sT^i_\NN =$
\begin{equation}
 \argmin_{\sT^i \in \cF^\sT_{\NN}} \frac{1}{N} \sum_{n=1}^N \left\| \sT^i \left(\left(\bigcirc_{s=1}^{|\bT \cap (w_i, w_{i+1}]|} \sP_\NN^i\right)(\sE^i_{\NN}(\bfx_n(w_i)))\right) - \sE^{i+1}_{\NN}\left(\bfx_n(w_{i+1})\right)\right\|_{\bR^{d+1}}^2
\label{eq:opttrans}
\end{equation}
where $\cF^\sT_{\NN}$ denotes the network class for transcoders. Figure \ref{fig:transcoderloss} diagrams the loss function for the transcoder.

\begin{figure}
    \centering
    \includegraphics[width=0.755\textwidth]{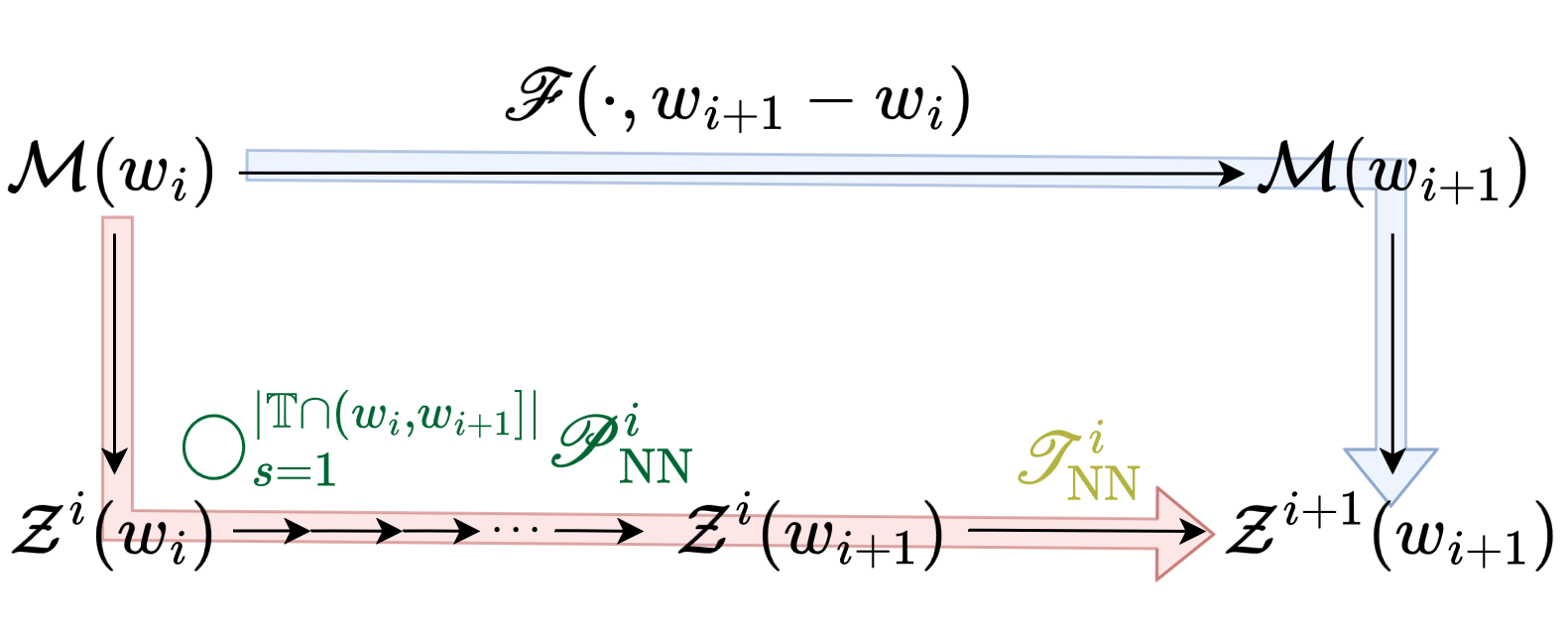}
  \caption{Diagram of the transcoder loss, where the loss is the MSE between the output of the red arrow and the blue arrow.}
  \label{fig:transcoderloss}
\end{figure}

\subsection{WeldNet Inference}
After WeldNet is trained, one can predict the trajectory given a new initial state $\bfx(0)$. The inference involves encoding the initial state to the latent code, evolving the latent code in time, and decoding at the final time.  This inference procedure can be represented by our WeldNet model, denoted by $\sW_{\NN}$, such that, for any $\bfx(0) \in \cM(0)$ and $t_k \in \bT \cap (w_i, w_{i+1}]$, WeldNet gives rise to (where $k(w_i)$ is the time grid index of the time $w_i$)
\begin{align}
\begin{split}
  \sW_{\NN}&(\bfx(0), t_k) := \\
  &\sD^i_{\NN} \circ \bigcirc_{\ell=1}^{k - k(w_i)} \sP_\NN^i 
  \circ \bigcirc_{j=1}^{i-1} (\sT^j_{\NN} \circ \bigcirc_{\ell=1}^{|\bT \cap (w_j, w_{j+1}]|} \sP^j_\NN)
  \circ  \sE^1_{\NN}(\bfx(0)). \label{eq:weldmultiwindow}
\end{split}
\end{align}

Figure \ref{fig:welddiagram} shows the process of encoding, propagating, and decoding for inference with WeldNet. 

\section{Approximation Theory}
\label{sec:theory}
In this section, we prove the approximation ability of WeldNet for a large class of evolutionary operators. The low-dimensional  structure in the trajectory manifold results in the existence of oracle maps for charts, time-evolution operator, and transition maps which can be approximated by the autoencoder, propagator, and transcoder networks respectively.

\subsection{Preliminaries}
\subsubsection{Neural Networks}
In this work, we consider ReLU networks as follows: 

\begin{definition}[Feedforward Neural Network (FNN)]
Let $L \in \mathbb{N}$ and suppose $\mathfrak{W}_1, \dots, \mathfrak{W}_{L+1}$ are weight matrices and $\mathfrak{b}_1, \dots, \mathfrak{b}_{L+1}$ are bias vectors with $\mathfrak{W}_\ell \in \bR^{d_{\ell - 1} \times d_\ell}$ and $\mathfrak{b}_\ell \in \bR^{d_\ell}$ for all $\ell \in [L+1]$. Let $\theta = (\mathfrak{W}_1, \mathfrak{b}_1, \dots, \mathfrak{W}_{L+1}, \mathfrak{b}_{L+1})$. We define a \textbf{feedforward neural network} (\textbf{FNN}) with weights $\theta$, denoted $f_{\NN}$, as a function of the form:
$$f_{\NN}(x) = \mathfrak{W}_{L+1} \sigma(\mathfrak{W}_L \sigma (\cdots \mathfrak{W}_1 x + \mathfrak{b}_1) \cdots + \mathfrak{b}_L) + \mathfrak{b}_{L+1}.$$
We say the \textbf{depth} and \textbf{width} of $f_{\NN}$ are $L$ and $W=\max_\ell d_\ell$ respectively.
\end{definition}
A FNN with depth $L$ and width $W$ is therefore a composition of $L+1$ functions, all but the last one of which involves an affine transform with output dimension at most $W$. With the (maximum) depth and width specified, we can define a function class of feedforward neural networks.
\begin{definition}[FNN Class]
Let $d_{in}, d_{out}, L, W \in \mathbb{N}$. The class of feedforward neural networks (\textbf{FNN class}) with $d_{in}$ inputs, $d_{out}$ outputs,  depth at most $L$, and  width at most $W$ is denoted $\cF_{\NN}(d_{in}, d_{out}, L, W)$. 
\end{definition}

\subsubsection{Manifolds}
\label{appmanifold}

\begin{definition}[Lipschitz Function]
Let $A, B$ be metric spaces with metrics $d_A$ and $d_B$ respectively, and let $L > 0$. A function $f: A \rightarrow B$ is called $L$-\textbf{Lipschitz} if 
\[
 \sup_{x \neq y \in A} \frac{\| f(x) - f(y) \|_{d_B}}{\|x - y \|_{d_A}} \leq L,\]
and we define the \textbf{Lipschitz constant} of $f$, denoted $\Lip_A(f)$, as the smallest $L$ for which $f: A \rightarrow B$ is $L$-Lipschitz.
\end{definition}

\begin{definition}[Manifold]
A $d$-dimensional \textbf{manifold} $\cM$ is a topological space endowed with a collection of \textbf{charts} $(\phi_i: U_i \subseteq \cM \rightarrow \tilde{U}_i \subseteq \bR^d)_{i \in \mathcal{I}}$ (for some index set $\mathcal{I}$) such that for all $i \in \mathcal{I}$, $\phi_i$ is a homeomorphism between the open sets $U_i$ and $\tilde{U}_i$, and we have $\bigcup_{i\in\mathcal{I}} U_i = \cM$. We say $\cM$ is a \textbf{smooth} manifold if for all charts $\phi_i$ and $\phi_j$ on $\cM$, the function 
\[\phi_j \circ \phi_i^{-1} : \tilde{U}_i \cap \tilde{U}_j \subseteq \bR^{d} \rightarrow \tilde{U}_i \cap \tilde{U}_j \subseteq \bR^{d} \]
is smooth in the usual Euclidean sense. Finally, for any $f: \cM \rightarrow \bR^k$, we say that $f$ has \textbf{injective derivative} if for each chart $\phi_i: U_i \rightarrow \tilde{U}_i$, we have the function $f \circ \phi_i^{-1}$ is differentiable  in the Euclidean sense, and its derivative is nonzero on its domain. 
\end{definition}

\begin{definition}[Embedded Riemannian Manifold]
We say that a $d$-manifold $\cM$ is a \textbf{Riemannian manifold embedded} in $\bR^D$ if $\cM \subseteq \bR^D$ and the inclusion function $\iota: \cM \rightarrow \bR^D$ is a smooth homeomorphism onto its image with injective derivative, and $\cM$ is a Riemannian manifold with respect to the induced metric from $\bR^D$. 
\end{definition}

More details on Riemannian manifolds can be found in standard texts on differential geometry \cite{lee2018introduction}. Next, we introduce a regularity/curvature parameter for embedded submanifolds of Euclidean space known as reach. Informally speaking, the reach of a manifold describes the curvature of the embedded manifold in the ambient Euclidean space. It is the radius of the largest ``ball'' (interior of a hypersphere) you can roll around the manifold without crossing it.

\begin{definition}[Reach \cite{federer1959curvature}]
Let $\cM \subseteq \bR^D$ be an embedded Riemannian manifold. The reach of $\cM$, denoted $\bm{\tau}(\cM)$, is
\begin{align*}
\inf&\{ r > 0 : \exists x \neq y \in \cM, v \in \bR^D \\
&\text{ such that } r = \|x-v\|=\|y-v\|=\text{dist}(v, \cM) \} .
\end{align*}
\end{definition}
Linear subspaces have reach $\infty$, and the reach of a hypersphere equals to its radius. It is well-known that compact embedded manifolds have positive reach \cite{thale200850}. The larger the reach, the easier the manifold is to be represented with neural networks, and that will affect the bounds in neural network approximation theory \cite{chen2019efficient,schonsheck2019chart,schonsheck2022semi,liu2024deep}.

\subsection{Main Theorems}
\label{sec:theorems}

According to Section \ref{sec:setup}, the time-evolution operator $\sF$ is defined on the set $\cM([0, T]) \times [0, T]$. We will assume that $\sF$ is Lipschitz for all times:

\begin{assumption}
\label{assum:evolve}
Assume $\sF$ is Lipschitz, i.e. ${\rm Lip}(\sF)  := $
\[
  \sup_{\substack{t \in [0, T], \\ s \in [0, T-t]}}\Lip_{\cM(t)}(\sF(\cdot, s))
= \sup_{\substack{t \in [0, T], \\ s \in [0, T-t]}}  \sup_{\substack{\bfx(t) \neq \bfy(t) \\ \in \cM(t)}} \frac{\|\sF(\bfx(t), s) - \sF(\bfy(t), s)\|}{\|\bfx(t) - \bfy(t)\|}< \infty.
\]
\end{assumption}
In this work, we will consider evolutionary processes whose initial conditions are sampled from an embedded $d$-dimensional manifold. This means that each initial condition can be described (locally) by $d$ parameters. More specifically, we will assume that the collection of trajectories at various segmented times form a $(d+1)$-dimensional manifold, with $d$-dimensions in parameter space and one dimension given by time. For convenience, we will use the notation $\cZ(t) = [0, 1]^d \times \{t\}$ and $\cZ([a, b]) = [0, 1]^d \times [a, b]$.

\begin{assumption}[Segmented Manifold]
\label{assum:mfd}
Suppose there exist $0 = s_1 < s_2 < \cdots < s_{\sfS+1} = T$ such that for all $i \in [\sfS]$, the subset $\cM([s_i, s_{i+1}])$ is a compact $(d+1)$-dimensional Riemannian manifold embedded in $\bR^D$ with reach $\bm{\tau}(\cM([s_i, s_{i+1}])) > 0$. 
In addition, for all $i \in [\sfS]$ we assume there is a smooth function $\sD_*^i$ that maps from a Euclidean space to the manifold with smooth inverse $\sE_*^i = (\sD_*^i)^{-1}$:

\[\sE^i_*: \cM([s_i, s_{i+1}]) \rightarrow \cZ([s_i, s_{i+1}]), \quad  \sD^i_*: \cZ([s_i, s_{i+1}])  \rightarrow \cM([s_i, s_{i+1}]), \]
such that for any $t \in [s_i, s_{i+1}]$, $\cM(t) = \sD^i_*(\cZ(t))$.
\end{assumption}

In Assumption \ref{assum:mfd}, $\sE^i_*$ and $\sD^i_*$ serve as the oracle encoder and decoder for the trajectory manifold for each segment $i \in [\sfS]$. We denote the Lipschitz constant of the oracle encoder on each time segment by ${\rm Lip}_{\sE^i} = \sup_{t \in [s_i, s_{i+1}]} \Lip_{\cM(t)}(\sE^i)$ and the oracle decoder in each time segment by ${\rm Lip}_{\sD^i} = \sup_{t \in [s_i, s_{i+1}]} \Lip_{\cZ(t)}(\sD^i)$.

We will first establish the result where we set the windows to be exactly aligned with the manifold segments in Assumption \ref{assum:mfd}. In other words, we will choose $\sfW = \sfS$ and define $0 < w_1 < \cdots < w_{\sfW+1} = T$ such that $w_i = s_i$ for all $i \in [\sfW]$.  We   construct encoder, decoder, and propagator networks over each window (which is the same as a segment). In order to translate between adjacent windows, we use transcoder networks. The inference by WeldNet for any initial state $\bfx(0) \in \cM(0)$ can be represented in \eqref{eq:weldmultiwindow}.

We introduce a latent time-evolution map, defined on the domain $\cZ([w_i, w_{i+1}])$ denoted as $\sP^i_*$ such that $\sP^i_*(\bfz, t) = \sE^i_*(\sF(\sD^i_*(\bfz), t))$ {for $t \le w_{i+1}-w_i$},  which gives the evolution of latent code for $\sF$ within window $i$:\ $\sD^i_* \circ \sP_* \circ \sE^i_* = \sF$. The evolution of $\sP^i_*$ operates in latent space, yet it matches the dynamics of the time-evolution operator $\sF$ on a section of the trajectory manifold. We will call $\sF$ the time-evolution operator, and $\sP^i_*$ the \textbf{oracle propagator} for the $i$th window. 

We are primarily interested in the case where the latent codes follow an ODE within each segment:
\begin{assumption}[Latent Dynamics]
\label{assum:latentdynamics}
Using the notation in Assumption \ref{assum:mfd}, suppose for all $i \in [\sfS]$, there is a Lipschitz function $g^i : \cZ([s_i, s_{i+1}]) \rightarrow \bR^{d+1}$ such that for all $\bfz(s_i) \in \cZ(s_i)$, the latent code $\bfz(t) = \sP^i_*(\bfz(s_i), t - s_i)$ for  $t \in [s_i, s_{i+1}]$ satisfies $\frac{\partial \bfz}{\partial t} (t) = g^i(\bfz(t))$.
\end{assumption}

The ODE in Assumption \ref{assum:latentdynamics} models the dynamics of the latent code. In particular, the $(d+1)$th coordinate of the latent code $\bfz(t)$ is time, i.e. $\bfz(t)_{d+1} = t$ and therefore the $(d+1)$th coordinate of $g^i$ is $1$, i.e. $g^i_{d+1} = 1$.
We denote the Lipschitz constant of $g^i$ (for the $i$th segment) as ${\rm Lip}(g^i) := \sup_{t \in [s_i, s_{i+1}]} \Lip_{\cZ(t)}(g^i)$ and the Lipschitz constant of $g=\{g^i\}_{i=1}^{\sfS}$ (among all segments) as ${\rm Lip}(g) = \sup_{i \in [\sfW]} {\rm Lip}(g^i)$.

Our first result (presented in Lemma \ref{lemma:weldlatentode} below) gives an approximation guarantee for WeldNet in the latent dynamics setting, assuming that the windows are set to be equal to the manifold segments. 

\begin{lemma}
\label{lemma:weldlatentode}
Suppose Assumptions \ref{assum:evolve}, \ref{assum:mfd}, and \ref{assum:latentdynamics} hold. Assume the time grid $\bT$ is uniform such that $t_{k+1} - t_k = \Delta t$ for all $k \in [\sfT-1]$.  Let $\epsilon > 0$, and suppose there are $\sfW = \sfS$ windows such that for all $i \in [\sfW]$, we have $w_i = s_i$ and each window has number of time steps $|\bT^i| > 1 + {{\rm Lip}(g) \|g\|_{L^{\infty}
} (w_{i+1} - w_i)^2 e^{{\rm Lip}(g) (w_{i+1} - w_i)}}{\epsilon}^{-1}$. 
Then for each $i \in [\sfW]$ there exist an encoder network $\sE_{\NN}^i \in \cF_{\NN}(D, d+1, L_{\sE^i}, W_{\sE^i})$, a decoder network $\sD_{\NN}^i \in \cF_{\NN}(d+1, D, L_{\sD^i}, W_{\sD^i})$, a propagator network $\sP_{\NN}^i \in \cF_{\NN}(d+1, d+1, L_{\sP^i}, W_{\sP^i})$, and a transcoder network (for $i < \sfW$) $\sT^i_{\NN} \in \cF_{\NN}(d+1, d+1, L_{\sT^i}, W_{\sT^i})$ such that for any $k \in [\sfT]$, the WeldNet $\sW_{\NN}$  given in \eqref{eq:weldmultiwindow} guarantees
\[ \sup_{\bfx(0) \in \cM(0)} \|\sW_{\NN}(\bfx(0), t_k) - \sF(\bfx(0), t_k) \|_{\bR^D} < \epsilon. \]
The network parameters are
\begin{align*}
\textstyle
& L_{\sE^i}=O\left(\log^2\left({\sfS}/{\epsilon}\right)\right), \ W_{\sE^i}=O\left(D\left({\sfS}/{\epsilon}\right)^{d+1} \right)
\\
& L_{\sD^i}=O\left(\log\left({\sfS}/{\epsilon}\right)\right), \
W_{\sD^i}=O\left(D\left({\sfS}/{\epsilon}\right)^{d+1} \right)
\\
& L_{\sP^i}=O\left(\log\left({\sfS}/{\epsilon} \right)\right), \ 
W_{\sP^i}=O\left(\left({\sfS}/{\epsilon}\right)^{d+1}\right)
\\
& L_{\sT^i}=O\left(\log\left({\sfS}/{\epsilon}\right)\right), \ 
W_{\sT^i}=O\left(\left({\sfS}/{\epsilon}\right)^{d} \right)
\end{align*}
where we hide constants depending on (for all $i \in [\sfW]$) $\|g_i\|_{L^{\infty}}$, ${\rm Lip}(g_i)$, $\log(D)$, $d$, $\bm{\tau}(\cM([s_i, s_{i+1}]))$, {$\max_i{\rm Lip}(\sE_*^i)$, $\max_i{\rm Lip}(\sD_*^i)$}, ${\rm Lip}(\sF)$, the volume of $\cM([s_i, s_{i+1}])$, and $\sup_{\bfx \in \cM([s_i, s_{i+1}])} \|\bfx\|_{\bR^D}$.
\end{lemma}

Lemma \ref{lemma:weldlatentode} provides a representation guarantee of WeldNet, and the network size scales crucially with the intrinsic dimension $d$ instead of the data dimension $D$, which demonstrates the efficiency of model reduction. The proof is deferred to Appendix \ref{app:proveweldlatentodelemma}.

Lemma \ref{lemma:weldlatentode} requires the number of windows to be equal to the number of segments. One major advantage of WeldNet is the ability to train with a number of windows potentially higher than the number of segments. Fortunately, the approximation guarantee of Lemma \ref{lemma:weldlatentode} can be extended to this scenario, as long as the windows partition each segment. This is the subject of the next theorem, which follows from Lemma \ref{lemma:weldlatentode}.

\begin{definition}
We say a sequence $\{w_i\}_{i=1}^{\sfW+1}$ with $0 = w_1 < w_2 < \cdots < w_{\sfW + 1} = T$ \textbf{subdivides the segments} $\{s_j\}_{j=1}^{\sfS+1}$ if for all $i \in [\sfW]$, there is a unique $j \in [\sfS]$ such that $[w_i, w_{i+1}] \subseteq [s_j, s_{j+1}]$. We denote by $\pi : [\sfW] \rightarrow [\sfS]$ that assigns each window index to the unique segment index it lies within.
\end{definition}

\begin{theorem}
\label{thm:weldlatentode}
Suppose Assumptions \ref{assum:evolve}, \ref{assum:mfd}, and \ref{assum:latentdynamics} hold. Assume the time grid $\bT$ is uniform such that $t_{k+1} - t_k = \Delta t$ for all $k \in [\sfT-1]$. Let $\epsilon > 0$ and suppose there are $\sfW \geq \sfS$ windows so that $0 = w_1 < w_2 < \cdots < w_{\sfW+1}$ subdivides the segments $\{s_j\}_{j=1}^{\sfS+1}$. Additionally, suppose for all $i \in [\sfW]$, each window has number of time steps $|\bT^i| > 1 + {{\rm Lip}(g) \|g\|_{L^{\infty}
} (w_{i+1} - w_i)^2 e^{{\rm Lip}(g) (w_{i+1} - w_i)}}{\epsilon}^{-1}$. 
Then for each $i \in [\sfW]$ there exist an encoder $\sE_{\NN}^i \in \cF_{\NN}(D, d+1, L_{\sE^i}, W_{\sE^i})$, a decoder $\sD_{\NN}^i \in \cF_{\NN}(d+1, D, L_{\sD^i}, W_{\sD^i})$, a propagator $\sP_{\NN}^i \in \cF_{\NN}(d+1, d+1, L_{\sP^i}, W_{\sP^i})$, and a transcoder (for $i < \sfW$) $\sT^i_{\NN} \in \cF_{\NN}(d+1, d+1, L_{\sT^i}, W_{\sT^i})$ such that for any $k \in [\sfT]$, the WeldNet $\sW_{\NN}$  given in \eqref{eq:weldmultiwindow} guarantees
\[ \sup_{\bfx(0) \in \cM(0)} \|\sW_{\NN}(\bfx(0), t_k) - \sF(\bfx(0), t_k) \|_{\bR^D} < \epsilon. \]
The network parameters are
\begin{align*}
\textstyle
& L_{\sE^i}=O\left(\log^2\left({\sfS}/{\epsilon}\right)\right), \ W_{\sE^i}=O\left(D\left({\sfS}/{\epsilon}\right)^{d+1} \right)
\\
& L_{\sD^i}=O\left(\log\left({\sfS}/{\epsilon}\right)\right), \
W_{\sD^i}=O\left(D\left({\sfS}/{\epsilon}\right)^{d+1} \right)
\\
& L_{\sP^i}=O\left(\log\left({\sfS}/{\epsilon} \right)\right), \ 
W_{\sP^i}=O\left(\left({\sfS}/{\epsilon}\right)^{d+1}\right)
\\
& L_{\sT^i}=O\left(\log\left({\sfS}/{\epsilon}\right)\right), \ 
W_{\sT^i}=O\left(\left({\sfS}/{\epsilon}\right)^{d} \right)
\end{align*}
where we hide constants depending on (for all $i \in [\sfW]$) $\|g_{\pi(i)}\|_{L^{\infty}}$, ${\rm Lip}(g_{\pi(i)})$, $\log(D)$, $d$, $\bm{\tau}(\cM([s_{\pi(i)}, s_{\pi(i)+1}]))$, {$\max_i{\rm Lip}(\sE_*^{\pi(i)})$, $\max_i{\rm Lip}(\sD_*^{\pi(i)})$}, ${\rm Lip}(\sF)$, the volume of $\cM([s_{\pi(i)}, s_{i+1}])$, and $\sup_{\bfx \in \cM([w_i, w_{i+1}])} \|\bfx\|_{\bR^D}$.
\end{theorem}

Theorem \ref{thm:weldlatentode} is a generalization of Lemma \ref{lemma:weldlatentode} (they both rely on Assumptions \ref{assum:evolve}, \ref{assum:mfd}, and \ref{assum:latentdynamics}), and its proof is deferred to Appendix \ref{app:proveweldlatentode}. Without Assumption \ref{assum:latentdynamics}, we can still establish an approximation error guarantee for WeldNet, but with a different construction for the propagator network. The proof of the following theorem is deferred to Appendix \ref{app:proveweldgeneral}.

\begin{theorem}
\label{thm:weldgeneral}
 Suppose Assumption \ref{assum:evolve} and \ref{assum:mfd} hold. Let $\epsilon > 0$, and suppose there are $\sfW \geq \sfS$ windows so that $0 = w_1 < w_2 < \cdots < w_{\sfW+1}$ subdivides the segments. Suppose the time grid $\sfT$ is uniform such that $t_{k+1} - t_k = \Delta t$ for all $k \in [\sfT - 1]$. For all $i \in [\sfS]$, let $\overline{\bm{T}_s} = |\bT \cap [s_i, s_{i+1}]|$ denote the number of time steps in segment $i$. Then for each $i \in [\sfW]$, there exist an encoder network $\sE_{\NN}^i \in \cF_{\NN}(D, d+1, L_{\sE^i}, W_{\sE^i})$, a decoder network $\sD_{\NN}^i \in \cF_{\NN}(d+1, D, L_{\sD^i}, W_{\sD^i})$, a propagator network $\sP_{\NN}^i \in \cF_{\NN}(d+1, d+1, L_{\sP^i}, W_{\sP^i})$, and a transcoder network (for $i < \sfW$) $\sT^i_{\NN} \in \cF_{\NN}(d+1, d+1, L_{\sT^i}, W_{\sT^i})$ such that for any $k \in [\sfT]$, the WeldNet  $\sW_{\NN}$ given in \eqref{eq:weldmultiwindow} guarantees
\[ \sup_{\bfx(0) \in \cM(0)} \|\sW_{\NN}(\bfx(0), t_k) - \sF(\bfx(0), t_k) \|_{\bR^D} < \epsilon. \]
The network parameters are
\begin{align*}
\textstyle
& L_{\sE^i}=O\left(\log^2\left({\sfS}/{\epsilon}\right)\right), \ W_{\sE^i}=O\left(D\left({\sfS}/{\epsilon}\right)^{d+1} \right)
\\
& L_{\sD^i}=O\left(\log\left({\sfS}/{\epsilon}\right)\right), \
W_{\sD^i}=O\left(D\left({\sfS}/{\epsilon}\right)^{d+1} \right)
\\
& L_{\sP^i}=O\left(\overline{\bm{T}_{\pi(i)}} \log\left({\sfS}/{\epsilon} \right)\right), \ 
W_{\sP^i}=O\left(\left({\overline{\bm{T}_{\pi(i)}}}/{\epsilon}\right)^d\right)
\\
& L_{\sT^i}=O\left(\log\left({\sfS}/{\epsilon}\right)\right), \ 
W_{\sT^i}=O\left(\left({\sfS}/{\epsilon}\right)^{d+1} \right)
\end{align*}
where we hide constants in $O$ depending
on (for all $i \in [\sfW]$) $d$, $\log(D)$, $\bm{\tau}(\cM([0, T]))$, {$\max_i{\rm Lip}(\sE_*^i)$, $\max_i{\rm Lip}(\sD_*^i)$}, ${\rm Lip}(\sF)$, volume of $\cM([0, T])$, and $\sup_{\bfx \in \cM([0, T])} \|\bfx\|_{\bR^D}$.
\end{theorem}

\textbf{Remark.} The main difference between Theorem \ref{thm:weldgeneral} and Theorem \ref{thm:weldlatentode} is that the former does not assume Assumption \ref{assum:latentdynamics}, which makes the approximation of the propagator network more challenging, resulting in the large propagator size. We briefly remark that that the idea of the proof of Theorem \ref{thm:weldgeneral} does not require a uniform time grid size, so it can be extended to the case of non-uniform time grids (in which case the size of the propagator will depend on the minimum and maximum time grid spacing).

\section{Experiments}
\label{sec:numerical}
\subsection{Data Generation}
\label{sec:diffeqsetup}
For simplicity, we  start from an evolutionary ordinary differential equation (ODE)
\begin{equation}
\begin{split}
\label{eq:ode}
\partial_t \bfx &= \mathfrak{F}(\bfx(t), t), \\
\bfx(0) &= \bfx_0
\end{split}
\end{equation}
with solution $\bfx(t) : [0, T] \rightarrow \bR^D$, where $T > 0$ is the end time.
We will measure the solution of the ODE in \eqref{eq:ode} at discrete time points; \, let $0 = t_1 < t_2 <\ldots < t_{\sfT} = T$ be the discretized time locations for some fixed $\sfT \in \mathbb{N}$. The data for the ODE represents a trajectory of \eqref{eq:ode} with initial condition $\bfx(t_1)=\bfx_0$: $\left\{\bfx(t_1), \bfx(t_2), \ldots, \bfx(t_{\sfT}) \right\} \subseteq \bR^D$. 

In the data-driven framework, we collect trajectory data from multiple initial conditions. Let $\{\bfx_n(t_1)\}_{n=1}^N$ be $N$ sets of initial conditions. The collection of discretized trajectory data with these initial conditions is denoted by 
\begin{equation}
\label{eq:trainingdata}
\left\{\bfx_n(t_1), \bfx_n(t_2), \ldots, \bfx_n(t_{\sfT})\right\}_{n=1}^N,
\end{equation}
which serves as the training data (and can be thought of as $N$ elements of $\bR^{\sfT \times D}$). 

We next consider a partial differential equation (PDE)
\begin{equation}
\begin{split}
\label{eq:pde}
\partial_t \mathfrak{u} &= \mathfrak{F}(\mathfrak{u}(\bfx, t), t), \\
\mathfrak{u}(\bfx, 0) &= f(\bfx), \quad \forall \bfx \in \Omega,
\end{split}
\end{equation}
with solution $\mathfrak{u} : \Omega \times [0, \sfT] \rightarrow \bR$ where $\Omega \subset \bR^{d_\Omega}$ is the spatial domain. We will discretize the solution of the PDE in the spatial and temporal domain. Let $\mathbb{X} = (X_1, \dots, X_D) \subseteq \Omega$ be a discretization set in the spatial domain and consider the same temporal discretization as before. The discretized data for the PDE solution $\mathfrak{u}$ is denoted $\left\{\mathfrak{u}(\mathbb{X},t_1),\mathfrak{u}(\mathbb{X},t_2), \ldots, \mathfrak{u}(\mathbb{X},t_T)\right\} \subseteq \bR^{D}$, which represents a discretized trajectory of the PDE in \eqref{eq:pde} with the initial condition $f$. 

Given initial conditions $\{f_n\}_{n=1}^N$ and solution trajectories $\{ \mathfrak{u}_n \}_{n=1}^N$ (for all $n \in [N]$, $\mathfrak{u}_n$ solves \eqref{eq:pde} with initial condition $f = f_n$), the trajectory data in this case has the same form as \eqref{eq:trainingdata} if we use the notation $\bfx_n(t_k) = \mathfrak{u}_n(\mathbb{X}, t_k)$, that is, $\bfx_n(t_k) \in \bR^D$ with $i$th component given by $\mathfrak{u}_n(X_i, t_k).$

Many differential equations in applications can be complex to simulate - the discretization set size grows exponentially in the spatial dimension $d_{\Omega}$. 
When the spatial dimension $d_{\Omega}$ is higher than one, a large number of discretization points are needed.
Our goal is to build a low-dimensional surrogate model based on training data of the form \eqref{eq:trainingdata} for ODEs, PDEs, and other evolutionary  processes. After training, one can predict the solution trajectory for a new initial condition by evolving the surrogate model in a low-dimensional space.

\subsection{Benefit of Windowed Approach}
\label{sec:illustrate}
In this section, we illustrate the benefit of using a windowed autoencoder approach on a simple but instructive example. Consider the one-dimensional transport equation for $T = 0.3$, given by
\begin{equation}
\label{eq:transport1d}
u_t = -u_x; \quad u(x, 0) = g(x), \: x \in (0, 1),
\end{equation}
%
with zero Dirichlet boundary conditions. We consider the weak version of this PDE, so the initial condition $g$ does not have to be differentiable everywhere.

We consider initial conditions containing two hats. Let $\sigma(x) = \max(x, 0)$, and fix $\epsilon = 0.05$. We define the ``hat'' function centered at $0$ with width $\epsilon$ as $$H_{ \epsilon} (x) =
2 \epsilon^{-1}\left(\sigma(x) - 2\sigma\left(x - \frac{\epsilon}{2}\right) + \sigma(x- \epsilon)\right).$$
We consider the following one-parameter set of initial conditions:
\begin{align}
\label{eq:tscale}
\hat{g}_\text{tscale} &= \{a\cdot H_{0.05}(x - 0.1) + H_{0.05}(x-0.2) : a \in [1, 4]\}. 
\end{align}
%

%
We discretize the spatial domain $[0, 1]$ with 512 equally spaced points. We consider a probability measure on the set $\hat{g}_\text{tscale}$ obtained by sampling the parameter $a$ uniformly from $[1, 4]$ and then outputting the discretized initial condition corresponding to that value of $a$. Formally, we define a function $G_\text{tscale} : [1, 4] \rightarrow \bR^{512}$ that maps from parameter ($a$) to initial condition. This allows us to construct the initial measure $\rho_\text{tscale}(0) = (G_\text{tscale})_\sharp\textbf{Unif}([1, 4])$. We sample by $\rho_\text{tscale}(0)$ by first choosing $1 \leq a \leq 4$ uniformly and outputting the corresponding (discretized) initial condition for $a$ as defined in \eqref{eq:tscale}. 

The trajectory manifold $\cM_\text{tscale}$ is obtained by collecting solutions from the PDE in \eqref{eq:transport1d} with initial conditions from $\rho_\text{tscale}(0)$. For simplicity, we omit the time interval notation for the manifolds in this section. We collect 500 trajectories with 301 time steps until $T=0.3$, so our data is of the form $((\bfx_n(t_k)))_{k=1}^{51})_{n=1}^{500}$, as detailed in Section \ref{sec:setup}.

We will train WeldNet models on this dataset with one window, two windows, and four windows, and we will attempt to roughly equalize the number of trainable parameters and number of training epochs. Specifically, we will implement each component of WeldNet using 3 layer neural networks so that we implement the one-window, two-window, and four-window models with networks of width 1000, 500, and 250 (respectively) and train each component for 1200, 600, and 300 epochs (respectively). The total training time for each model in minutes is 42.28, 15.73, and 8.31 (respectively).

In this section, we will use a latent space dimension of $2$ for the autoencoders - equal to the intrinsic dimension of the data (see Table \ref{table:datasetid}). This results are in a lower performance than a latent space dimension of $4$ that we use later in the paper. We use this reduced latent space dimension to show that using windows with the estimated intrinsic dimension (instead of using a higher latent space dimension) can overcome the error accumulation issues of a single-window autoencoder model, and due to the ease of visualizing two dimensional latent spaces.

We train autoencoders and propagators together for all models. We show the latent space of each model in Figures \ref{fig:illustratew1}, \ref{fig:illustratew2}, and \ref{fig:illustratew4} for the one-window, two-window, and four-window models respectively. We show each window individually, and we show the points clouds colored by their parameter in the left plot and time in the right plot. The latent space for the one window model in Figure \ref{fig:illustratew1} involves pinches and twists, while the latent spaces (in each window) for the two-window and four-window are much more uniform. We note that even though the latent space in the second window of the two-window model has a compressed area in Figure \ref{fig:illustratew2}, the propagator is still able to learn these latent dynamics.

\begin{figure}
    \centering
\fbox{\includegraphics[width=0.55\linewidth]{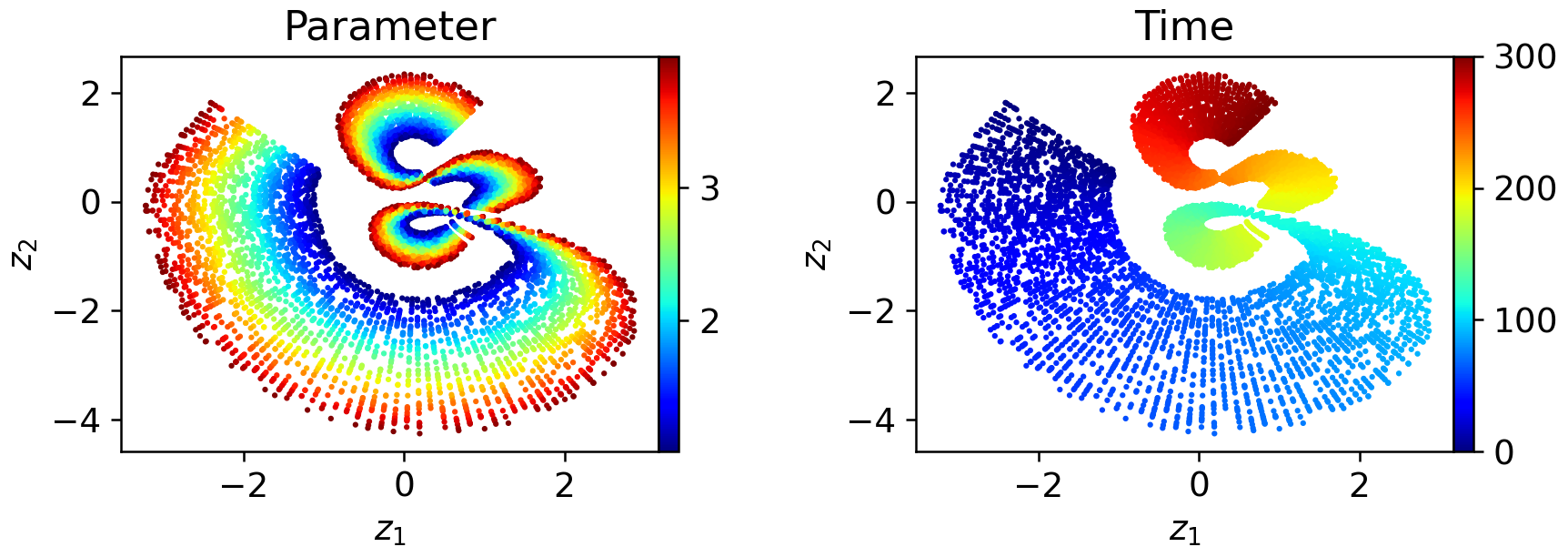}}
    \caption{Latent space of illustrative example with one-window model. The left plot is colored according to the parameter $a$ in the initial condition \eqref{eq:tscale} and the right plot is colored according to time $t$.}
    \label{fig:illustratew1}
\end{figure}

\begin{figure}[htbp]
  \begin{minipage}[t]{0.49\textwidth}
    \centering
\fbox{\includegraphics[width=0.96\textwidth]{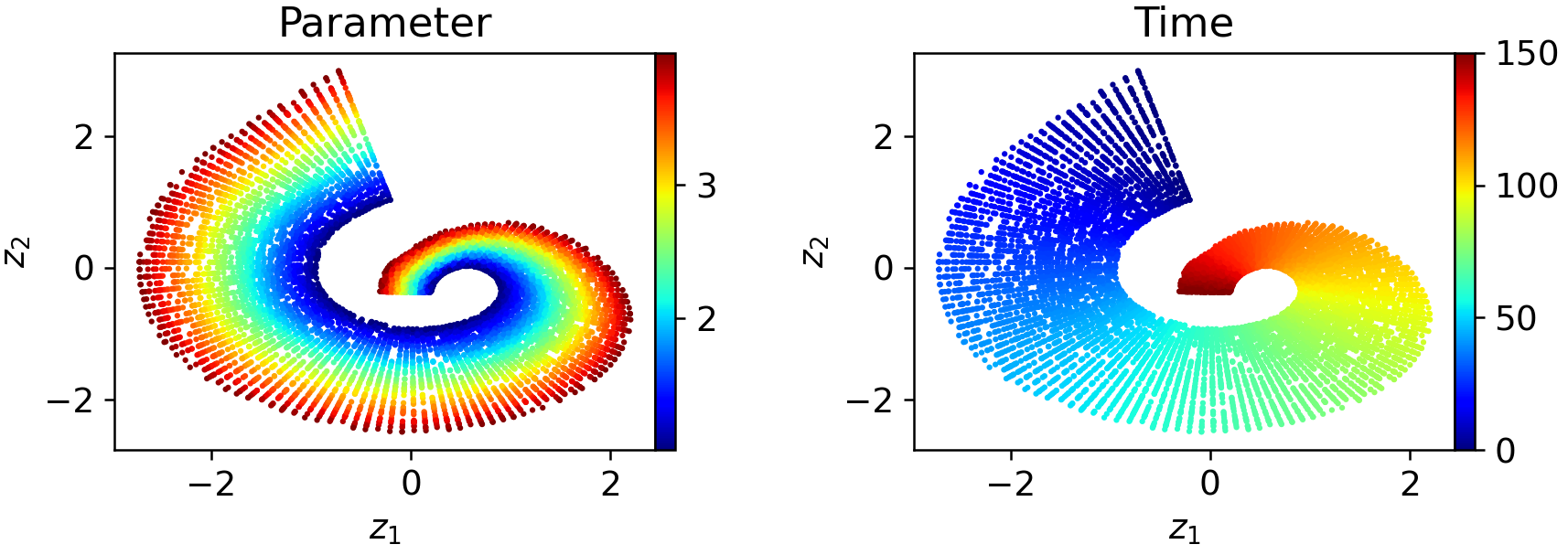}}
  \end{minipage}
  \hfill
  \begin{minipage}[t]{0.49\textwidth}
    \centering
\fbox{\includegraphics[width=0.96\textwidth]{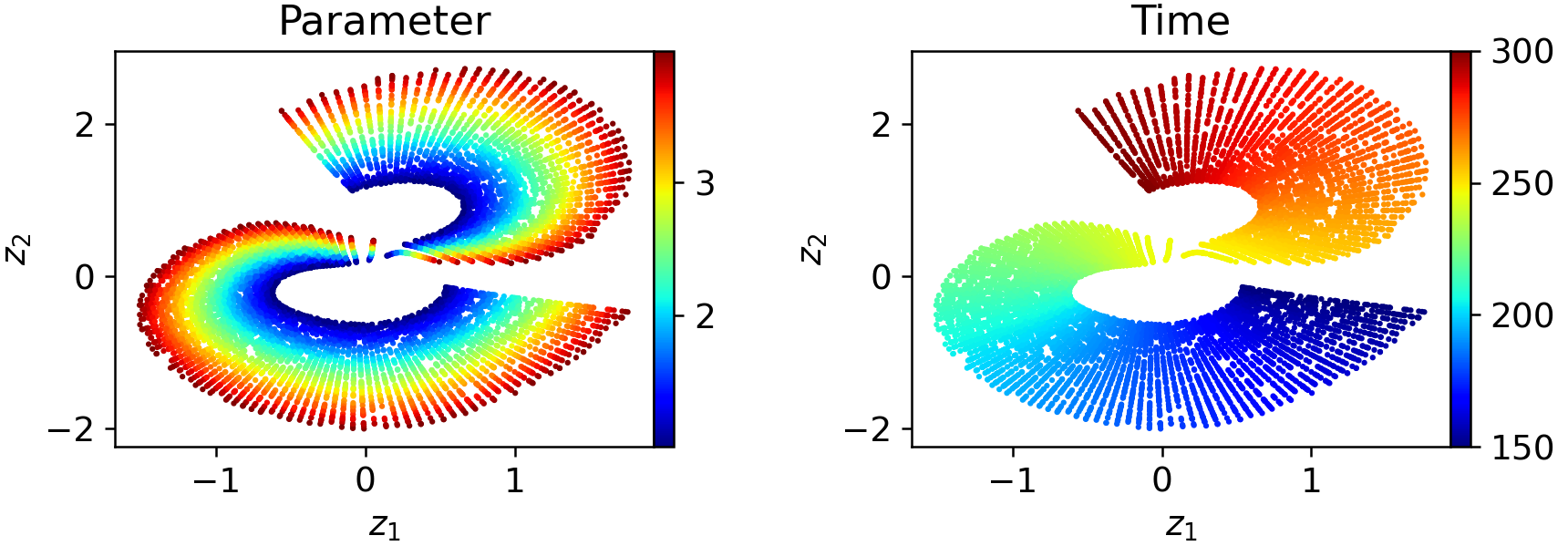}}
  \end{minipage}
  \caption{Latent space of illustrative example with two-window model.}
  \label{fig:illustratew2}
\end{figure}

  \begin{figure}[htbp]
  \begin{minipage}[t]{0.49\textwidth}
    \centering
\fbox{\includegraphics[width=0.96\textwidth]{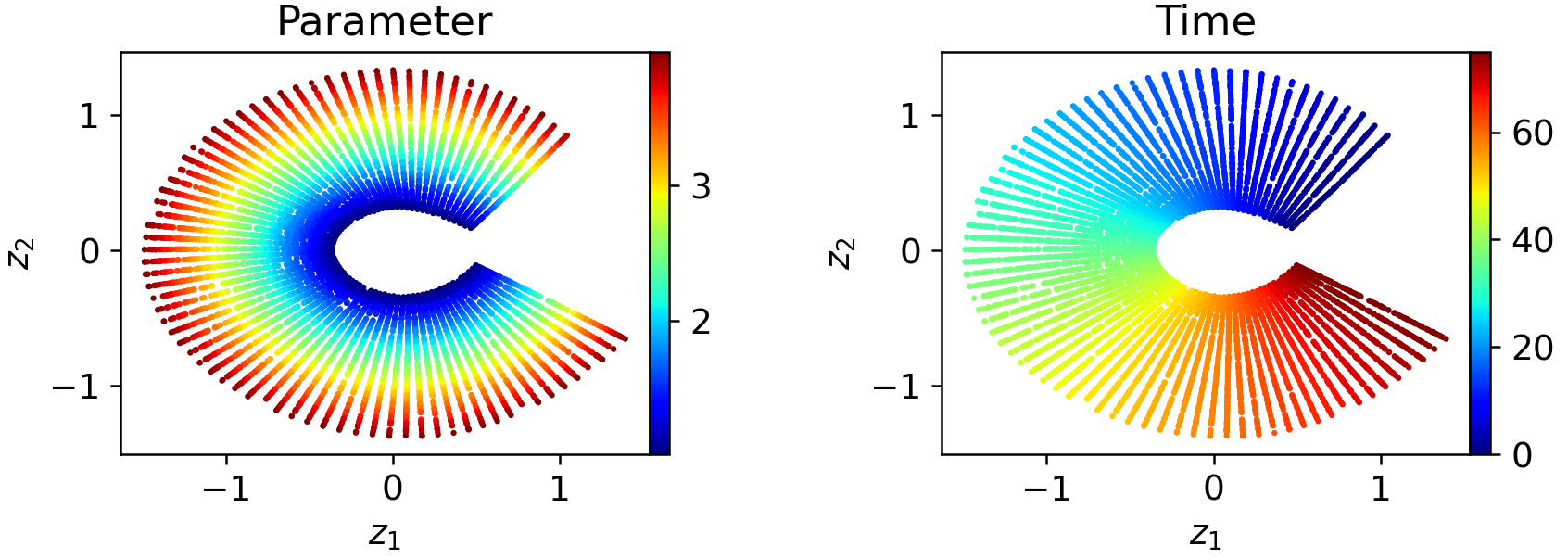}}
  \end{minipage}
  \hfill
  \begin{minipage}[t]{0.49\textwidth}
    \centering
\fbox{\includegraphics[width=0.96\textwidth]{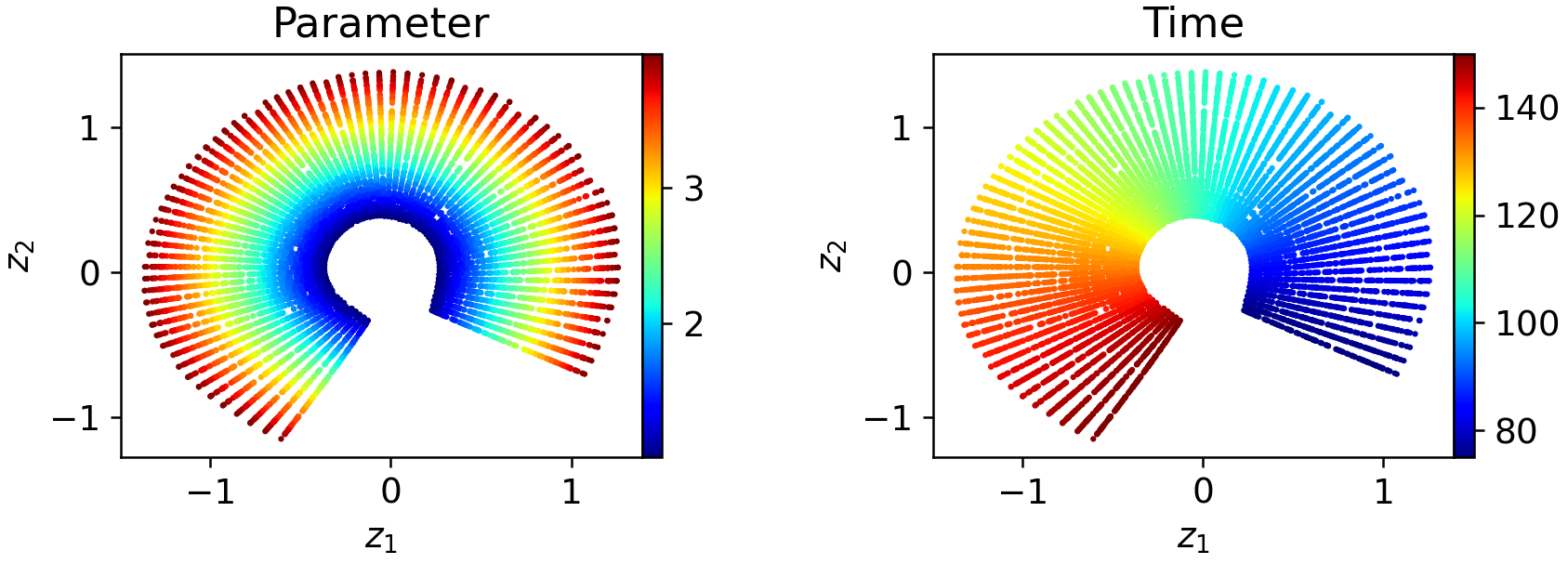}}
  \end{minipage} \\

\begin{minipage}[t]{0.49\textwidth}
    \centering
\fbox{\includegraphics[width=0.96\textwidth]{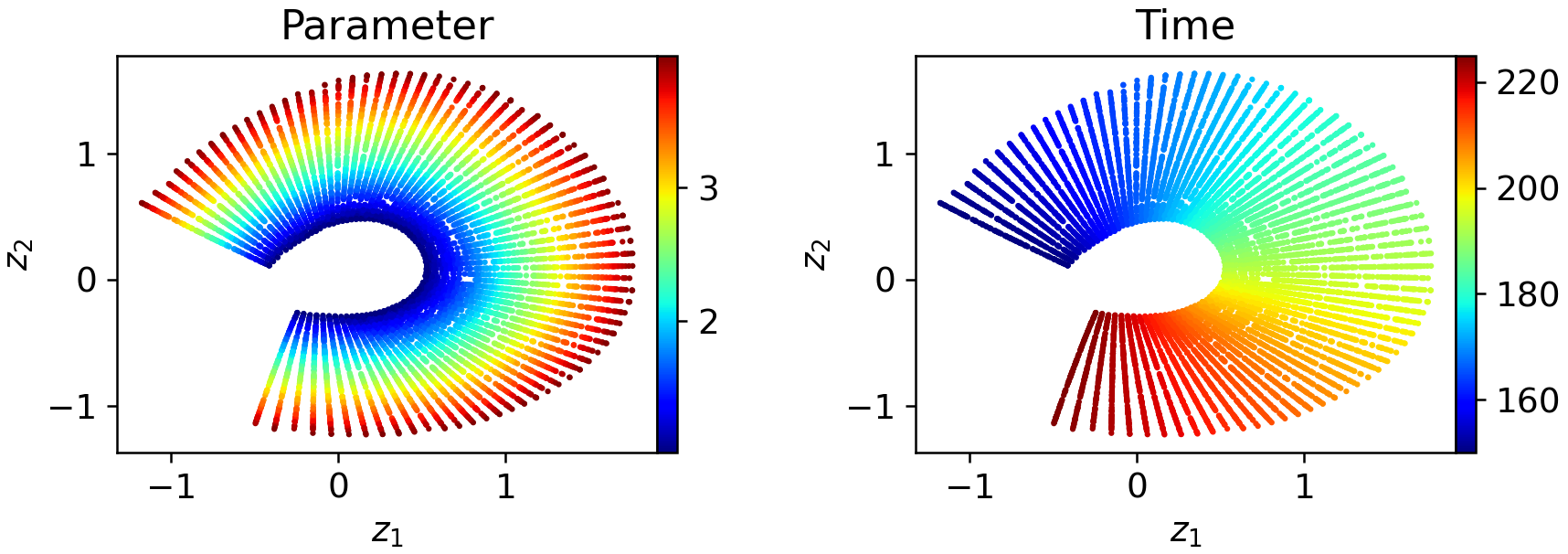}}
  \end{minipage}
  \hfill
  \hspace{0.3em}
  \begin{minipage}[t]{0.49\textwidth}
    \centering
\fbox{\includegraphics[width=0.96\textwidth]{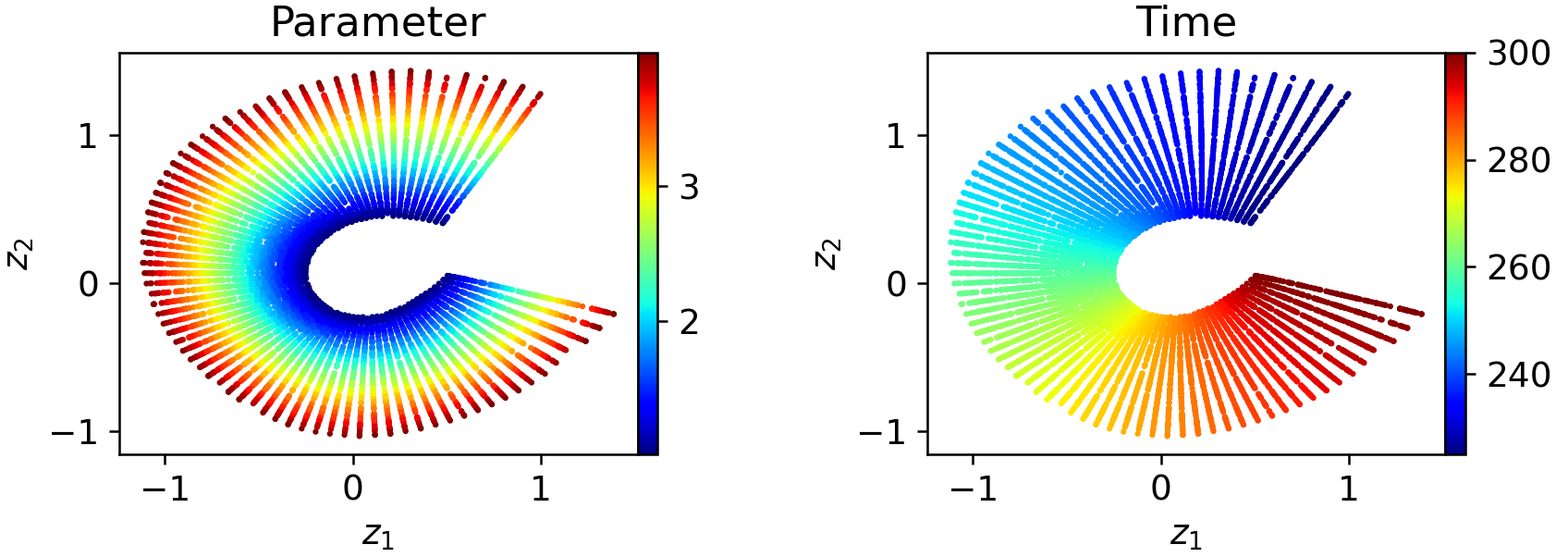}}
  \end{minipage}
  \caption{Latent space of illustrative example with four-window model.}
  \label{fig:illustratew4}
  
\end{figure}

Next, we train propagators for the one-window model and the propagators and transcoders for the two-window and four-window models. We compare the projection and operator error of the models over all times in Figure \ref{fig:illustrativeprojops}. 
{The projection error is the error of the autoencoder reconstruction on the test data (i.e. the representation or approximation error of the autoencoder in representing the solution manifold). The operator error is the prediction error of a solution given an initial condition, which depends on both the projection error and the error of the propagator network in learning the dynamics.} We observe severe error accumulation in the one-window model. This problem is resolved by using more windows.

\begin{figure}
    \centering
    \includegraphics[width=0.6\linewidth]{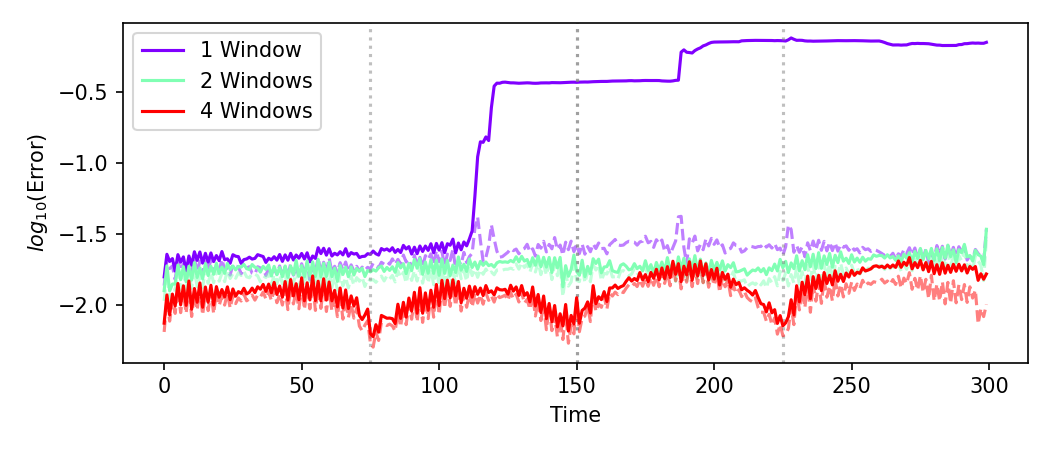}
    \caption{Projection (dashed) and Operator (solid) error for the example in Section \ref{sec:illustrate}.}
    \label{fig:illustrativeprojops}
\end{figure}

\subsection{Implementation Details}
\label{sec:implementation}
The datasets for the Burgers' equation and KdV equation trajectory manifolds were generated using the chebfun package \cite{driscoll2014chebfun} in Matlab. The datasets for the transport equation trajectory manifolds were generated using the analytic solution of the equation.

\subsubsection{Model Descriptions}
\label{sec:modeldetails}
In this section, we detail the architectures of the components of WeldNet and comparison models. Suppose the data is $D$-dimensional. 

FF-WeldNet and Conv-WeldNet uses ReLU networks with 3-hidden layers and width 200 for the propagators and transcoders. FF-WeldNet uses feedforward ReLU networks with 3-hidden layers and width 500 for the encoder and decoder. Conv-WeldNet uses a 4 layer convolutional encoder with the convolution layers having channel size 8, 16, 32, 32, and kernel size 8, 8, 8, 4, respectively, with stride 2, and symmetric zero padding of 1, and it uses a symmetric architecture of convolution transpose layers for the decoder. 

Note that a one-window WeldNet model is a normal latent dynamics model which has already been considered in the literature (e.g. see \cite{lee2020model}), so we will use the label ``AENet'' to refer to that. The goal of our work is to show the benefit of multiple windows, so the performance of a one-window WeldNet model or an AENet should only be considered as a baseline.

To illustrate the advantage of dimension reduction, we train a model consisting of a propagator operating directly on the $D$-dimensional data. We call this approach High Dimensional Propagator (HDP). Specifically, we train (depth 3) ReLU networks with input and output dimension equal to $D$ ($D = 512$ for most of our examples). We use a width of $1000$ since HDP only consists of a single network (while WeldNet and other models have several networks). We found that using residual/displacement based training as done for the the WeldNet propagator (as described in Section \ref{sec:displacement}) leads to training instability. Thus, we will train the HDP to predict the high dimensional state at the next time directly, i.e. given an input of $x(t) \in \cM(t) \subseteq \bR^D$, the HDP should predict $x(t+\Delta t) \in \bR^D$.

The Time-Input feedforward networks are three-hidden layer ReLU networks with width $1000$. If the data is $D$-dimensional, the Time-Input feedforward network has $D+1$ input neurons and $D$ output neurons. Given an input $u(0)$ (representing a data point initial condition) and a time $t$, the Time-Input network outputs the evolution of that input by $t$ time units, i.e. it should approximately be $u(t)$.

We also compare to a model called Latent Deep Operator Network (Latent-DON) \cite{kontolati2024learning}. Similar to a one-window WeldNet, Latent-DON trains a single autoencoder over all times and then trains a latent prediction model that uses a Deep Operator Network (DON) architecture. DON is an operator learning method that takes time as an input and tries to predict the evolution of an initial condition by that given time, similar to time-input models, but in latent space. Given an input $u(0) \in \bR^D$ and a time $t$, Latent-DON first projects to a $k$-dimensional latent code $z(0) \in \bR^k$ using the encoder, and then uses the latent code and time as inputs to the DON to try to predict the evolution of that latent code by $t$ time units. We then decode the predicted evolved latent code, and this value should approximately be $u(t)$. See Appendix \ref{app:don} for specific details on the DON architecture used in Latent-DON.

We remark that the Time-Input and Latent-DON models have an easier learning task than WeldNet, because they are only tested to predict the evolution of the initial conditions by a given time, while WeldNet is capable of evolving data from any start time to any end time on the time grid.

LDNet is a model reduction based method for learning dynamical systems. The method trains a dynamics network that evolves latent codes and a reconstruction network that maps from latent code to function predictions on a grid. Specifically, each initial condition is assigned a latent code consisting of $f$ fixed dimensions and $k$ dynamic dimensions that start at 0. The fixed dimensions are given by the codes assigned to the initial conditions (which need to be known ahead of time). The dynamics network is trained to evolve the dynamic portion of each latent code, with $f+k$ input neurons and $k$ output neurons. The latent code at a given time can be evolved by the dynamics network iteratively. The output prediction is then obtained by passing the evolved latent code through the reconstruction network. This means that the reconstruction network has $f+k$ input neurons and $D$ output neurons. We use a width of $500$ for the dynamics and reconstruction networks, and we use the ReLU activation function for both networks. 

\textbf{Remark.} This description of LDNet is reliant on a fixed size $D$ for the data and reconstruction network, while the original implementation of LDNet in \cite{regazzoni2024learning} uses a grid-independent reconstruction network. We use a modified grid based version for a more direct comparison to our grid based WeldNet models, and due to training issues with the grid independent version. See Appendix \ref{app:ldnet} for a comparison between the original LDNet and the grid based LDNet that we will subsequently use in this paper.

Weak-form Latent Space Dynamics Identification, WLaSDI, is a reduced-order modeling technique recently proposed in \cite{tran2024weak}. WLaSDI uses a one-hidden layer autoencoder with 1024 hidden neurons in the encoder and 6168 hidden neurons in the decoder with the $\bm{silu}$ activation function, $\bm{silu}(x) = \frac{x}{1+e^{-x}}$. The latent dynamics are then learned by solving a least-squares problem involving a dictionary of trial functions; specifically we apply the region-based local dynamics identification algorithm discussed in Section 3.3.2 of \cite{tran2024weak}.

All WeldNet, Time-Input, HDP, Latent-DON, and LDNet models are trained with a batch size of 32, learning rate of 1e-4, with an adaptive learning schedule that decay by a factor of 0.3 with a patience of 15 epochs, to a minimum learning rate of 1e-6. We use PyTorch \cite{paszke2019pytorch} to implement the neural networks, and we use PyTorch's implementation of the AdamW optimization algorithm with default settings for training the models.

For WLaSDI, the autoencoder is trained with a batch size of 20, learning rate of 1e-4, with an adaptive learning rate schedule that decays by a factor of 0.1 with a patience of 10 epochs, using the implementation in \cite{tran2024weak}.





\subsubsection{Displacement Networks}
\label{sec:displacement}
When training the propagators, we noticed that using a ``displacement'' or residual architecture for the propagator and transcoder significantly aids in training. Specifically, we implement the propagator network $\sP_\NN$ as a function of the form:

\[ \sP_\NN(\bfz) = \bfz + f_\NN(\bfz), \]
where $f_\NN$ is a trainable neural network. We implement the transcoder in the same way. While changing the propagator and transcoder to this form does not affect the approximation theory, we found experimentally that this change resulted in a much lower loss in the training. On the other hand, the loss explodes even for low learning rate when this method is used for the HDP model, so we only employ this displacement network strategy for WeldNet. Note that LDNet uses a displacement form for its dynamics network as part of its original specification, so we also use this strategy there.

\subsection{Numerical Results}

\label{sec:experiment}
We test WeldNet on several numerical examples, similar to the illustrative example in Section \ref{sec:illustrate}, demonstrating the efficacy of our nonlinear dimension reduction and trajectory learning model. For all examples, we collect 500 trajectories with 301 time steps, so our data is of the form $((\bfx_n(t_k)))_{k=1}^{301})_{n=1}^{500}$, as detailed in Section \ref{sec:setup}.

We consider three versions of WeldNet based on the architecture for the autoencoder: FF-Weld (feedforward networks), Conv-Weld (convolutional networks), and PCA-WeldNet (Principal Component Analysis for linear dimension reduction). Note that PCA-WeldNet is analogous to the classic POD or linear model reduction methods.

We compare with Feedforward Neural Networks, LDNets in \cite{lu2021learning} and Weak-form Latent Space Dynamics Identification (WLaSDI) as described in \cite{tran2024weak}. We use a latent dimension of \textbf{4} for WeldNet and WLaSDI, and we use 1 fixed dimension and 3 dynamic dimensions for LDNet.

The numerical test error of all the trained models at various times can be found in Appendix \ref{app:errortables}.

\subsubsection{Burgers' Equation}
We consider the viscous Burgers' equation with $\nu = 1/1000$ and with periodic boundary conditions for $t \in[0,1] $. 
\begin{align}
\label{eq:burgers}
u_t = \nu u_{xx} - uu_x; \quad u(x, 0) = g(x), \: x \in (0, 1). 
\end{align}%

We consider initial conditions generated by combining two complex base waves. Let $w_0, w_1$ be functions sampled from the Gaussian Random Field $N(0, 7^4(-\frac{d^2}{dx^2} + 7^2I)^{-2.5})$ on $[0, 1)$. For any $a \in [-0.9, 0.9]$, consider the two sets of initial conditions:
\begin{align}
\hat{g}_\text{bscale} = \{aw_0(x) + \sqrt{1-a^2}w_1(x) : a \in [-0.9, 0.9]\}, 
\label{eq:bscaleinitial}
\\
\hat{g}_\text{bshift} = \{0.5 w_0(x-h) + \sqrt{0.75}w_1(x-h) : h \in [0, 1]\}.
\label{eq:bshiftinitial}
\end{align}
The trajectory manifolds $\cM_\text{bscale}$ and $\cM_\text{bshift}$ are obtained by collecting solutions from the PDE in \eqref{eq:burgers} with initial conditions from $\hat{g}_\text{bscale}$ and $\hat{g}_\text{bshift}$ respectively. We collect 500 data trajectories with 301 time steps. We display the estimated intrinsic dimensionality of these datasets by Maximum Likelihood Estimation (MLE) \cite{levina2004maximum} at various times and as a whole in Table \ref{table:datasetid} in Section \ref{sec:intrinsicdim} for details on this computation). The intrinsic dimension for each time is approximately $1$, and the dimension of the datasets for all times is approximately $2$, matching our expectation of 1 parameter dimension plus 1 time dimension.

We compare the performance of our WeldNet models (for $\sfW = 4$) with other models in Figures \ref{fig:bscaleerrorparams} and \ref{fig:bshifterrorparams} in terms of final time Prediction error. We use a latent space dimension of \textbf{four} for this example. The $x$-axis is the parameter value for $a$ or $h$ (respectively), and the $y$-axis shows the prediction error at the final time for the initial condition corresponding to that parameter value. In  Figures \ref{fig:bscaleerrorparams} and \ref{fig:bshifterrorparams}, FF-WeldNet significantly outperforms HDP, highlighting the advantages of model reduction: low-dimensional latent propagators are substantially easier to learn than high-dimensional ones. FF-WeldNet also significantly outperforms PCA-WeldNet, which demonstrates the advantage of nonlinear dimension reduction. On this problem, FF-WeldNet and LDNet achieve the best result.

\begin{figure}[htbp]
  \begin{minipage}[t]{0.48\textwidth}
    \centering
    \includegraphics[width=0.95\textwidth]{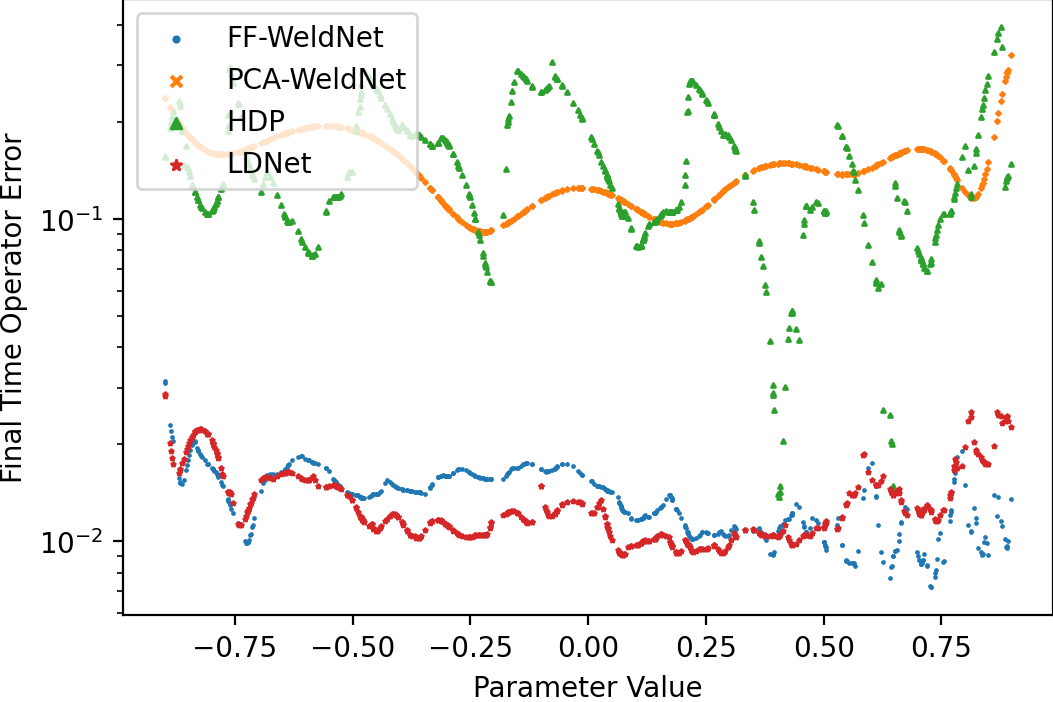}
    \caption{Error vs Parameter value for   $\cM_\text{bscale}$ about the trajectory manifold of the Burgers' equation \eqref{eq:burgers} with initial conditions in \eqref{eq:bscaleinitial}. }
    \label{fig:bscaleerrorparams}
  \end{minipage}
  \hfill
  \begin{minipage}[t]{0.48\textwidth}
    \centering
    \includegraphics[width=0.95\textwidth]{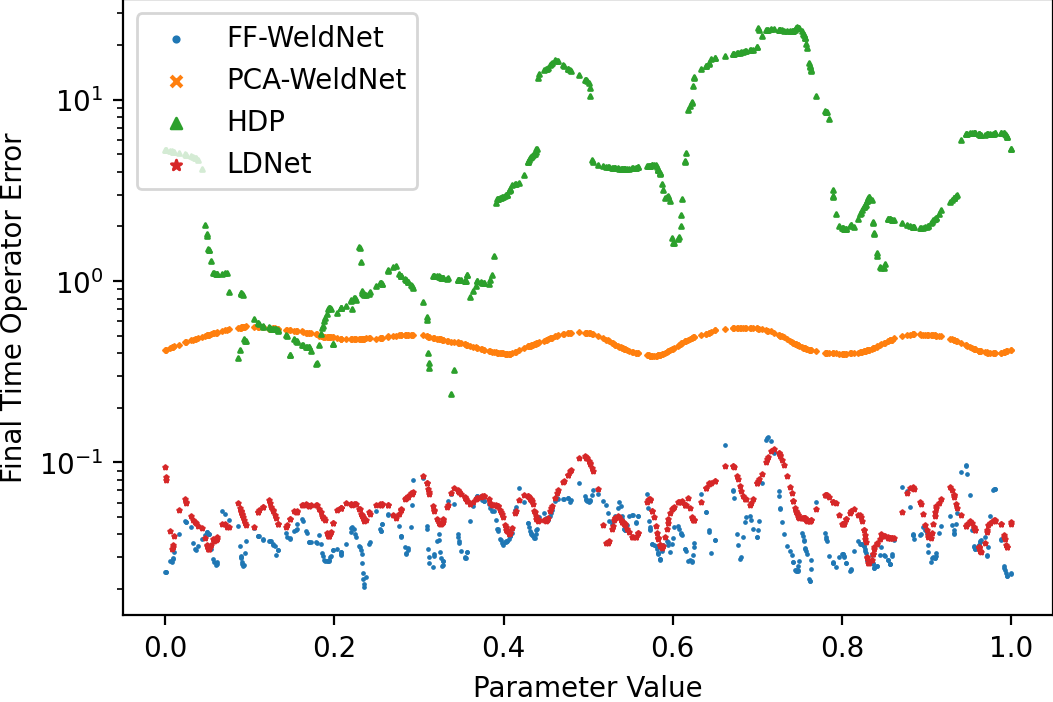}
    \caption{Error vs Parameter value for $\cM_\text{bshift}$ about the trajectory manifold of the Burgers' equation \eqref{eq:burgers} with initial conditions in \eqref{eq:bshiftinitial}.}
    \label{fig:bshifterrorparams}
  \end{minipage}
\end{figure}

For both of the trajectory manifolds $\cM_\text{bscale}$ and $\cM_\text{bshift}$, we can compare the reconstruction error and the operator/prediction error at each time. Figures \ref{fig:bscalepo} and \ref{fig:bshiftpo} are line plots of error versus time for FF-WeldNet, Conv-WeldNet, and PCA-WeldNet for $\cM_\text{bscale}$ and $\cM_\text{bshift}$, respectively. For this example, FF-WeldNet and Conv-WeldNet with latent dimension 4 outperforms PCA-WeldNet with latent dimension 4, since the solution manifold are nonlinear in these examples.

\begin{figure}
    \centering
    \includegraphics[width=0.6\linewidth]{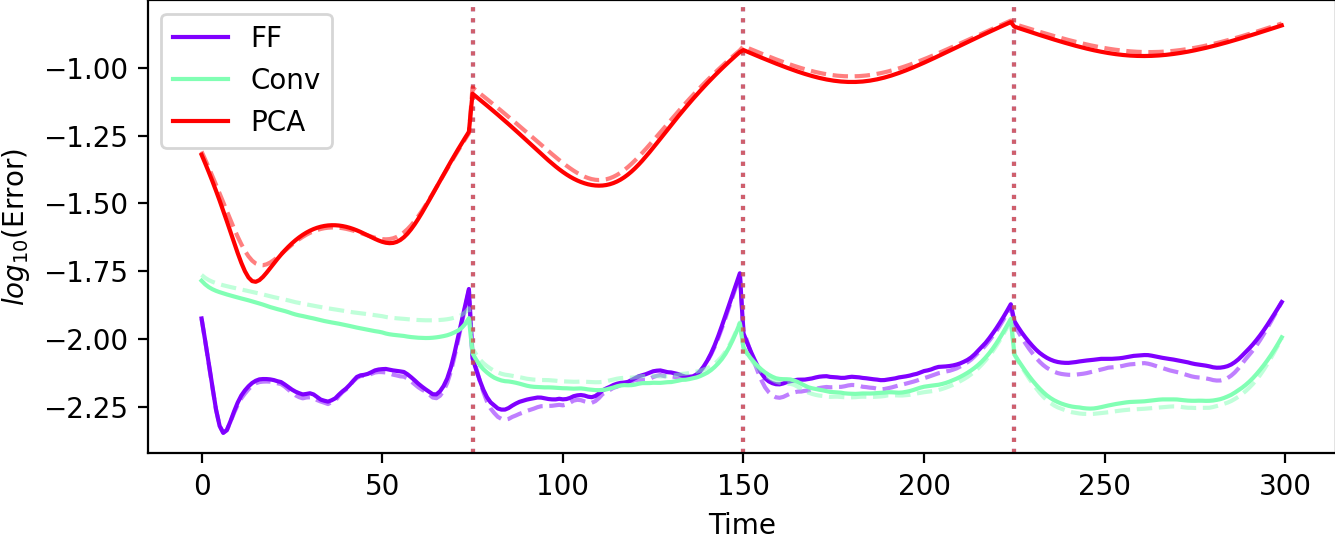}
    \caption{Autoencoder error (dashed) and Test error for $4$-window WeldNet models for $\cM_\text{bscale}$. FF refers to FF-WeldNet, Conv refers to Conv-WeldNet, and PCA refers to PCA-WeldNet. The latent space dimension is $4$.}
    \label{fig:bscalepo}
\end{figure}

\begin{figure}
    \centering
    \includegraphics[width=0.6\linewidth]{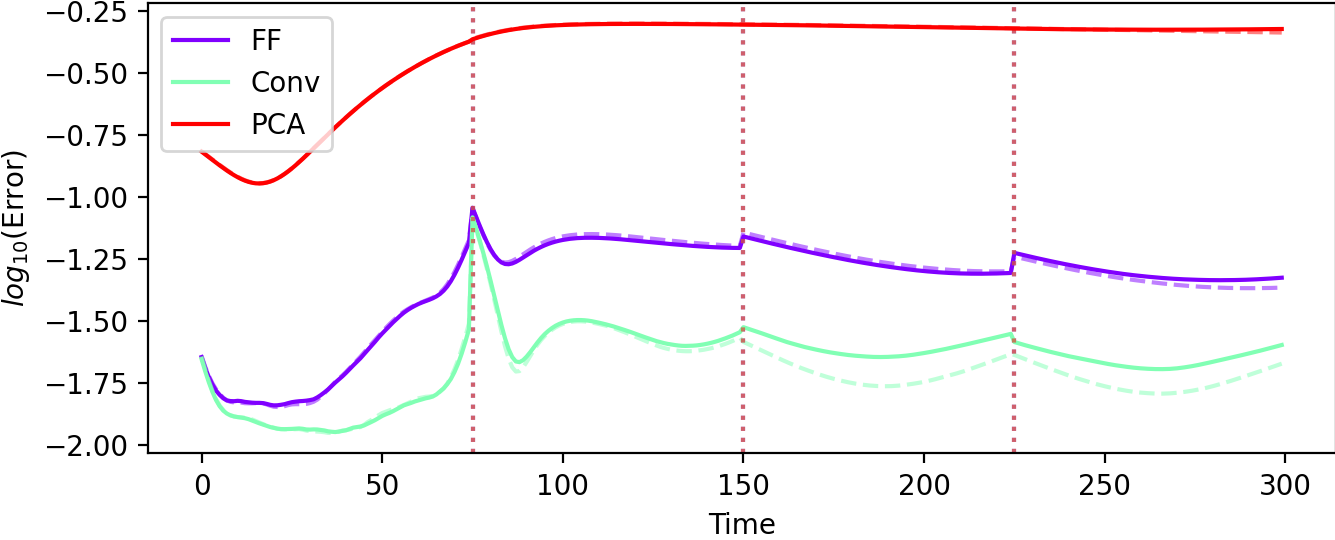}
    \caption{Autoencoder error (dashed) and Test error for $4$-window WeldNet models for $\cM_\text{bshift}$.}
    \label{fig:bshiftpo}
\end{figure}

\subsubsection{Transport Equation}
We next consider the transport equation given by \eqref{eq:transport1d}. We consider initial conditions given by $\hat{g}_\text{tscale}$ in \eqref{eq:tscale} and
\begin{align}
\hat{g}_\text{tshift}(x) &= \{H_{0.05}(x-0.1) + 2.5 \cdot H_{0.05}(x-(0.2 + 0.1h)) : h \in [0, 3] \} 
\label{eq:tshiftinitial}
\end{align}
and collect the trajectory manifold $\cM_\text{tscale}$ and $\cM_\text{tshift}$ from the corresponding condition sets. We use a latent space dimension of \textbf{four} for this example. We compare the performance of four window WeldNet models with other models in terms of final time prediction error in Figures \ref{fig:tscaleerrorparams} and \ref{fig:tshifterrorparams} respectively. WLaSDI has a very high relative operator error, so it does not appear in the figure to allow for clarity of scaling. The transport equation is a representative example of advection-dominated PDEs, which generates a highly nonlinear solution manifold. In Figures \ref{fig:tscaleerrorparams} and \ref{fig:tshifterrorparams}, FF-WeldNet with latent dimension 4 outperforms PCA-WeldNet with latent dimension 4, HDP and LDNet.

\begin{figure}[htbp]
  \begin{minipage}[t]{0.48\textwidth}
    \centering
    \includegraphics[width=0.95\textwidth]{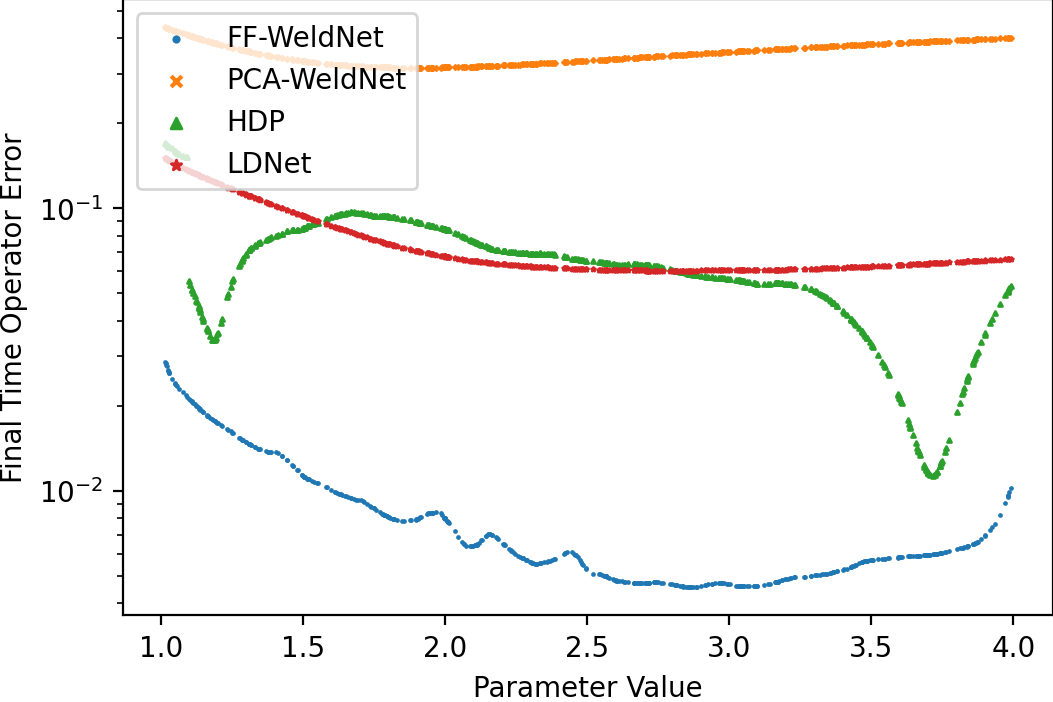}
    \caption{Error vs Parameter value for $\cM_\text{tscale}$ about the trajectory manifold of the transport equation \eqref{eq:transport1d} with initial conditions in \eqref{eq:tscale}.}
    \label{fig:tscaleerrorparams}
  \end{minipage}
  \hfill
  \begin{minipage}[t]{0.48\textwidth}
    \centering
    \includegraphics[width=0.95\textwidth]{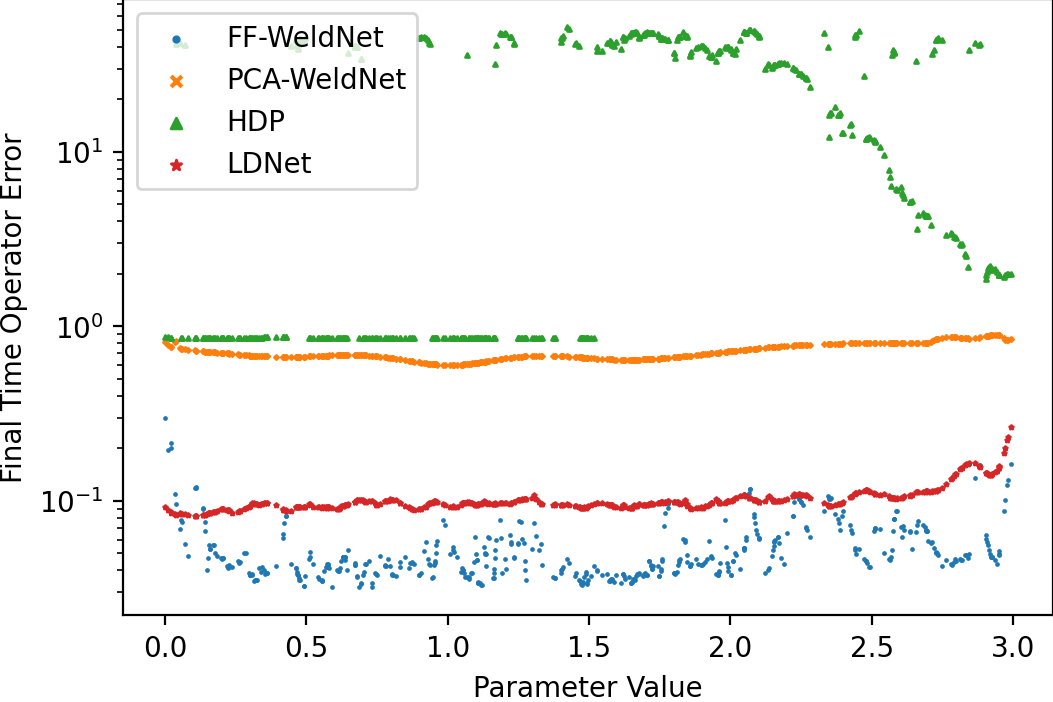}
    \caption{Error vs Parameter value for $\cM_\text{tshift}$ about the trajectory manifold of the transport equation \eqref{eq:transport1d} with initial conditions in \eqref{eq:tshiftinitial}.}
    \label{fig:tshifterrorparams}
  \end{minipage}
\end{figure}

For both of the trajectory manifolds $\cM_\text{tscale}$ and $\cM_\text{tshift}$, we can compare the reconstruction error and the operator error at each time. Figures \ref{fig:tscalepo} and \ref{fig:tshiftpo}, are line plots of error versus time for FF-WeldNet, Conv-WeldNet, and PCA-WeldNet for $\cM_\text{tscale}$ and $\cM_\text{tshift}$, respectively. In Figures \ref{fig:tscalepo} and \ref{fig:tshiftpo}, Conv-WeldNet achieves the lowest error, followed by FF-WeldNet. Both models substantially outperform PCA-WeldNet for the same latent space dimension.

\begin{figure}
    \centering
    \includegraphics[width=0.6\linewidth]{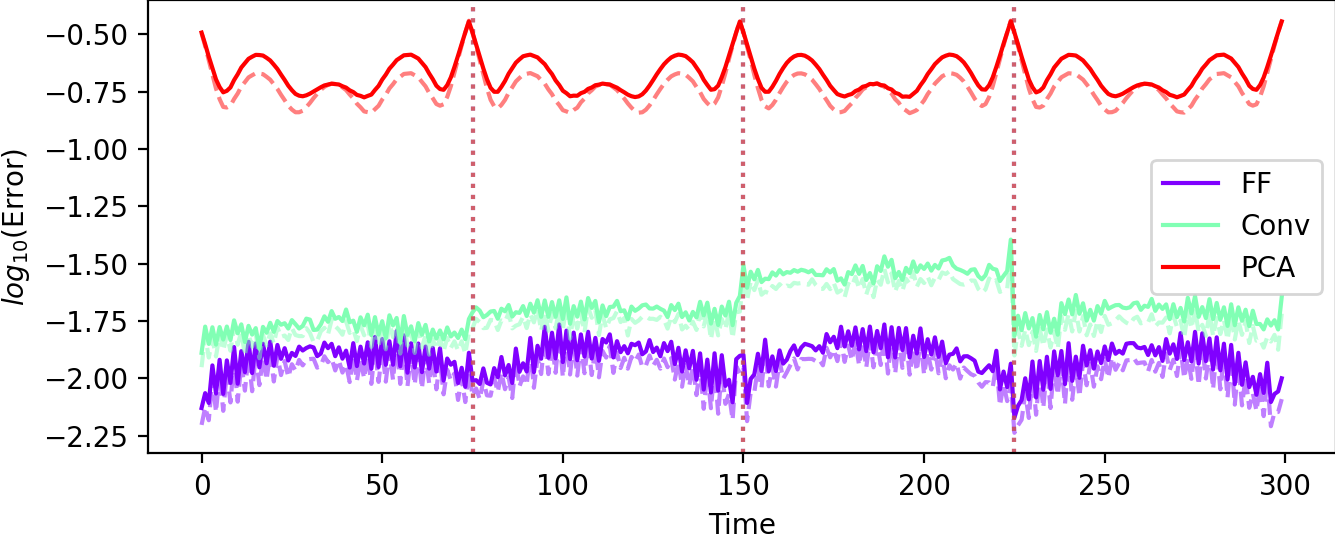}
    \caption{Autoencoder error (dashed) and Test error for  $4$-window WeldNet models for $\cM_\text{tscale}$.}
    \label{fig:tscalepo}
\end{figure}

\begin{figure}
    \centering
    \includegraphics[width=0.6\linewidth]{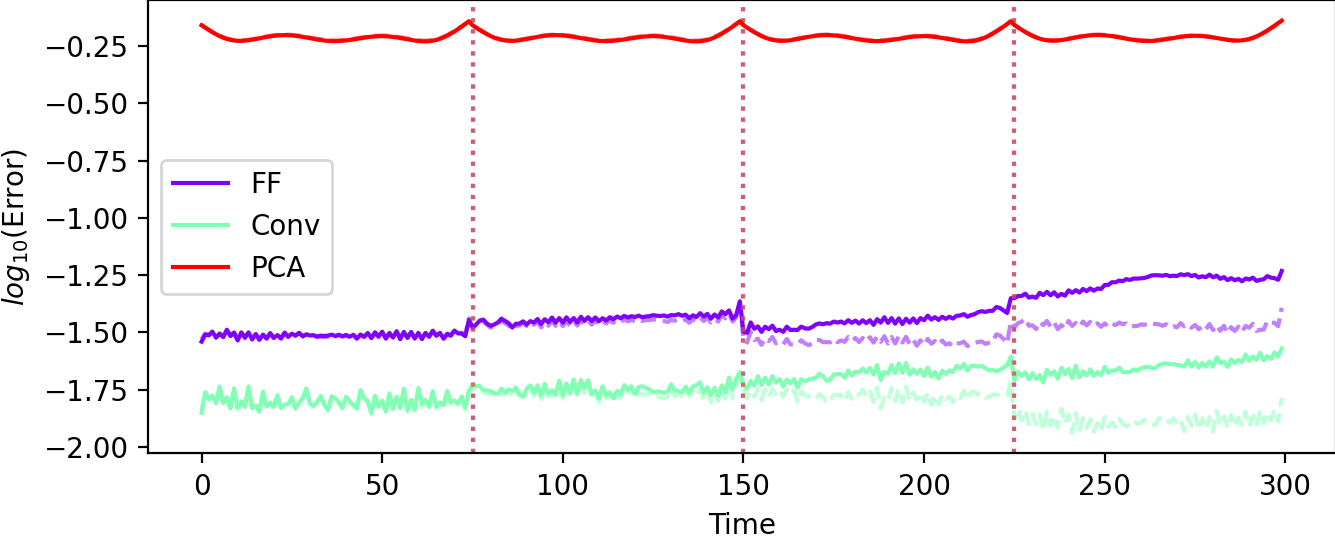}
    \caption{Autoencoder error (dashed) and Test error for  $4$-window WeldNet models for $\cM_\text{tshift}$.}
    \label{fig:tshiftpo}
\end{figure}

\subsubsection{KdV Equation} 
\label{sec:kdv}
We also consider the Korteweg–De Vries (KdV) equation for $T = 0.01$, given by
\begin{align}
\label{eq:kdv}
u_t = -u_{xxx} - uu_x, \quad u(x, 0) = g(x), \: x \in (0, 6).
\end{align}
We describe our sets of initial conditions. Given a constant $c$, we define the soliton of size $c$ as 
$\phi_c(x) = {c}/{2} \text{sech}\left({\sqrt{c}} x/2\right)^2. $
We consider initial conditions: 
\begin{align} 
\hat{g}_\text{kscale} = \{&\phi_{a^2} (x-1) +\phi_{36}(x - 2) : a \in [6, 18]\},
\label{eq:kscaleinitial}
\\
\hat{g}_\text{kshift} = \{&\phi_{36} (x-1) +\phi_{36}(x - 2 - h) : h \in [0, 0.4]\},
\label{eq:kshiftinitial}
\end{align}
and collect the trajectory manifold $\cM_\text{kscale}$ and $\cM_\text{kshift}$ from the corresponding condition sets. We compare the performance of four-window WeldNet models (with latent space dimension four) with other models in terms of final time operator error in Figure \ref{fig:kscaleerrorparams} and \ref{fig:kshifterrorparams} respectively. WLaSDI has a very high relative operator error, so most of it is cut off from the figure to allow for clarity of scaling. For both examples, FF-WeldNet is the best model, outperforming PCA-WeldNet, HDP, and LDNet for almost all parameter values.

\begin{figure}[htbp]
  \begin{minipage}[t]{0.48\textwidth}
    \centering
    \includegraphics[width=0.95\textwidth]{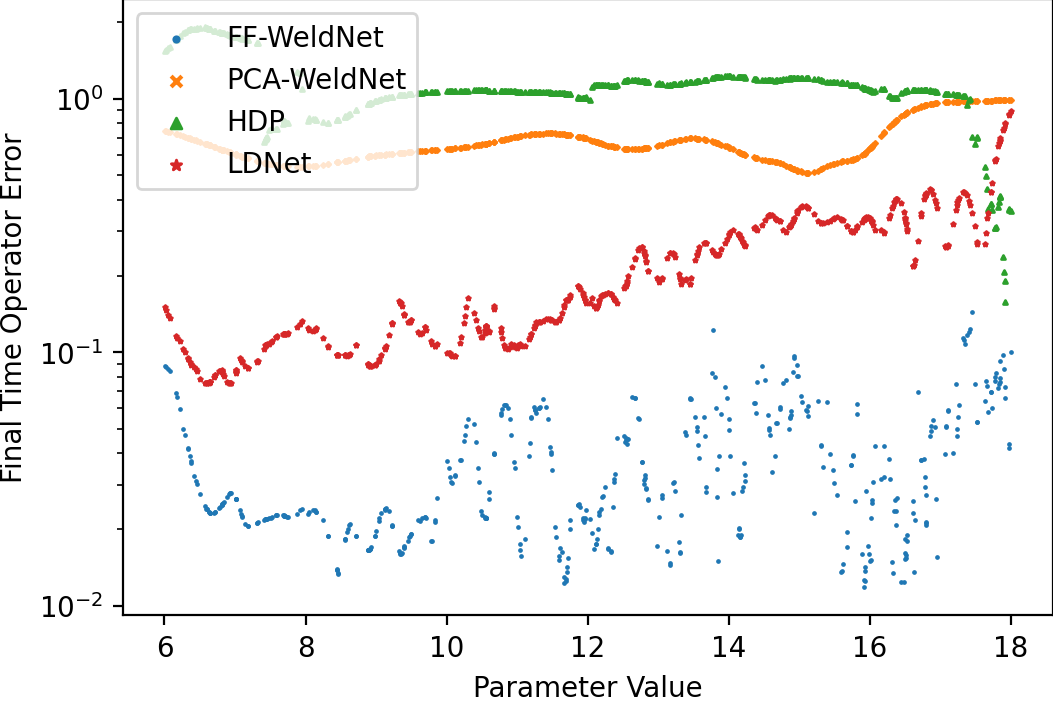}
    \caption{Error vs Parameter value for $\cM_\text{kscale}$ about the trajectory manifold of the KdV equation \eqref{eq:kdv} with initial conditions in \eqref{eq:kscaleinitial}.}
    \label{fig:kscaleerrorparams}
  \end{minipage}
  \hfill
  \begin{minipage}[t]{0.48\textwidth}
    \centering
    \includegraphics[width=0.95\textwidth]{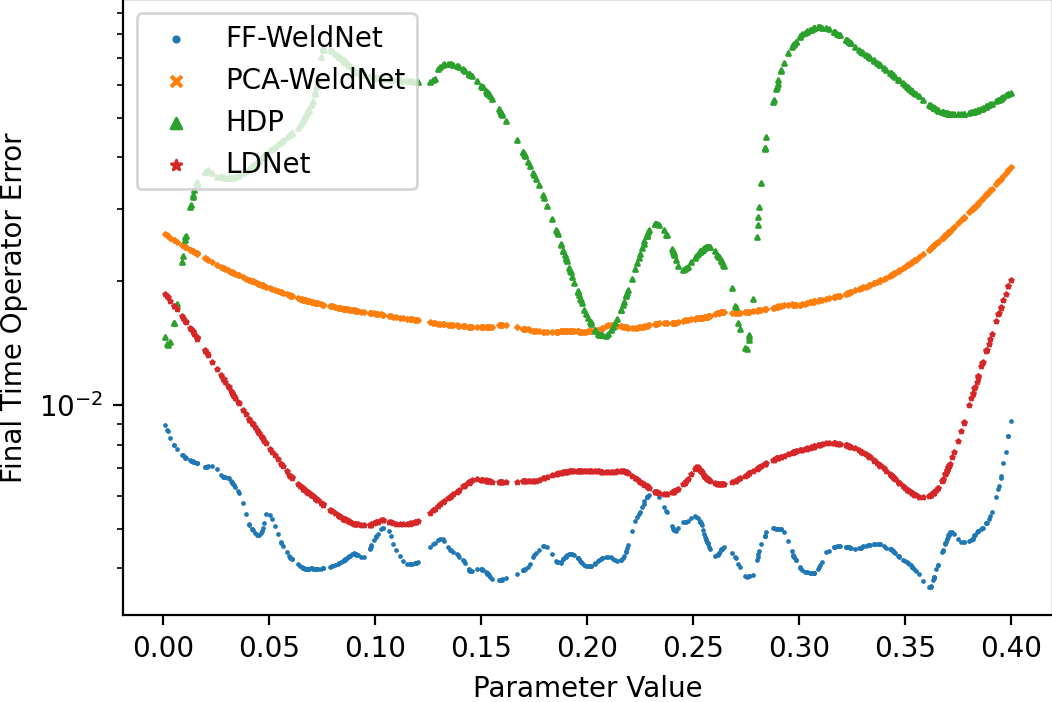}
    \caption{Error vs Parameter value for $\cM_\text{kshift}$ about the trajectory manifold of the KdV equation \eqref{eq:kdv} with initial conditions in \eqref{eq:kshiftinitial}.}
\label{fig:kshifterrorparams}
  \end{minipage}
\end{figure}

For both of the trajectory manifolds $\cM_\text{kscale}$ and $\cM_\text{kshift}$, we can compare the reconstruction error and the operator error at each time. Figures \ref{fig:kscalepo} and \ref{fig:kshiftpo}, are line plots of error versus time for FF-WeldNet, Conv-WeldNet, and PCA-WeldNet for $\cM_\text{kscale}$ and $\cM_\text{kshift}$, respectively. FF-WeldNet and Conv-WeldNet perform similarly on this problem, and they both greatly outperform PCA-WeldNet, showing the advantage of nonlinear dimension reduction.

\begin{figure}
    \centering
    \includegraphics[width=0.6\linewidth]{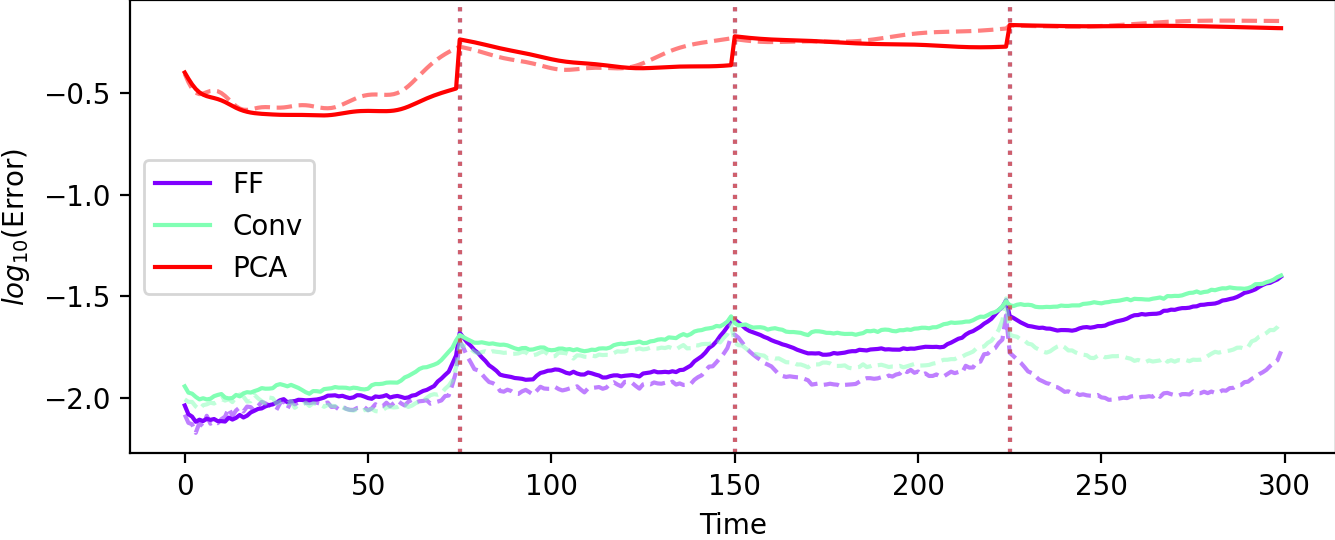}
    \caption{Autoencoder error (dashed) and Test error for WeldNet models for $\cM_\text{kscale}$.}
    \label{fig:kscalepo}
\end{figure}

\begin{figure}
    \centering
    \includegraphics[width=0.6\linewidth]{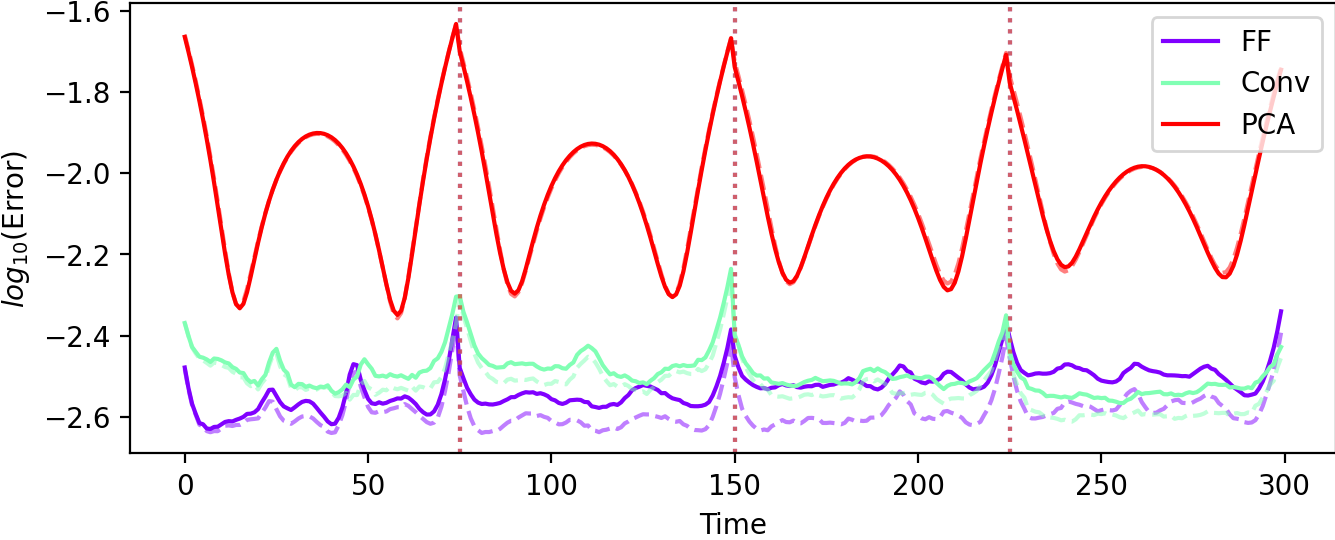}
    \caption{Autoencoder error (dashed) and Test error for WeldNet models for $\cM_\text{kshift}$.}
    \label{fig:kshiftpo}
\end{figure}

\subsubsection{Shallow-Water Equations}
We now consider an example with two spatial dimensions. Consider the shallow-water equations over the spatial domain $[-2.5, 2.5]^2$ and time domain $[0, 1]$:
\begin{align}
\label{eq:shallow}
\partial_t h + \partial_x (hu) + \partial_y (hv) &= 0,\\
\partial_t (hu) + \partial_x \left(u^2 h + \tfrac{1}{2} g_r h^2 \right) + \partial_y (uvh) &= -g_r h \, \partial_x b,\nonumber \\
\partial_t (hv) + \partial_y \left(v^2 h + \tfrac{1}{2} g_r h^2 \right) + \partial_x (uvh) &= -g_r h \, \partial_y b,\nonumber
\end{align}
where $u$ is horizontal velocity, $v$ is vertical velocity, $h$ is the water depth, $g_r$ is the gravitational acceleration, and $b$ is a scalar field known as bathymetry. We impose zero dirichlet boundary conditions for $u$ and $v$. 

We consider the following set of initial conditions representing centered bumps of varying radii (provided by the PDEBench dataset \cite{DARUS-2986_2022, takamoto2022pdebench}):
\begin{equation}
\label{eq:shallowinitial}
\hat{g}_\text{shallow} = \left\{h(x,y) =1 + \mathbbm{1}[\sqrt{x^2 + y^2} < r] : r \in [0.3, 0.7]\right\}
\end{equation}
and collect the trajectory manifold $\cM_\text{shallow}$. Note that the initial condition is discontinuous, and the location of the discontinuity depends on the sample. For this example only, we use a $128\times 128$ grid for each sample, and we use 101 time steps. We collect 81 examples.

Note that the inputs are sampled on a size 16384 grid. This makes the size of feedforward networks prohibitively high for training. While more sophisticated architectures such as two-dimensional convolutions exist, we will instead reduce the dimension from 16384 to 128 using PCA, and then train our models to predict the evolution of the PCA modes. We compute the top 128 principal components (over all times) using the training data only, and this results in a relative projection error of 0.11\%. 

Now we have 81 samples with 101 time steps and 128 features each. We train each of our models on this dataset. Even though the compression to 128 features is very accurate, it can be challenging to learn the dynamics. We compare the performance of the four-window WeldNet model (with latent space dimension of $4$) with other models in terms of final time operator error in Figure \ref{fig:shallowerrorparams}. WeldNet greatly outperforms all other models by almost an entire order of magnitude. 

\begin{figure}
    \centering
    \includegraphics[width=0.5\linewidth]{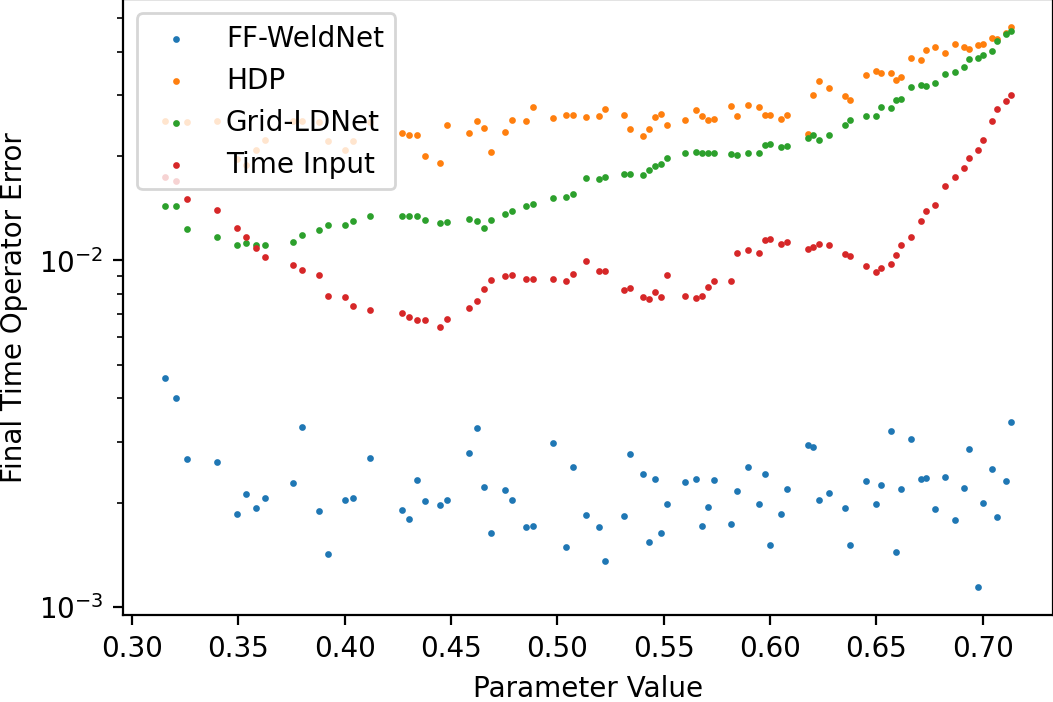}
    \caption{Error vs Parameter value for $\cM_{shallow}$ about the trajectory manifold of the shallow water equation \eqref{eq:shallow} with initial conditions in \eqref{eq:shallowinitial}.}
    \label{fig:shallowerrorparams}
\end{figure}

The lower error for WeldNet stems from the fact that the trajectory manifold is harder to represent at earlier times than later times, and WeldNet is able to separate the latent dynamics learning in each of its segments. The transcoder is able to lower the error as a code is evolved between windows. On the other hand, there is no error reduction mechanism in the other dynamics learning models, so the error only increases with time. We can see this clearly in Figure \ref{fig:shallowoperror}, a plot of the operator error versus time for various models. This plot shows that FF-WeldNet has improved performance for later times, while the error stagnates or increases with time on the other models. This shows the advantage of having multiple windows; a decoupled representation of this trajectory manifold outperforms a fully coupled one. 

\begin{figure}
    \centering
    \includegraphics[width=0.6\linewidth]{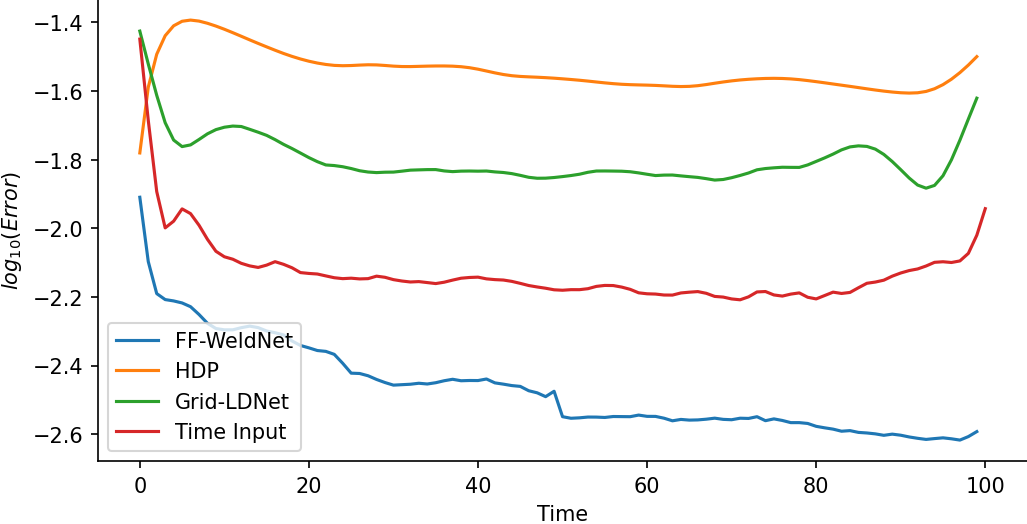}
    \caption{Error vs Parameter value for $\cM_{shallow}$ about the trajectory manifold of the shallow water equation \eqref{eq:shallow} with initial conditions in \eqref{eq:shallowinitial}.}
    \label{fig:shallowoperror}
\end{figure}

\subsection{Intrinsic Dimension Estimation}
We estimate the intrinsic dimensionality of our datasets using the Scikit-dimension~\cite{bac2021scikit} Python package and display it in Table \ref{table:datasetid}. Specifically, we use the MLE algorithm \cite{levina2004maximum} with the ``Haro'' integral approximation \cite{haro2008translated} and the TwoNN algorithm \cite{facco2017estimating} with the parameter discard$\_$fraction$=0$. We use a random size 50\,000 subset for computing the intrinsic dimension of all times together. By construction, the initial conditions for all considered datasets are generated by sampling one random scalar, so we expect the intrinsic dimensionality of each time slice to be approximately 1 and the intrinsic dimensionality of all times together to be approximately 2. This justifies the use of dimension reduction methods for these datasets and indicates that a small latent space dimension can effectively capture the data.

For the shallow water dataset, the relatively high intrinsic dimensionality for the initial time and all times together likely comes from the discontinuity of the initial condition in \eqref{eq:shallow}. The estimated intrinsic dimensionality for a continuous approximation to the initial conditions can be computed, and we empirically found that it was about 1 regardless of the specific construction used (as long as the functions are continuous).

\label{sec:intrinsicdim}
\begin{table}[]
\centering
\begin{tabular}{|c|l|l|l|l|l|l|l|}
\hline
Dataset & t=0   & t=60  & t=120 & t=180 & t=240 & t=300 & All t \\ \hline
bscale  & 0.973 & 0.974 & 0.974 & 0.976 & 0.978 & 0.977 & 1.89  \\ \hline
bshift  & 0.998 & 1.019 & 1.128 & 1.091 & 1.065 & 1.048 & 1.685 \\ \hline
tscale  & 0.976 & 0.976 & 0.976 & 0.976 & 0.976 & 0.976 & 1.118 \\ \hline
tshift  & 1.013 & 1.013 & 1.012 & 1.013 & 1.013 & 1.012 & 1.973 \\ \hline
kscale  & 0.945 & 0.945 & 0.945 & 0.945 & 0.945 & 0.945 & 1.803 \\ \hline
kshift  & 0.995 & 0.995 & 0.995 & 0.995 & 0.995 & 0.995 & 1.901 \\ \hline
shallow  & 2.184 & 1.282 & 1.257 & 1.264 & 1.244 & 1.231 & 2.593 \\ \hline
\end{tabular}
\caption{Estimated intrinsic dimensionality of all datasets for various times and as a whole. MLE is used for all datasets except for kscale, for which TwoNN was used (see Appendix \ref{sec:intrinsicdim} for details).}
\label{table:datasetid}
\end{table}

\subsection{Propagator Ablation Study}
\label{sec:ablate}
WeldNet first trains the autoencoder and propagator together using the propagator displacement loss in \eqref{eq:optprop} (where the name displacement comes from being one step of applying propagator) and then finetunes the propagator using the accumulation loss in \eqref{eq:propfinetune}. Alternatively, we could train the autoencoder and propagator separately, or we could use only a single loss function for the propagator. We compare our Weldnet model training algorithm against three variants.

\begin{enumerate}[label=(\roman*)]
    \item ``WeldNet'' (original) - train autoencoder and propagator together with displacement loss, then finetune the propagator with accumulation loss.
    \item ``Together / Displacement'' - same as (i) but finetune with displacement loss.
    \item ``Separate / Accumulate'' - train autoencoder first, then train propagator with accumulation loss.
    \item ``Separate / Displacement'' - same as (iii) but with the displacement loss.
\end{enumerate}

We train autoencoder, propagator, and transcoder models using each of the above four training algorithms. For algorithms (i) and (ii), we jointly train the autoencoder and propagator for 300 epochs and then finetune the propagator for 150 epochs. For algorithms (iii) and (iv), we train the autoencoder for 300 epochs and then the propagator for 450 epochs. We train any transcoders for 300 epochs. All other settings are identical to the main paper. We use the $kshift$ dataset (for the KdV equation) described in Section \ref{sec:kdv}.

Figures \ref{fig:ablateloss1} and \ref{fig:ablateloss2} compare the relative test error of one-window and two-window models (respectively) trained with each algorithm. The final time test error for each one-window model is (in order): 0.87\%, 1.28\%, 1.31\%, 20.4\%, and the average test error for each two-window model is (in order): 0.42\%, 0.45\%, 0.50\%, 1.3\%. While increasing the number of windows reduces the gap between the models, the best model is (i) which is the training algorithm used in the main paper.

\begin{figure}
    \centering
    \includegraphics[width=0.7\textwidth]{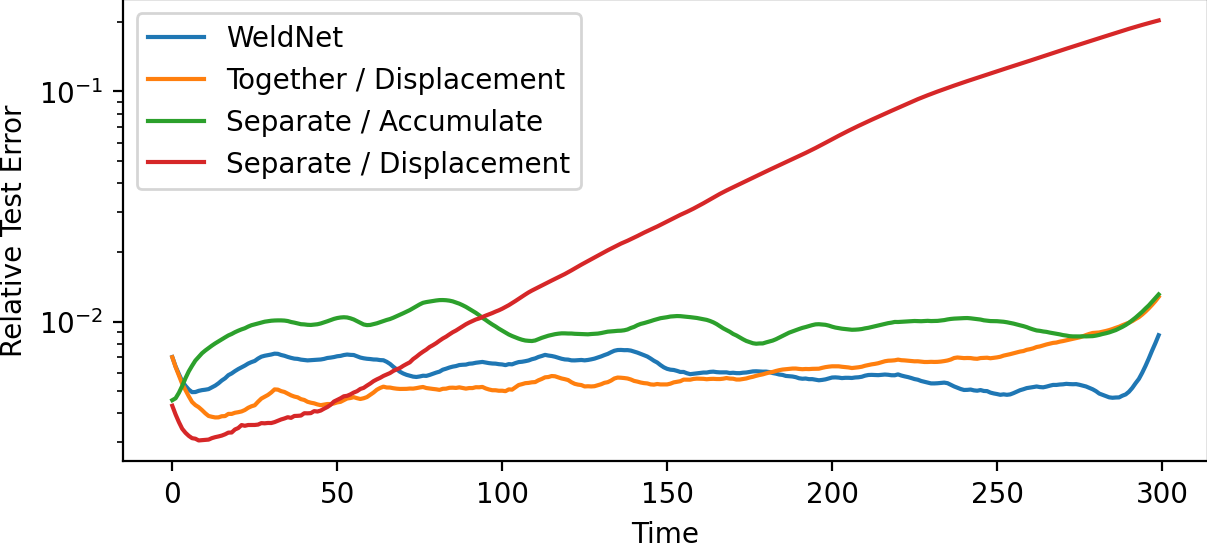}
    \caption{Test error vs time for variant training algorithms on $\cM_{kshift}$ (1 window models).}
    \label{fig:ablateloss1}
\end{figure}
    
\begin{figure}
    \centering
    \includegraphics[width=0.7\textwidth]{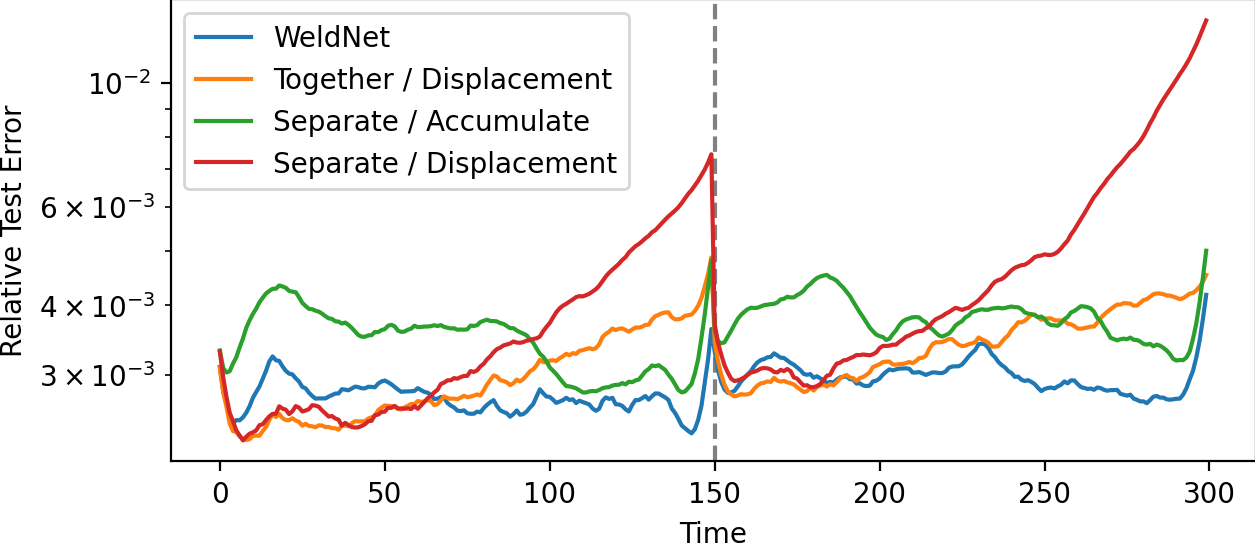}
    \caption{Test error vs time for variant training algorithms on $\cM_{kshift}$ (2 window models).}
    \label{fig:ablateloss2}
\end{figure}

\section{Conclusion}
\label{sec:conclusion}


This paper introduces WeldNet, a data-driven framework for nonlinear model reduction that constructs low-dimensional surrogate models for complex evolutionary systems. The architecture consists of three key components: autoencoders, propagators, and transcoders. These components operate together to model dynamics in reduced latent spaces. By transferring the evolution of high-dimensional systems into these latent spaces and decomposing the time domain into overlapping windows, WeldNet effectively captures intricate nonlinear structures while simplifying long-time propagation into a sequence of short, tractable segments.

In addition to the algorithmic design, a representation theory is developed to establish the approximation capability of WeldNet under the manifold hypothesis, thereby providing a mathematical foundation for nonlinear model reduction via deep learning. Extensive numerical experiments on a variety of differential equations demonstrate the robustness and accuracy of the proposed approach. Across all tested scenarios, WeldNet consistently achieves smaller prediction errors than both classical projection-based techniques and recently developed nonlinear reduced-order models. {In addition, WeldNet with parallel computing is more computationally efficient to train than other methods.}

Overall, this work shows that windowed latent-space learning offers a powerful and principled strategy for modeling complex dynamical systems, and it opens promising ways for future research in nonlinear model reduction in scientific machine learning.

\section*{Acknowledgment}

The authors would like to thank Professor Peng Chen and his group at Georgia Tech for their insightful comments and suggestions, which helped improve this paper from its earlier version.

\bibliographystyle{siamplain}
\bibliography{references}

\appendix

\section{Helper Approximation Results}
\label{app:helperappx}
In this section, we introduce some approximation results for ReLU networks. They are later used to prove our main results.

First, we establish that ReLU networks (with a given size) can implement the identity function.

\begin{proposition}
\label{prop:idappx}
Let $d, L, W \in \bN$, with  $W \geq 2d$. Then there is a $f_\NN \in \cF_\NN(d, d, L, W)$ such that $f_\NN(\bfx) = \bfx$ for all $\bfx \in \bR^d$.
\end{proposition}
\begin{proof}[Proof of Proposition \ref{prop:idappx}]
We denote by $I_d$ the identity function on $\bR^d$. First, suppose $L=1$ and $W=2d$. Consider the weights and biases $\mathfrak{W}_1 = \begin{bmatrix} I_d \\ -I_d  \end{bmatrix}$, $\mathfrak{b}_1 = \vec{0}_{2d}$, $\mathfrak{W}_2 = \begin{bmatrix} I_d & -I_d \end{bmatrix}$, and $\mathfrak{b}_2 = \vec{0}_d$. Then the one-hidden-layer ReLU network with the above weights and biases is in $\cF_\NN(d, d, 1, 2d)$ and exactly implements the identity (since $\sigma(x)-\sigma(-x) = x$ for $\sigma(x) = \max\{x, 0\}$).

Next, if $L > 1$ and $W=2d$, define $\mathfrak{W}_1 = \begin{bmatrix} I_d \\ -I_d  \end{bmatrix}$ and $\mathfrak{W}_{L+1} = \begin{bmatrix} I_d & -I_d \end{bmatrix}$. Also for any $1 < \ell < L+1$, let $\mathfrak{W}_\ell = I_{2d}$. Let all bias vectors be zero. Then the $L$-hidden-layer ReLU network with the above weights and biases is in $\cF_\NN(d, d, L, 2d)$ and exactly implements the identity (as the output of the $L$th hidden layer is the same as the output of the first hidden layer).

Finally, suppose $W > 2d$. It is easy to see that augmenting weight matrices $\mathfrak{W}_1, \cdots, \mathfrak{W}_L$ defined above with zeros in all new rows and columns (with the original weight matrix being in the top left corner) does not affect output of the neural network. Thus we have constructed an $L$-hidden-layer ReLU network with width exactly $W$ that implements the identity in $\bR^d$. 
\end{proof}

\begin{proposition}
\label{prop:appxtools}
Let $f_{\NN} \in \cF_{\NN}(d_{in}, d_{out}, L, W)$ be a feedforward ReLU network. Suppose $W > d_{out}$.
\begin{enumerate}
\item For any $L' > L$, there exists a FNN $g_{\NN} \in \cF_{\NN}(d_{in}, d_{out}, L', 2W)$ with depth exactly $L'$ such that $g_\NN(\bfx) = f_\NN(\bfx)$ for all $\bfx \in \bR^{d_1}$.
\item Let $A = \prod_{i=1}^{d_{out}} [a_i, b_i]$ be a cube. Then there is a $h_{\NN} \in \cF_{\NN}(d_{in}, d_{out}, L+2, \max\{d_{out}, W\})$ such that for all $\bfx \in \bR^{d_{in}}$, $h_{\NN}(\bfx)$ is vector valued and $= (h_{\NN}^{(1)}(\bfx), \dots, h_{\NN}^{(d_{out})}(\bfx)) \in \bR^{d_{out}}$ with $h_{\NN}^{(i)}(\bfx) = \min
\{\max\{f^{(i)}_{\NN}(\bfx), a_i\}, b_i\}$.  
\item Let $f^1_\NN, \dots, f^m_\NN$ be FNNs such that $f^j_\NN \in \cF_\NN(d_{in}, d_{out}, L_j, W_j)$ for all $j \in [m]$. Then there is a $g_\NN \in \cF_\NN(d_{in}, d_{out}, \sum_{j=1}^m L_j, 2 d_{in} + 2 d_{out} + \max_j W_j)$ such that for all $\bfx \in \bR^{d_{in}}$, we have $g_\NN(\bfx) = \sum_{j=1}^m f_\NN^j(\bfx)$. 
\end{enumerate}
\end{proposition}

Proposition \ref{prop:appxtools} gives rise to basic properties of feedforward ReLU networks. The first property is that, any neural network can be realized by a deeper one; The second property says that, feedforward ReLU networks allows one to truncate or restrict its output to be in the set of $A = \prod_{i=1}^{d_2} [a_i, b_i]$; The third property allows one to concatenate several feedforward ReLU networks and then output its summation.

\begin{proof}[Proof of Proposition \ref{prop:appxtools}]
{\bf Part 1}:
 Denote the number of neurons in the last hidden layer of $f_\NN$ by $r$ (note $r \leq W$). Let $L_f$ be the number of layers of $f_\NN$, and by the hypothesis we have $L_f < L'$.

First, consider the function $\phi_1 : \bR^r \rightarrow \bR^{2r}$ defined by $\phi_1(x_1, \dots, x_r) = [\sigma(x_1),$\\$ \sigma(-x_1), \dots, \sigma(x_r), \sigma(-x_r)]$. Clearly, $\phi_1$ can be implemented by a single ReLU layer. Next, consider the function $\phi_2$ on $\bR^{2r}$ defined by $\phi_2(x_1, \dots, x_{2r}) = [\sigma(x_1), \dots, \sigma(x_{2r})]$. Clearly, $\phi_2$ can also be implemented by a single ReLU layer. Note that $\bigcirc_{i=1}^n \phi_2 = \phi_2$ for any $n \in \bN$. We will construct $g_\NN$ from $f_\NN$ by appending a hidden layer implementing $\phi_1$, appending $L' - L_f - 1$ hidden layers implementing $\phi_2$, and then adjusting the weights in the final layer of $f_\NN$.

Denote the final weight matrix of $f_\NN$ by $\mathfrak{W} \in \bR^{r \times d_{out}}$. We will adjust this matrix to handle inputs of size $2r$. Define the matrix $\hat{\mathfrak{W}} \in \bR^{2r \times d_{out}}$ by $\hat{\mathfrak{W}}_{i,j} = (-1)^{j+1} \mathfrak{W}_{i,\lceil j/2\rceil}$, for all $i \in [d_{out}], j \in [2r]$. This means $\hat{\mathfrak{W}}(x_1, x_2, \dots, x_{2r-1}, x_{2r}) = \mathfrak{W}(x_1 - x_2, \dots, x_{2r-1} - x_{2r})$.

Note that for all $i \in [r]$, we have that $\sigma(x_i) - \sigma(-x_i)= x_i$. This means (where we use denote multiplication vector multiplication by parenthesis)

\begin{align*}
    \hat{\mathfrak{W}}& \left((\bigcirc_{i=1}^{L' - L_f - 1} \phi_2) \circ \phi_1(x_1, \dots, x_r)\right) \\ 
    &= \hat{\mathfrak{W}}(\sigma(x_1), \sigma(-x_1), \dots, \sigma(x_r), \sigma(-x_r)) = \mathfrak{W}(x_1 \dots, x_r).
\end{align*}

Finally, we can describe $g_\NN$. Before the final weight matrix of $f_\NN$, we will append one layer implementing $\phi_1$ and $L' - L_f - 1$ layers implementing $\phi_2$ to the network. Then, we change the final weight matrix to $\hat{\mathfrak{W}}$. Thus, $g_\NN$ and $f_\NN$ implement the same function, and $g_\NN \in \cF(d_{in}, d_{out}, L', 2W)$.

{\bf Part 2}: Note that the function $b - \sigma(b - a - \sigma(x - a)) = \min(\max(x, a), b).$ 

Let $\mathfrak{b} \in \bR^{d_{out}}$ be the bias vector for the final layer. Define a new bias vector $\tilde{\mathfrak{b}} \in \bR^{d_{out}}$ by $\tilde{\mathfrak{b}}_i = \mathfrak{b}_i - a_i$, for all $i \in [d_{out}]$. 


If we denote $\mathbf{a} = [a_1, \dots, a_{d_{out}}]$ and $\mathbf{b} = [b_1, \dots, b_{d_{out}}]$, then note that for all $x \in \bR^{d_{in}}$,

\begin{equation}
\label{eq:maxmin}
-\sigma(-(\sigma(f_\NN(x)-\mathbf{a}) + \mathbf{b} - \mathbf{a}) + \mathbf{b} = \min(\max(f_\NN(x), \mathbf{a}), \mathbf{b}),
\end{equation}
where $\sigma$ and $\max,\min$ are applied componentwise. We can implement the LHS of Equation \eqref{eq:maxmin} starting from $f_\NN$ using the following steps (where $I_{d_{out}}$ is the identity matrix for $\bR^{d_{out}}$):

\begin{enumerate}
    \item Modify the last bias vector of $f_\NN$ to $\tilde{\mathfrak{b}}$. This implements $f_\NN(x)-\mathbf{a}$.
    \item Add a new layer with weight matrix equal to $-I_{d_{out}}$ and bias vector equal to $\mathbf{b} - \mathbf{a}$. This implements $-\sigma(f_\NN(x)-\mathbf{a}) + \mathbf{b} - \mathbf{a}$.
    \item Add a new layer with weight matrix equal to $-I_{d_{out}}$ (identity) and bias equal to $\mathbf{b}$. This implements the LHS of \eqref{eq:maxmin}.
\end{enumerate}

Thus we have constructed a neural network $h_\NN$ by modifying the bias vector in the final layer of $f_\NN$ and adding two more layers.

{\bf Part 3}: We add extra layers that store the running sum. For every network $f^j_\NN$, we add $2d_{in}$ extra neurons that store the positive and negative values of the inputs every layer. For the final layer of each network, we store the positive and negative values of the output with $2d_{out}$ neurons. 

We then concatenate each of these networks. We collect the running sum in the $2d_{out}$ neurons, and we preserve the input in the $2d_{in}$ extra intermediate neurons. The depth of this concatenated network is the sum of the depths of each network, and the width is the maximum width of the networks plus $2 d_{out} + 2 d_{in}$ neurons.

Each layer will need $2 d_{out}$ neurons to capture the positive and negative values of the $d_{out}$ outputs in the sum. Specifically, we modify the final layer of $f_\NN$ to have double the number of output neurons as $d_{out}$

\end{proof}

Now, let's consider to approximate H\"older functions using ReLU networks. 
For any function $f: \Omega \rightarrow \bR$, we define it $1$-H\"older norm:
$$
\|f\|_{\cH^1(\Omega;\bR)}=\max\left\{\sup_{x\in \Omega} |f(x)|, \ \sup_{x,y\in \Omega}\frac{|f(x)-f(y)|}{\|x-y\|_2}\right\}
$$
and 
\begin{align*}
    \cH^1(\Omega;\bR)=\left\{f: \|f\|_{\cH^1(\Omega;R)}< \infty\right\}.
\end{align*}

\begin{proposition}
\label{prop:lipsingle}
Let $f^* \in \cH^1([0,1]^d; \, \bR)$. Then, $\forall \sfM \in \mathbb{N}$, there is a $f_{\NN} \in \cF_{\NN}(d, 1, L, W)$ such that 

\[ \|f^* - f_{\NN} \|_{L^\infty([0,1]^d; \, \bR)} \leq C\sfM^{-\frac{1}{d}}. \]
Here $L = O(\log(\sfM))$ and $W = O(\sfM)$, where the big O and $C$ hides constants only depending on $d$ and $\|f^*\|_{\cH^1([0,1]^d; \, \bR)}$.
\end{proposition}
\begin{proof}
This is a restatement of \cite[Lemma 7]{oono2019approximation}.
\end{proof}

We extend Proposition \ref{prop:lipsingle} to functions with vector-valued outputs. 

\begin{proposition}
\label{prop:lipmultiple}
Let $f^* \in \cH^1([0,1]^{d_{in}}; \, A)$, where $A = \bR^{d_{out}}$ is the Euclidean space, or $A = \prod_{i=1}^{d_{out}} [a_i, b_i]$ is a cube. Then, $\forall \sfM \in \mathbb{N}$, $\exists f_{\NN} \in \cF_{\NN}(d_{in}, d_{out}, L, W)$ such that 

\[ \|f^* - f_{\NN} \|_{L^\infty([0,1]^{d_{in}}; \, \bR^{d_{out}})} \leq C\sfM^{-\frac{1}{d_{in}}}. \]
In addition, the range of $f_{\NN}$ is contained in $A$. Here $L = O(\log(\sfM))$ and $W = O(\sfM)$, where big $O$ and $C$ hide constants depending on $d_{in}$, $d_{out}$, and $\|f^*\|_{\cH^1([0,1]^{d_{in}}; \, \bR^{d_{out}})}$.
\end{proposition}
\begin{proof}[Proof of Proposition \ref{prop:lipmultiple}]
For all $i \in [d_{out}]$, we denote the $i$th component function of $f^*$ by $f^i$. Then since $f^i \in \cH^1([0,1]^{d_{in}}, \bR)$, by Proposition \ref{prop:lipsingle} there exists an FNN $f^i_{\NN} \in \cF_{\NN}(d_{in}, 1, L_i, W_i)$ such that $\|f^i - f_{\NN}^i\|_{L^\infty([0,1]^{d_{in}}; \, \bR)} \leq C_i \sfM^{-\frac{1}{d_{in}}}$ for some $C_i$ depending on $d$ and $\|f^i\|_{\cH^1([0,1]^{d_{in}}; \, \bR)}$. Let $L = \max_i L_i$, and $W = \sum_{i=1}^{d_{out}} W_i$. We will add extra layers implementing the identity so that each $f^i_{\NN}$ has exactly $L$ layers. This is done by applying Proposition \ref{prop:appxtools} Part 1, which will double the width of each $f^i_{\NN}$, leaving the term inside the big $O$ notation unchanged.

Now, define $f_{\NN}$ as the network implementing each of the $f^i_{\NN}$ as component functions. Then $f_{\NN} \in \cF_{\NN}(d_{in}, d_{out}, L, W)$ with $L = O(\log(\sfM))$ and $W = O(\sfM)$. For all $x \in [0,1]^{d_{in}}$ we estimate
\begin{align*}
\|f^*(x) - f_{\NN}(x)\|_{\bR^{d_{out}}}
    &= \left(\sum_{i=1}^{d_{out}} (f_i^*(x) - f^i_{\NN}(x) )^2 \right)^{\frac{1}{2}} \\
    &\leq \left(\sum_{i=1}^{d_{out}} C_i^2 \sfM^{-\frac{2}{d_{in}}}\right)^{\frac{1}{2}} \\
    &\leq \left(d_{out} \left(\max_i C_i^2\right) \sfM^{-\frac{2}{d_{in}}}\right)^{\frac{1}{2}} \\
    &= \sqrt{d_{out} \left(\max_i C_i^2\right)} \sfM^{-\frac{1}{d_{in}}}.
\end{align*}
Finally, if $A = \prod_{i=1}^{d_{out}} [a_i, b_i]$, we replace each $f^i_{\NN}$ by $\sigma(-\sigma(-f^i_{\NN}+b_i-a_i)+a_i)$ so that for any $x\in[0,1]^{d_{in}}$, $f^i_{\NN}(x)\in [a_i,b_i]$ and thus $f_{\NN}(x)\in A$.
\end{proof}

We use Proposition \ref{prop:lipmultiple} to approximate target functions up to $\epsilon$ accuracy. Setting the approximation error as $\epsilon$ in Proposition \ref{prop:lipmultiple} gives rise to the following proposition.
\begin{proposition}
\label{prop:lipappx}
Let $\epsilon > 0$ and $f^* \in \cH^1([0,1]^{d_{in}}; \, A)$, where $A = \bR^{d_{out}}$ is Euclidean space or $A = \prod_{i=1}^{d_{out}} [a_i, b_i]$ is a cube. Then $\exists f_{\NN} \in \cF_{\NN}(d_{in}, d_{out}, L, W)$ such that 

\[ \| f^* - f_{\NN} \|_{L^\infty([0,1]^{d_{in}}; \, \bR^{d_{out}})} < \epsilon.\]

In addition, the range of $f_{\NN}$ is contained in $A$. Here $L = O(\log(\epsilon^{-1}))$ and $W = O(\epsilon^{-d_{in}})$ , where the big $O$  hides constants depending only on $d_{in}$, $d_{out}$, and $\|f^*\|_{\cH^1([0,1]^{d_{in}}; \, \bR^{d_{out}})}$.
\end{proposition}
\begin{proof}[Proof of Proposition \ref{prop:lipappx}]
This follows directly from Proposition \ref{prop:lipmultiple}.
\end{proof}

WeldNet uses a composition of neural networks (i.e. propagators and transcoders) that are sequentially trained. In order to control the approximation error of such a composition, we use the following proposition.

\begin{proposition}[Composition] 
\label{prop:composition}
Let $f_1, \dots, f_n$ be Lipschitz functions, such that $f_j : A_{j-1} \rightarrow A_j$ for all $j \in [n]$. Suppose for each $j \in [n]$ and  $\epsilon_j > 0$, there is an approximating function $\hat{f}_j : A_{j-1} \rightarrow A_j$ such that $$\sup_{x \in A_{j-1}} \| f_j(x) - \hat{f}_j(x) \| < \epsilon_j.$$  Then
\begin{equation}
\label{eq:composition}
\sup_{x \in A_0} \| (\bigcirc_{j=1}^n f_j)(x) - (\bigcirc_{j=1}^n\hat{f}_j)(x) \| < \sum_{j=1}^n \Lip(\bigcirc_{\ell=j+1}^n f_\ell) \epsilon_i .
\end{equation}
\end{proposition}
\begin{proof}[Proof of Proposition \ref{prop:composition}]

{We have
\begin{align*}
\sup&_{x \in A_0} \| (\bigcirc_{j=1}^n f_j)(x) - (\bigcirc_{j=1}^n \hat{f}_\ell)(x) \| \\ 
    & \leq \sup_{x \in A_0} \sum_{j=1}^n \| (\bigcirc_{\ell=j+1}^n f_\ell \circ f_j \circ \bigcirc_{\ell=1}^{j-1}\hat{f}_\ell)(x) - (\bigcirc_{\ell=j+1}^n f_\ell \circ \hat{f}_j \circ \bigcirc_{\ell=1}^{j-1}\hat{f}_\ell)(x) \| \\
    & \leq \sum_{j=1}^n \sup_{x \in A_0} \| (\bigcirc_{\ell=j+1}^n f_\ell \circ f_j \circ \bigcirc_{\ell=1}^{j-1} \hat{f}_\ell)(x) - (\bigcirc_{\ell=j+1}^n f_\ell \circ \hat{f}_j \circ \bigcirc_{\ell=1}^{j-1} \hat{f}_\ell)(x) \| \\
    & \leq \sum_{j=1}^n \Lip(\bigcirc_{\ell=j+1}^n f_\ell) \sup_{x \in A_0} \| (f_j \circ \bigcirc_{\ell=1}^{j-1}\hat{f}_\ell)(x) - (\hat{f}_j \circ \bigcirc_{\ell=1}^{j-1}\hat{f}_\ell)(x) \| \\
     & \leq \sum_{j=1}^n \Lip(\bigcirc_{\ell=j+1}^n f_\ell) \sup_{y \in A_{j-1}} \| f_j(y) - \hat{f}_j(y) \| \\
    & \leq \sum_{j=1}^n \Lip(\bigcirc_{\ell=j+1}^n f_\ell) \epsilon_j .
\end{align*}}

\end{proof}

\section{Proof of Main Results}
\subsection{Proof of Lemma \ref{lemma:weldlatentode}}
\label{app:proveweldlatentodelemma}
In this section, we prove Lemma \ref{lemma:weldlatentode}. In the following, the windows are the same as the segments. First, we prove lemmas that establish approximation results for the autoencoders, propagators, and transcoders.

\begin{lemma}[Autoencoder Approximation]
\label{lemma:aeonechart}
Suppose Assumption \ref{assum:mfd} holds, and let $0 < \epsilon_1, \epsilon_2 < \min\left\{1, \bm{\tau}(\cM([0, T])) / 2\right\}$. Then for any $i \in [\sfS]$, there exists $\sE^i_{\NN} \in \cF_{\NN}(D, d+1, L_{\sE^i}, W_{\sE^i})$ and $\sD^i_{\NN} \in \cF_{\NN}(d+1, D, L_{\sD^i}, W_{\sD^i})$
with parameters
\begin{align*}
& L_{\sE^i} = O(\log^2 (\epsilon_1^{-1})),\quad   W_{\sE^i} = O(D\epsilon_1^{-(d+1)})
\\
&L_{\sD^i} = O(\log (\epsilon_2^{-1})),\quad  
 W_{\sD^i} = O(D \epsilon_2^{-(d+1)})
\end{align*} 
such that 
\[ \sup_{\bfx \in \cM([s_i, s_{i+1}])} \| \sE_{\NN}(\bfx) - \sE^i_*(\bfx) \|_{\bR^{d+1}} \le \epsilon_1 \text{  and } \sup_{\bfz \in \cZ([s_i, s_{i+1}])} \| \sD_{\NN}(\bfz) - \sD_*(\bfz) \|_{\bR^D} \le \epsilon_2. \]

Moreover, the range of $\sE_\NN$ is contained in $\cZ([0, T])$. The constants hidden in $O$ depend on (for each $i \in [\sfS]$) $\log D$, $d$, $\bm{\tau}(\cM([s_i, s_{i+1}]))$, $s_{i+1}-s_i$, $\Lip(\sE^i_*)$, $\Lip(\sD^i_*)$, the volume of $\cM([s_i, s_{i+1}])$, and $\sup_{\bfx \in \cM([s_i, s_{i+1}])} \|\bfx\|_{\bR^D}$. 
\end{lemma}

\begin{proof}[Proof of Lemma \ref{lemma:aeonechart}]
This is follows from  \cite[Lemma 6 and Lemma 8]{liu2024deep}, with one additional step to restrict the range of $\sE_\NN$. This is done by applying Proposition \ref{prop:appxtools}(Part 2), which increases the number of layers by 2.
\end{proof}

\begin{lemma}[Propagator Approximation with Latent Dynamics]
\label{lemma:propdynamic}
Suppose there is a Lipschitz function $g : \cZ([a, b]) \rightarrow \bR^{d+1}$ such that for all $\bfz(a) \in \cZ(a)$, if we denote $\bfz(t) = \sP_*(\bfz(a), t - a)$ for all $t \in [a, b]$, then these latent codes satisfy:
\begin{equation} \frac{\partial \bfz}{\partial t} (t) = g(\bfz(t)). 
\label{eq:odeapp}
\end{equation}
Denote ${\rm Lip}(g) = \sup_{t \in [a,b]} {\rm Lip}_{\cZ(t)}(g
(\bfz(t)))$ and $T = b - a$. Suppose $a = t_1 < \dots < t_{\sfT-1} < t_\sfT = b$ is an equally spaced time grid, i.e. $t_{k-1} - t_k = \Delta t$ is the same for all $k \in [\sfT -1]$. For any $\epsilon > 0$ and $\sfT > 1 + {{\rm Lip}(g) T^2 e^{{\rm Lip}(g) T}}{\epsilon}^{-1}$, then there is a $\sP_\NN \in \cF_\NN(d+1, d+1, L, W)$, with range contained in $\cZ([a, b])$, such that for any $k \in [\sfT - 1]$ and $i \in [\sfT - k]$,
\[ \sup_{\bfz(t_k) \in \cZ(t_k)} \|\left(\bigcirc_{j=1}^i \sP_\NN \right)(\bfz(t_k)) - \sP_*(\bfz(t_k), t_{k+i} - t_k) \|_{\bR^{d+1}} < \epsilon. \]
Here $L = O\left(\log\left(\frac{1}{\epsilon}\right)\right)$, $W = O\left(\epsilon^{-(d+1)}\right)$, where the constants hidden in $O$ depend on $d$, ${\rm Lip}(g)$, $ \|g\|_{L^\infty(\cZ([0, T]))}$, and $T$.
\end{lemma}
\begin{proof}[Proof of Lemma \ref{lemma:propdynamic}]
Since $g$ is Lipschitz, by Proposition \ref{prop:lipappx} there is a $g_\NN \in g_\NN(d+1, d+1, L_g, W_g)$ such that
\[ \sup_{\bfz \in \cZ([a, b])} \| g_\NN(\bfz) - g(\bfz) \|_{\bR^{d+1}} < \delta := \frac{\epsilon}{6 T e^{\mathrm{Lip}(g)T}}. \] 
Here $L_g = cO\left(\log\left(\frac{1}{\epsilon}\right) \right)$ and $W_g = O\left(\epsilon^{-(d+1)}\right)$. We define $\sP_\NN : \cZ([a, b]) \rightarrow \cZ([a, b])$ component-wise. For all $i \in [d]$, define the $i$th component of $\sP_\NN$ as
\[\sP^i_\NN(\bfz) = \max(0, \min(\bfz_i + \Delta t (g_\NN(\bfz))^i)\]
and the $(d+1)$th component as $\sP_\NN^{d+1}(\bfz_1, \dots, \bfz_{d+1}) = \bfz_{d+1} + \Delta t$. Then note we can implement $\sP_\NN$ as a neural network by modifying the last layer of $\sF_\NN$ and then adding two layers using Proposition \ref{prop:appxtools}(Part 2). Then $\sP_\NN \in g_\NN(d+1, d+1, L_\sP, W_\sP)$ with $L_\sP = O\left(\log\left(\frac{1}{\epsilon}\right) \right)$ and $W_\sP = O\left(\epsilon^{-(d+1)}\right)$.

We next use classical argument for the convergence of the Euler  method for ODEs to finish the proof (c.f. \cite{leveque2007finite} Chapter 6). Let $\bfz(a) \in \cZ(a)$. For all $k \in [\sfT]$, define $\bfz(t_k) = \sP_*(\bfz(a), t_k-a)$, and define $\hat{\bfz}_k$ using the iterative formula
\[ \hat{\bfz}_k = \sP_\NN(\hat{\bfz}_{k-1}), \quad \hat{\bfz}_1 = \bfz(a_i). \]
Note that 
\begin{align*} 
\bfz(t_{k+1}) &= \bfz(t_k) + \Delta t \frac{\partial \bfz}{\partial t}(t_k) + \int_{t_k}^{t_{k+1}} \left(\frac{\partial \bfz}{\partial t}(s) - \frac{\partial \bfz}{\partial t}(t_k)\right) ds \\
&= \bfz(t_k) + \Delta t g(\bfz(t_k)) + \int_{t_k}^{t_{k+1}} \left(g(\bfz(s)) - g(\bfz(t_k))\right) ds.
\end{align*}
This means
\begin{align*} 
\left\|\frac{1}{\Delta t} (\bfz(t_{k+1}) - \bfz(t_k)) - g(\bfz(t_k))\right\|_{\bR^d} 
    &= \frac{1}{\Delta t} \left\| \int_{t_k}^{t_{k+1}} \left(g(\bfz(s)) - g(\bfz(t_k))\right) ds \right\|_{\bR^{d+1}} \\
    &\leq \frac{1}{\Delta t} \int_{t_k}^{t_{k+1}} \| g(\bfz(s)) - g(\bfz(t_k)) \|_{\bR^{d+1}} ds \\
    &\leq \frac{1}{\Delta t} \int_{t_k}^{t_{k+1}} {\rm Lip}(g) \| \bfz(s) - \bfz(t_k) \|_{\bR^{d+1}} ds \\
    &\leq \frac{{\rm Lip}(g)}{\Delta t} \int_{t_k}^{t_{k+1}} \|g\|_{L^\infty(\cZ([a, b]))} (s - t_k) ds \\
    &\leq \frac{{\rm Lip}(g)}{\Delta t} \|g\|_{L^\infty(\cZ([a, b]))} \frac{(\Delta t)^2}{2} \leq C \Delta t,
\end{align*}
where $C={\rm Lip}(g) \|g\|_{L^\infty(\cZ([a, b]))}/2$.

Now, let $\mathbf{e}_k = \frac{1}{\Delta t}(\bfz(t_{k+1}) - \bfz(t_k)) - g_{\rm NN}(\bfz(t_k))$ be the step $k$ truncation error for Euler method of solving the ODE \eqref{eq:odeapp} while the governing equation is $g_{\rm NN}$ instead of $g$. We have
\[\bfz(t_{k+1}) = \bfz(t_k) + \Delta t {g_{\rm NN}}(\bfz(t_k)) + \Delta t \mathbf{e}_k. \]
This local truncation error satisfies the error bound:
\begin{align*} 
\|\mathbf{e}_k \|_{\bR^{d+1}} 
    &= \left\|\frac{1}{\Delta t}(\bfz(t_{k+1}) - \bfz(t_k)) - g_\NN(\bfz(t_k))\right\|_{\bR^{d+1}} \\
    &\leq \left\|\frac{1}{\Delta t}(\bfz(t_{k+1}) - \bfz(t_k)) - g(\bfz(t_k)) \right\|_{\bR^{d+1}} + \|g(\bfz(t_k)) - g_\NN(\bfz(t_k)) \|_{\bR^{d+1}} \\
    &\leq C \Delta t + \delta. 
\end{align*}
Then we decompose
\begin{align*}
\bfz(t_{k+1}) - \hat{\bfz}_{k+1} 
    &= \bfz(t_k) + \Delta t g_\NN(\bfz(t_k)) + \Delta t \mathbf{e}_k - \hat{\bfz}_k - \Delta t g_\NN(\hat{\bfz}_k) \\
    &= (\bfz(t_k) - \hat{\bfz}_k) + \Delta t (g_\NN(\bfz(t_k)) - g_\NN(\hat{\bfz}_k)) + \Delta t \mathbf{e}_k.
\end{align*}
This means 
\[\|\bfz(t_{k+1}) - \hat{\bfz}_{k+1}\|_{\bR^{d+1}} \leq \|\bfz(t_k) - \hat{\bfz}_k \|_{\bR^{d+1}} + \Delta t \|g_\NN(\bfz(t_k)) - g_\NN(\hat{\bfz}_k)\|_{\bR^{d+1}} + \Delta t \|\mathbf{e}_k\|_{\bR^{d+1}}.\]

Now $\|g_\NN(\bfz(t_k)) - g_\NN(\hat{\bfz}_k)\|_{\bR^{d+1}}
    \leq \|g_\NN(\bfz(t_k)) - g(\bfz(t_k))\|_{\bR^{d+1}} + \|g(\bfz(t_k)) - g(\hat{\bfz}_k)\|_{\bR^{d+1}} + \|g(\hat{\bfz}_k) - g_\NN(\hat{\bfz}_k)\|_{\bR^{d+1}}
     \leq 2 \delta + {\rm Lip}(g) \| \bfz(t_k) - \hat{\bfz}_k \|_{\bR^{d+1}}$. So we have the inequality 

\[ \|\bfz(t_{k+1}) - \hat{\bfz}_{k+1}\|_{\bR^{d+1}} \leq \left(1 + \Delta t {\rm Lip}(g) \right) \|\bfz(t_k) - \hat{\bfz}_k \|_{\bR^{d+1}} + \Delta t( 2\delta + \|\mathbf{e}_k\|_{\bR^{d+1}}). \]

Since $\|\bfz(t_1) - \hat{\bfz}_1\|_{\bR^d} = 0$, this implies
\begin{align*}
\|\bfz(t_{k+1})& - \hat{\bfz}_{k+1}\|_{\bR^{d+1}} \\
&\leq \Delta t \sum_{i=1}^k \left(1 + \Delta t {\rm Lip}(g)\right)^{k-i} (2\delta + \|\bfe_i\|_{\bR^{d+1}}) \leq e^{{\rm Lip}(g) T} \Delta t \sum_{i=1}^k (2\delta + \|\bfe_i\|_{\bR^{d+1}}).
\end{align*}
Therefore (since $k \Delta t < T$ and $\|\mathbf{e}_k \|_{\bR^{d+1}} \leq  C \Delta t + \delta$) 
\begin{align*} 
\|\bfz(t_{k+1}) - \hat{\bfz}_{k+1}\|_{\bR^{d+1}} &\leq e^{{\rm Lip}(g) T} \Delta t \sum_{i=1}^k (3 \delta + C \Delta t) \\
&= e^{{\rm Lip}(g) T} (k \Delta t) (3 \delta + C \Delta t) \leq 3 T e^{{\rm Lip}(g) T} \delta + C T e^{{\rm Lip}(g) T} \Delta t. 
\end{align*}

Now since $\Delta t = \frac{T}{\sfT-1}$, if we have $\sfT - 1 > \frac{2 C T^2 e^{{\rm Lip}(g) T}}{\epsilon}$, then

\begin{align*}
3 T e^{{\rm Lip}(g) T} \delta &+ C T e^{{\rm Lip}(g) T} \Delta t \\
&\leq \frac{3 T e^{{\rm Lip}(g) T} \epsilon}{6 T e^{{\rm Lip}(g) T} } + \frac{CT^2 e^{{\rm Lip}(g) T}}{\sfT - 1} \leq \frac{\epsilon}{2} + (C T^2 e^{{\rm Lip}(g) T}) \frac{\epsilon}{2 C T^2 e^{{\rm Lip}(g) T}} = \epsilon. \end{align*}
\end{proof}

\begin{lemma}[Transcoder Approximation]
\label{lemma:transcoder}
 Suppose Assumption \ref{assum:mfd} holds. Fix an $i \in [\sfS]$, and for any $\epsilon > 0$, there exists a $\sT_{\NN} \in \cF_{\NN}(d+1, d+1, L_\sT, W_\sT)$ with parameters
 \begin{equation*}
 L_\sT = O(\log(\epsilon^{-1})),\quad W_\sT = O(\epsilon^{-d})
 \end{equation*}
 such that
\[
\sup_{\bfz(s_{i+1}) \in \cZ(s_{i+1})} \left\| \sT_{\NN}(\bfz(s_{i+1})) - \sE^{i+1}_*(\sD^i_*(\bfz(s_{i+1}))) \right\|_{\bR^{d+1}} < \epsilon
\]
Moreover, the range of $\sT_\NN$ is contained in $\cZ(s_{i+1})$. Here, $L_\sT = O\left(\log\left(\frac{1}{\epsilon}\right)\right)$ and $W_\sT = O\left(\epsilon^{-\frac{1}{d}}\right).$ The constants hidden in $O$ depend on $d$, $\tau_\cM$, $\Lip_{\cM(s_{i+1})}(\sE^{i+1}_*)$, $\Lip_{\cZ(s_{i+1})}(\sD^i_*)$, the volume of $\cM(s_{i+1})$, and $\sup_{x \in \cM(s_{i+1})} \|x\|_{\bR^D}$.
\end{lemma}
\begin{proof}[Proof of Lemma \ref{lemma:transcoder}] Consider the oracle transcoder $\sT_* = \sE^{i+1}_* \circ \sD^i_*$, which satisfies the Lipschitz condition: 
\[\Lip_{\cZ(s_{i+1})}(\sT_*) = \Lip_{\cZ(s_{i+1})}( \sE^{i+1}_* \circ \sD^i_*) \leq \Lip_{\cZ(s_{i+1})}(\sE^{i+1}_*) \Lip_{\cM(s_{i+1})}(\sD^i_*) < \infty.\]

The proof then follows by applying Proposition \ref{prop:lipappx} to $\sT_*$. The width of the  neural network size is exponential in $d$ instead of $d+1$, because the function $\sT_*$ leaves the $(d+1)$th component unchanged (as that component is the time), so it can be exactly represented by a neural network with no need for approximation.
\end{proof}

Now we prove Lemma \ref{lemma:weldlatentode}. We refer to Table \ref{table:thmnotations} for some important notations to be used in the proof.

\begin{proof}[Proof of Lemma \ref{lemma:weldlatentode}]
For any $i \in [\sfS]$, denote $\bm{T}_i = |\bT \cap (s_i, s_{i+1}]|$. Suppose $t_k$ is the $j$th element of $\bT \cap (s_i, s_{i+1}]$, i.e. it is in window/segment number $i$. 
We will construct encoder, decoder, propagator, and transcoder networks such that for any $t_k$, 
\begin{equation} 
\label{eq:weldproof}
\left\| \sD^i_{\NN} \circ \bigcirc_{\ell = 1}^j \sP_\NN^i \circ \bigcirc_{w=1}^{i-1} \left(\sT^w_{\NN} \circ \bigcirc_{\ell=1}^{\bm{T}_w} \sP^w_\NN\right) \circ \sE^1_{\NN} - \sF(\cdot, t_k) \right\|_{L^\infty(\cM(0))} < \epsilon. 
\end{equation}

We denote $t_{\ell}^w$ as the $\ell$th element of $\bT \cap [s_w, s_{w+1}]$. 
Note that the evolution operator  to $t_k$ is
\begin{align*}
&\sF(\cdot, t_k) \\
    &= \bigcirc_{\ell=1}^j \sF(\cdot, t^i_{\ell+1} - t^i_\ell) \circ \bigcirc_{w=1}^{i-1} \bigcirc_{\ell=1}^{\bm{T}_w} \sF(\cdot, t^w_{\ell+1} - t^w_{\ell} ) \\
    &= \sD^i_* \circ \bigcirc_{\ell = 1}^j (\sE^i_* \circ \sF(\cdot, t^i_{\ell+1} - t_\ell) \circ \sD^i_*) \circ \\
    & \quad \bigcirc_{w=1}^{i-1} \left(\sE^{w+1}_* \circ \sD^w_* \circ \bigcirc_{\ell=1}^{\bm{T}_w} (\sE^w_* \circ \sF(\cdot, t^w_{\ell+1} - t^w_{\ell} ) \circ \sD^w_*)\right) \circ \sE^1_*.
\end{align*}

We will approximate each term of the composition in this expression by neural networks, and then apply Proposition \ref{prop:composition}.

For all $w \in [\sfS]$, consider the encoders and decoders $\sE^w_*$ and $\sD^w_*$. By Lemma \ref{lemma:aeonechart}, there is a neural network $\sE^w_\NN$ with $O\left(\log^2\left(\frac{\sfS}{\epsilon}\right)\right)$ layers and width $O\left(D\left(\frac{\sfS}{\epsilon}\right)^{d+1} \right)$ and a neural network $\sD^w_\NN$ with $O\left(\log\left(\frac{\sfS}{\epsilon}\right)\right)$ layers and width $O\left(D\left(\frac{\sfS}{\epsilon}\right)^{d+1} \right)$ such that 

\begin{equation} 
\label{eq:dynamicthmae}
\begin{split}
&\| \sE_\NN^w - \sE_*^w \|_{L^\infty(\cM([s_w, s_{w+1}]))} < \frac{\epsilon}{(2\sfS + 1){\rm Lip}(\sD_*) {\rm Lip}(\sF)},
\\
&\| \sD_*^w - \sD_\NN^w \|_{L^\infty(\cZ( [s_w, s_{w+1}]))} < \frac{\epsilon}{2\sfS + 1},
\end{split}
\end{equation}
where ${\rm Lip}(\sD_*)=\max_w {\rm Lip}(\sD^w_*)$. Note, the range of $\sE_\NN^w$ is contained in $\cZ([s_w, s_{w+1}])$. 

Next, for all $w \in [\sfW]$ consider $\sP^w_*(\bfz, t) = \sE^w_*(\sF(\sD_*^w(\bfz), t))$ which is called the oracle propagator. Since $\sP^w_*$ is Lipschitz as a composition of Lipschitz functions, by Lemma \ref{lemma:propdynamic}, there is a neural network $\sP^w_\NN$ with $O\left(\log\left( \frac{\sfS}{\epsilon} \right)\right)$ layers and width $O\left(\left(\frac{\sfS}{\epsilon}\right)^{d+1}\right)$ such that 
\begin{equation}
\label{eq:dynamicthmprop}
    \| \bigcirc_{\ell=1}^{\bm{T}_w} \sP_\NN^w - \bigcirc_{\ell=1}^{\bm{T}_w} \sP^w_*(\cdot, t^w_{\ell+1} - t^w_{\ell}) \|_{L^\infty(\cZ(s_w))} < \frac{\epsilon}{(2\sfS + 1) {\rm Lip}(\sD) {\rm Lip}(\sF)}.
\end{equation}
The composition above makes sense as the range of $\sP_\NN^w$ is contained in $\cZ([s_w, s_{w+1}])$ by construction.

Finally, for all $w \in [\sfS - 1]$, consider $\sE^{w+1}_* \circ \sD^w_*$. By Lemma \ref{lemma:transcoder}, there is a neural network $\sT^w_\NN$ with $O\left(\log\left(\frac{\sfS}{\epsilon}\right)\right)$ layers and width $O\left(\left(\frac{\sfS}{\epsilon}\right)^d \right)$ such that
\begin{equation} 
\label{eq:dynamicthmtrans}
\| \sT_\NN^w - \sE_*^{w+1} \circ \sD_*^w \|_{L^\infty(\cZ (s_{w+1}))} < \frac{\epsilon}{(2\sfS + 1){\rm Lip}(\sD) {\rm Lip}(\sF)}, \end{equation}
with the range of $\sT_\NN^w$ contained in $\cZ(s_{w+1}) = [0, 1]^d \times \{s_{w+1}\}$. Theorem \ref{thm:weldlatentode} then follows by applying Proposition \ref{prop:composition} to the expression (where we indicate matching terms with the same letter and subscript):
\[\|\underbrace{\sD^i_{\NN}}_{a} \circ \underbrace{\bigcirc_{\ell = 1}^j \sP_\NN^i}_{b} \circ \bigcirc_{w=1}^{i-1} (\underbrace{\sT^w_{\NN}}_{c_w} \circ \underbrace{\bigcirc_{\ell=1}^{\bm{T}_w} \sP^w_\NN}_{d_w}) \circ \underbrace{\sE^1_{\NN}}_{e} - \]
\[\underbrace{\sD^i_*}_{a} \circ \underbrace{\bigcirc_{\ell = 1}^j \sP^w_*(\cdot, t^i_{\ell+1} - t_\ell)}_{b} \circ \bigcirc_{w=1}^{i-1} \underbrace{(\sE_*^{w+1} \circ \sD_*^w)}_{c_w} \circ \underbrace{\bigcirc_{\ell=1}^{\bm{T}_w} \sP^w_*(\cdot, t^w_{\ell+1} - t^w_{\ell+1})}_{d_w}) \circ \underbrace{\sE^1_*}_{e} \|,\]
where the norm is taken over the space ${L^\infty(\cM(0))}$. Note that we can do this because the range of each neural network in the composition is contained in the domain of the subsequent neural network (by construction). To compute the bound given in Proposition \ref{prop:composition}, we compute Lipschitz constants of compositions of functions. Both terms in the difference involve $2 + 2(i - 1) + 1 = 2i + 1$ functions. Each term in Equation \eqref{eq:composition} is the product of the approximation error of a network (encoder, propagator, transcoder, or decoder), and the Lipschitz constant of the oracle maps that come later in the approximation. We will show that each term in the sum is less than $\frac{\epsilon}{2i+1}$, so that the total error is less than $\epsilon$. Recall $t^i_{\ell+1} = t_k$ and $a_1 = t_1$.

\begin{itemize}
    \item[(a)] \textbf{Decoder}. The decoder $\sD^i_*$ corresponds to the $(2i+1)$th and last term of the sum in Equation \eqref{eq:composition}. The Lipschitz factor is over an empty composition, which is defined to be $1$ by convention. Thus the decoder term contributes an error of $\frac{\epsilon}{2\sfS + 1}$ by Equation \eqref{eq:dynamicthmae}, which is less than $ \frac{\epsilon}{2 i + 1}$ since $i \leq \sfS$.

    \item[(b)] \textbf{Final Propagator}. The final propagator is $\sP_*^i$, and it corresponds to the second to last term of the sum in Equation \eqref{eq:composition}. The Lipschitz factor is the Lipschitz constant of the decoder $\sD_*^i$. Thus this term contributes an error of (by the approximation in \eqref{eq:dynamicthmprop})

    \[ \Lip_{\cZ^i(t_k)}(\sD_*^i) \cdot \frac{\epsilon}{(2\sfS + 1) \Lip(\sD) \Lip(\sF)}  \leq \frac{\epsilon}{(2\sfS + 1) \Lip(\sF)} \leq \frac{\epsilon}{2 i + 1}.\]

    \item[(c)] \textbf{Transcoder}. For any window $m \in [i-1]$, the $m$th transcoder corresponds to a middle term in the sum in Equation \eqref{eq:composition}. For each of these terms, the error is again given by the approximation error of the corresponding transcoder network, scaled by the Lipschitz constant of the neural networks that come after. The Lipschitz constant of the composition of all functions after the $m$th transcoder are (where $\Lip(F)$ is short for $\Lip_{\cZ^m(s_{m+1})}(F)$)

{\scriptsize
    \begin{align*} 
    &\Lip\left( \sD^i_* \circ \bigcirc_{\ell = 1}^j (\sP^i_*(\cdot, t^i_{\ell+1} - t^i_\ell)) \circ \bigcirc_{w=m+1}^{i-1} \left(\sE^{w+1}_* \circ \sD^w_* \circ \bigcirc_{\ell=1}^{\bm{T}_w} (\sP^w_*(\cdot, t^w_{\ell+1} - t^w_{\ell} )\right) \right) \\
    &=\Lip \Big( \sD^i_* \circ \bigcirc_{\ell = 1}^j (\sE^i_* \circ \sF(\cdot, t^i_{\ell+1} - t^i_\ell) \circ \sD^i_*) \circ \\
    &\quad \bigcirc_{w=m+1}^{i-1} \left(\sE^{w+1}_* \circ \sD^w_* \circ \bigcirc_{\ell=1}^{\bm{T}_w} (\sE^w_* \circ \sF(\cdot, t^w_{\ell+1} - t^w_{\ell} ) \circ \sD^w_*)\right)\Big) \\
    &= \Lip\Big(\sD^i_* \circ \sE_*^i \circ \sF(\cdot, t^i_{j+1} - s_i) \circ \sD_*^i \circ \\
    &\quad \bigcirc_{w=m+1}^{i-1} \left( \sE_*^{w+1} \circ \sD_*^w \circ \sE_*^w \circ \sF(\cdot, s_{w+1} - s_w ) \circ \sD_*^w \right)\Big) \\
    &= \Lip\left(\sF(\cdot, t_k - s_i) \circ \sD_*^i \circ \bigcirc_{w=m+1}^{i-1} \left( \sE_*^{w+1} \circ \sF(\cdot, s_{w+1} - s_w ) \circ \sD_*^w \right) \right) \\
    &= \Lip\left(\sF(\cdot, t_k - s_i) \circ \sD_*^i \circ \sE_*^i \circ \sF(\cdot, s_i - s_{m+1} ) \circ \sD_*^{m+1} \right) \\
    &= \Lip\left(\sF(\cdot, t_k - s_i) \circ \sF(\cdot, s_i - s_{m+1} ) \circ \sD_*^{m+1} \right) \\
    &= \Lip\left(\sF(\cdot, t_k - s_{m+1} ) \circ \sD_*^{m+1} \right) \leq \Lip_{\cM(s_{m+1})}(\sF(\cdot, t_k - s_{m+1} )) \Lip_{\cZ^{m+1}(s_{m+1})}(\sD_*^{m+1}). 
    \end{align*}
    }

     Thus, by scaling the approximation guarantee in Equation \eqref{eq:dynamicthmtrans}, we see that each transcoder term contributes error at most

    \begin{align*}
    \Lip&_{\cM(s_{m+1})}(\sF(\cdot, t_k - s_{m+1} )) \Lip_{\cZ^{m+1}(s_{m+1})}(\sD_*^{m+1}) \frac{  \epsilon}{(2\sfS + 1) \Lip(\sD) \Lip(\sF)}\\
    & \leq \frac{\epsilon}{2\sfS + 1} \leq \frac{\epsilon}{2i + 1}. 
    \end{align*}
    
    \item[(d)] \textbf{Intermediate Propagators}. For any window $m \in [i-1]$, the $m$th propagator corresponds to a middle term in the sum in Equation \eqref{eq:composition}, similar to the transcoder case. The Lipschitz constant of the composition of all functions after the $m$th oracle propagator are (where $\Lip(F)$ is short for $\Lip_{\cZ^m(s_{m+1})}(F)$)

    {\scriptsize
    \begin{align*} 
    &\Lip\Big( \sD^i_* \circ \bigcirc_{\ell = 1}^j (\sP^i_*(\cdot, t^i_{\ell+1} - t^i_\ell) \circ \sD^i_*) \circ \\
    &\quad \bigcirc_{w=m+1}^{i-1} \left(\sE^{w+1}_* \circ \sD^w_* \circ \bigcirc_{\ell=1}^{\bm{T}_w} (\sP^w_*(\cdot, t^w_{\ell+1} - t^w_{\ell} )\right) \circ \sE_*^{m+1} \circ \sD_*^{m}\Big) \\
    &= \Lip( \sD^i_* \circ \bigcirc_{\ell = 1}^j (\sE^i_* \circ \sF(\cdot, t^i_{\ell+1} - t^i_\ell) \circ \sD^i_*) \circ \\
    & \quad \bigcirc_{w=m+1}^{i-1} \left(\sE^{w+1}_* \circ \sD^w_* \circ \bigcirc_{\ell=1}^{\bm{T}_w} (\sE^w_* \circ \sF(\cdot, t^w_{\ell+1} - t^w_{\ell} ) \circ \sD^w_*)\right)  \circ \sE_*^{m+1} \circ \sD_*^{m} ) \\
    &= \Lip(\sD^i_* \circ \sE_*^i \circ \sF(\cdot, t^i_{j+1} - s_i) \circ \sD_*^i \circ \\
    &\quad \bigcirc_{w=m+1}^{i-1} \left( \sE_*^{w+1} \circ \sD_*^w \circ \sE_*^w \circ \sF(\cdot, s_{w+1} - s_w ) \circ \sD_*^w ) \circ \sE_*^{m+1} \circ \sD_*^{m}\right) \\
    &= \Lip\left(\sF(\cdot, t_k - s_i) \circ \sD_*^i \circ \bigcirc_{w=m+1}^{i-1} \left( \sE_*^{w+1} \circ \sF(\cdot, s_{w+1} - s_w ) \circ \sD_*^w \right) \circ \sE_*^{m+1} \circ \sD_*^{m} \right) \\
    &= \Lip\left(\sF(\cdot, t_k - s_i) \circ \sD_*^i \circ \sE_*^i \circ \sF(\cdot, s_i - s_{m+1} ) \circ \sD_*^{m+1} \circ \sE_*^{m+1} \circ \sD_*^{m} \right) \\
    &= \Lip\left(\sF(\cdot, t_k - s_i) \circ \sF(\cdot, s_i - s_{m+1} ) \circ \sD_*^m \right) \\
    &= \Lip\left(\sF(\cdot, t_k - s_{m+1} ) \circ \sD_*^{m} \right) \leq \Lip_{\cM(s_{m+1})}(\sF(\cdot, t_k - s_{m+1} )) \Lip_{\cZ^{m}(s_{m+1})}(\sD_*^{m}). 
    \end{align*}
    }%

    Thus, by scaling the approximation guarantee in Equation \eqref{eq:dynamicthmprop}, we see that each intermediate propagator term contributes error at most

    \begin{align*}
    \Lip&_{\cM(s_{m+1})}(\sF(\cdot, t_k - s_{m+1} )) \Lip_{\cZ^{m}(s_{m+1})}(\sD_*^{m}) \frac{  \epsilon}{(2\sfW + 1) \Lip(\sD) \Lip(\sF)} \\
    &\leq \frac{\epsilon}{2\sfW + 1} \leq \frac{\epsilon}{2i + 1}. 
    \end{align*}
    
    \item[(e)] \textbf{Encoder}. The encoder term $\sE^1$ corresponds to the first term of the sum in  Equation \eqref{eq:composition}. The Lipschitz factor can be computed as:

{\scriptsize
    \begin{align*} 
    &\Lip_{\cZ^1(t_1)}\left( \sD^i_* \circ \bigcirc_{\ell = 1}^j (\sP^i_*(\cdot, t^i_{\ell+1} - t^i_\ell) \circ \sD^i_*) \circ \bigcirc_{w=1}^{i-1} \left(\sE^{w+1}_* \circ \sD^w_* \circ \bigcirc_{\ell=1}^{\bm{T}_w} (\sP^w_*(\cdot, t^w_{\ell+1} - t^w_{\ell} )\right) \right) \\
    &=\Lip_{\cZ^1(t_1)}\Big( \sD^i_* \circ \bigcirc_{\ell = 1}^j (\sE^i_* \circ \sF(\cdot, t^i_{\ell+1} - t^i_\ell) \circ \sD^i_*) \circ \\
    &\quad \bigcirc_{w=1}^{i-1} \left(\sE^{w+1}_* \circ \sD^w_* \circ \bigcirc_{\ell=1}^{\bm{T}_w} (\sE^w_* \circ \sF(\cdot, t^w_{\ell+1} - t^w_{\ell} ) \circ \sD^w_*)\right) \Big) \\
    &= \Lip_{\cZ^1(t_1)}\Big(\sD^i_* \circ \sE_*^i \circ \sF(\cdot, t^i_{j+1} - s_i) \circ \sD_*^i \circ \\
    &\quad \bigcirc_{w=1}^{i-1} \left( \sE_*^{w+1} \circ \sD_*^w \circ \sE_*^w \circ \sF(\cdot, s_{w+1} - s_w ) \circ \sD_*^w \right)\Big) \\
    &= \Lip_{\cZ^1(t_1)}\left(\sF(\cdot, t_k - s_i) \circ \sD_*^i \circ \bigcirc_{w=1}^{i-1} \left( \sE_*^{w+1} \circ \sF(\cdot, s_{w+1} - s_w ) \circ \sD_*^w \right)\right) \\
    &= \Lip_{\cZ^1(t_1)}\left(\sF(\cdot, t_k - s_i) \circ \sD_*^i \circ \sE_*^i \circ \sF(\cdot, s_i - s_1 ) \circ \sD_*^1 \right) \\
    &= \Lip_{\cZ^1(t_1)}\left(\sF(\cdot, t_k - t_1) \circ \sF(\cdot, s_i - s_1 ) \circ \sD_*^1 \right) \\
    &= \Lip_{\cZ^1(t_1)}\left(\sF(\cdot, t_k - t_1 ) \circ \sD_*^1 \right) \leq \Lip_{\cM(t_1)}(\sF(\cdot, t_k - t_1 )) \Lip_{\cZ^1(t_1)}(\sD_*^1). 
    \end{align*}
}%

    Thus, by scaling the approximation guarantee in Equation \eqref{eq:dynamicthmae}, we see that the encoder term contributes error at most

    \begin{align*}
        \Lip_{\cM(t_1)}&(\sF(\cdot, t_k - t_1 )) \Lip_{\cZ^1(t_1)}(\sD_*^1)\frac{  \epsilon}{(2\sfS + 1) \Lip(\sD) \Lip(\sF)} 
        \\ &\leq \frac{\epsilon}{2\sfS + 1} \leq \frac{\epsilon}{2i + 1}.
    \end{align*} 
\end{itemize}

\end{proof}

\subsection{Proof of Theorem \ref{thm:weldlatentode}}
\label{app:proveweldlatentode}
In this section, we prove Theorem \ref{thm:weldlatentode}, which follows quickly from Lemma \ref{lemma:weldlatentode}. The key is to construct a WeldNet model with more windows that implements exactly the same function as the WeldNet constructed before.

\begin{proof}[Proof of Theorem \ref{thm:weldlatentode}]
Let $\pi:[\sfW]\rightarrow[\sfS]$ be the function that indicates (with the index) which segment each window falls inside. If $\sfW > \sfS$, this function is not one-to-one, but we can still define a (left) ``inverse'' $\pi^{[-1]}(s)=\min_i \{i: \pi(i)=s\}$. It is easy to see that for all $i\in[\sfS]$, we have that $w_{\pi^{[-1]}(i)} = s_i$.

the $\sfS$-window WeldNet from Lemma \ref{lemma:weldlatentode} denoted $\overline{\sW_\NN}$ with components (for all $s \in [\sfS]$) $\overline{\sE^s_\NN}$, $\overline{\sD^s_\NN}$, $\overline{\sP^s_\NN}$, and (for $s < \sfS$) $\overline{\sT^s_\NN}$ such that for any $k \in [\sfT]$ with $k$ being the $j$th element of  $\sfT \cap [s_i, s_{i+1}))$,

\[ \sup_{\bfx(0) \in \cM(0)} \left\|\overline{\sW_\NN}(\bfx(0), t_k) - \sF(\bfx(0), t_k)\right\|_{\bR^D} < \epsilon.\]

We now construct a $\sfW$-window WeldNet that implements the same function as the $\sfS$-window $\overline{\sW_\NN}$. For each $i \in [\sfW]$, consider the neural networks
\[ \sE_\NN^i = \overline{\sE_\NN^{\pi(i)}}, \sD_\NN^i = \overline{\sD_\NN^{\pi(i)}},  \sP_\NN^i = \overline{\sP_\NN^{\pi(i)}},\]
and for each $i \in [\sfW]$ such that $w_{i+1}=s_{j+1}$ for some $j \in [\sfS]$ (i.e  windows that end at the end of a segment), we define $\sT_\NN^i = \overline{\sT_\NN^j}$, and we construct every other transcoder to be the identity function in $\bR^{d+1}$ given by Proposition \ref{prop:idappx}.

Since for all $\bfx(0) \in \cM(0)$, we have $\|\overline{\sW_\NN}(\bfx(0), t_k) - \sF(\bfx(0), t_k)\|_{\bR^D} < \epsilon$, we can complete the proof by showing that $\sW_\NN(\bfx(0), t_k) = \overline{\sW_\NN}(\bfx(0), t_k)$. We denote $\bm{T}_i = |\bT \cap (w_i,w_{i+1}]|$ and $\overline{\bm{T}_i} = |\bT \cap (s_i,s_{i+1}]|$. Suppose that $t_k$ is the $j$th element of $\bT \cap (w_i, w_{i+1}]$ and also that $t_k$ is the $j'$-th element of $\bT \cap (s_{i'}, s_{i'+1}]$ (which means $\pi(i) = i'$). Then we have 

\begin{align*}
&\sW_\NN(\bfx(0), t_k) =  \sD^i_{\NN} \circ \bigcirc_{\ell = 1}^j \sP_\NN^i \circ \bigcirc_{w=1}^{i-1} \left(\sT^w_{\NN} \circ \bigcirc_{\ell=1}^{\bm{T}_w} \sP^w_\NN\right) \circ \sE^1_{\NN}(\bfx(0)) \text{ and } \\  
&\overline{\sW_\NN}(\bfx(0), t_k) = \overline{\sD^{i'}_{\NN}} \circ \bigcirc_{\ell = 1}^j \overline{\sP_\NN^{i'}} \circ \bigcirc_{s=1}^{i'-1} \left(\overline{\sT^s_{\NN}} \circ \bigcirc_{\ell=1}^{\overline{\bm{T}_s}} \overline{\sP^s_\NN}\right) \circ \overline{\sE^1_{\NN}}(\bfx(0)).
\end{align*} 

Note that $\sE_\NN^1 = \overline{\sE^{\pi(1)}_\NN}$ since the first window is inside of the first segment. Next, note that $\sD_\NN^i = \overline{\sD_\NN^{\pi(i)}} = \overline{\sD_\NN^{i'}}$ and $\sP_\NN^i = \overline{\sP_\NN^{\pi(i)}} = \overline{\sP_\NN^{i'}}$. Finally, we will decompose the composition of propagators and transcoders to be over each segment before being over each window. 

Let $n_s^m$ denote the index of the $m$th window in the $s$th segment, and let $\bm{n}_s$ denote the index of the final window in the $s$th segment. Note that all transcoders except for the last window in each segment implements the identity, i.e. $\sT_\NN^{n_s^m} = I_{d+1}$ if $m < |\pi^{-1}(\{s\})|$, where $\pi^{-1}(\{s\})$ is the pre-image of $s$ by $\pi$, which is the set of all window indices inside of Segment $s$; and the last transcoder within segment $s$ of $\sW_\NN$ is equal to the $s$th transcoder of $\overline{\sW_\NN}$, i.e. $\sT_\NN^{\bm{n}_s} = \overline{\sT_\NN^s}$. Also by construction $\sP^{n^m_s}_\NN = \overline{\sP^s_\NN}$. Then we decompose the non encoder/decoder terms of the above composition

\begin{align*} 
&\bigcirc_{\ell = 1}^j \sP_\NN^{i} \circ \bigcirc_{w=1}^{i-1} \left(\sT^w_{\NN} \circ \bigcirc_{\ell=1}^{\bm{T}_w} \sP^w_\NN\right) \\
& = \bigcirc_{\ell = 1}^j \overline{\sP_\NN^{i}} \circ \bigcirc_{w=1}^{i-1} \left(\sT^w_{\NN} \circ \bigcirc_{\ell=1}^{\bm{T}_w} \sP^w_\NN\right) \\
& = \bigcirc_{\ell = 1}^j \overline{\sP_\NN^{i'}} \circ \boldsymbol{\bigcirc}_{s=1}^{i'-1} \left[ \bigcirc_{m=1}^{|\pi^{-1}(s)|} \left(\sT^{n_s^m}_{\NN} \circ \bigcirc_{\ell=1}^{\bm{T}_{n_s^m}} \sP^{n_s^m}_\NN\right)\right] \\
& = \bigcirc_{\ell = 1}^j \overline{\sP_\NN^{i'}} \circ \boldsymbol{\bigcirc}_{s=1}^{i'-1} \left[ \bigcirc_{m=1}^{|\pi^{-1}(s)|} \left(\sT^{n_s^m}_{\NN} \circ \bigcirc_{\ell=1}^{\bm{T}_{n_s^m}} \overline{\sP^s_\NN}\right)\right] \\
& = \bigcirc_{\ell = 1}^j \overline{\sP_\NN^{i'}} \circ \boldsymbol{\bigcirc}_{s=1}^{i'-1} \left[ \left(\sT^{\bm{n}_s}_{\NN} \circ \bigcirc_{\ell=1}^{\bm{T}_{\bm{n}_s}} \overline{\sP^s_\NN}\right) \circ \bigcirc_{m=1}^{|\pi^{-1}(s)|-1} \left(\sT^{n_s^m}_{\NN} \circ \bigcirc_{\ell=1}^{\bm{T}_{n_s^m}} \overline{\sP^{s}_\NN}\right)\right] \\ 
& = \bigcirc_{\ell = 1}^j \overline{\sP_\NN^{i'}} \circ \boldsymbol{\bigcirc}_{s=1}^{i'-1} \left[ \overline{\sT^s_{\NN}} \circ \bigcirc_{\ell=1}^{\bm{T}_{\bm{n}_s}} \overline{\sP^s_\NN} \circ \bigcirc_{m=1}^{|\pi^{-1}(s)|-1} \left(I_{d+1} \circ \bigcirc_{\ell=1}^{\bm{T}_{n_s^m}} \overline{\sP^{s}_\NN}\right)\right] \\
& = \bigcirc_{\ell = 1}^j \overline{\sP_\NN^{i'}} \circ \boldsymbol{\bigcirc}_{s=1}^{i'-1} \left[ \overline{\sT^s_{\NN}} \circ \bigcirc_{\ell=1}^{\bm{T}_{\bm{n}_s}} \overline{\sP^s_\NN} \circ \bigcirc_{m=1}^{|\pi^{-1}(s)|-1}  \bigcirc_{\ell=1}^{\bm{T}_{n_s^m}} \overline{\sP^{s}_\NN}\right] \\
& = \bigcirc_{\ell = 1}^j \overline{\sP_\NN^{i'}} \circ \boldsymbol{\bigcirc}_{s=1}^{i'-1} \left[ \overline{\sT^s_{\NN}} \circ \bigcirc_{\ell=1}^{\overline{\bm{T}_s}} \overline{\sP^{s}_\NN}\right].
\end{align*}

Thus,

\begin{align*}
\sW_\NN(\bfx(0), t_k) 
    &=  \sD^i_{\NN} \circ  \bigcirc_{\ell = 1}^j \sP_\NN^i \circ \bigcirc_{w=1}^{i-1} \left(\sT^w_{\NN} \circ \bigcirc_{\ell=1}^{\bm{T}_w} \sP^w_\NN\right) \circ \sE^1_{\NN} \\ 
    &=  \sD^i_{\NN}  \circ \bigcirc_{\ell = 1}^j \overline{\sP_\NN^{i'}} \circ \boldsymbol{\bigcirc}_{s=1}^{i'-1} \left[ \overline{\sT^s_{\NN}} \circ \bigcirc_{\ell=1}^{\overline{\bm{T}_s}} \overline{\sP^{s}_\NN}\right] \circ  \sE^1_{\NN} \\
    &= \overline{\sD^{i'}_\NN} \circ \bigcirc_{\ell = 1}^j \overline{\sP_\NN^{i'}} \circ \boldsymbol{\bigcirc}_{s=1}^{i'-1} \left[ \overline{\sT^s_{\NN}} \circ \bigcirc_{\ell=1}^{\overline{\bm{T}_s}} \overline{\sP^{s}_\NN}\right] \circ  \overline{\sE^1_\NN} \\
    &= \overline{\sW_\NN}(\bfx(0), t_k).
\end{align*}

Now note that the size of all autoencoders, all propagators, and $\sfS - 1$ transcoders of $\sW_\NN$ are exactly equal to the size of (corresponding) components in $\overline{\sW_\NN}$. For the $\sfW - \sfS$ identity transcoders of $\sW_\NN$, we can use any network size as a result of Proposition \ref{prop:idappx}. This completes the proof.
\end{proof}

\subsection{Proof of Theorem \ref{thm:weldgeneral}}
\label{app:proveweldgeneral}
In this section, we prove Theorem \ref{thm:weldgeneral}. First, we prove a version of Lemma \ref{lemma:propdynamic} but without the dynamics assumption. 

\begin{lemma}[Propagator Approximation in General Case]
\label{lemma:propappx}
Let $\sP_* : \cZ([0, T]) \times [0, T] \rightarrow \cZ([0, T])$ be a Lipschitz function such that
\begin{itemize}
    \item For all $x \in \cZ([0, T])$, $\sP_*(x, 0) = x$,
    \item For all $x \in \cZ([0, T]), t \in [0, T], s \in [0, T - t]$,
    $\sP_*(\sP_*(x, t), s) = \sP_*(x, t + s).$  
\end{itemize}
Let $0 = t_1 < t_2 < \dots < t_\sfT = T$ be a grid for $[0, T]$ of $\sfT$ points, and $\epsilon > 0$. Then there is a function $\sP_\NN \in \cF_{\NN}(d+1, d+1, L, W)$ with parameters
\begin{equation*}
L = O\left(\sfT \log\left(\frac{\sfT}{\epsilon}\right)\right),\quad W = O\left(\left(\frac{\sfT }{\epsilon}\right)^d\right),
\end{equation*}
such that for all $k \in [\sfT-1]$ and $i \in [\sfT - k]$, we have
\[ \sup_{\bfz(t_k) \in \cZ(t_k)} \|\left(\bigcirc_{j=1}^i \sP_\NN \right)(\bfz(t_k)) - \sP_*(\bfz(t_k), t_{k+i} - t_k) \|_{\bR^{d+1}}\leq \varepsilon.  \]
Moreover, the range of $\sP_\NN$ is contained in $\cZ([0, 1])$. The constants hidden in $O$ depend on $d$ and ${\rm Lip}_*(\sP)=\sup_{j \in [\sfT - 1], i \le \sfT-j} \Lip_{\cZ(t_j)}(\sP_*(\cdot, t_{j+i} - t_{j}))$. 
\end{lemma}
\begin{proof}[Proof of Lemma \ref{lemma:propappx}]
For all $j \in [\sfT - 1]$, define $p_*^j : [0, 1]^d \rightarrow [0,1]^d$ by $p_*^j(\bfv) = [\sP_*((\bfv, t_j), t_{j+1} - t_j)]_{1, \dots, d}$ for all $\bfv \in [0,1]^d$. In other words, $p_*^j$ returns the first $d$-components (i.e. without time) of the code propagated from time $t_j$ to $t_{j+1}$, and note that this is Lipschitz. 

By Proposition \ref{prop:lipappx}, there is a neural network $p_\NN^j \in \cF_{\NN}(d, d, L_j, W_j)$ such that $\|p_\NN^j - p_*^j\|_{L^\infty([0, 1]^d)} < \frac{\epsilon}{(\sfT - 1) {\rm Lip}(\sP_*)}$, with $O\left(\log\left(\frac{\sfT}{\epsilon} \right)\right)$ layers and width $O\left( \left(\frac{\sfT}{\epsilon}\right)^d \right)$. Moreover, the range of of $p^j_\NN$ can be restricted to $[0,1]^d$ by Proposition \ref{prop:appxtools}(Part 2).

For all $k \in [\sfT-1]$ and $i \in [\sfT - k]$, note that $\bigcirc_{j=k}^{k+i-1} p_*^j = \bigcirc_{j=k}^{k+i-1} (\sP_*(\cdot, t_{j+1} - t_j)) = \sP_*(\cdot, t_{k+i}-t_k)$. By Proposition \ref{prop:composition}, we have 
\begin{align*}
\sup_{\bfv \in [0,1]^d} & \left\| \left(\bigcirc_{j=k}^{k+i-1} p^j_\NN\right)(\bfv) - \sP_*((\bfv, t_k), t_{k+i} - t_k) \right\| \\
    &\leq \sum_{j=k}^{k+i-1} \Lip_{\cZ(t_{j+1})}\left(\bigcirc_{\ell=j+1}^{k+i-1} \sP_*(\cdot, t_{\ell + 1} - t_\ell) \right) \|p_\NN^j - p_*^j\|_{L^\infty([0, 1]^d)} \\
    &= \sum_{j=k}^{k+i-1} \Lip_{\cZ(t_{j+1})}\left(\sP_*(\cdot, t_{k+i} - t_{j+1}) \right) \|p_\NN^j - p_*^j\|_{L^\infty([0, 1]^d)} \\ 
   &\leq \sum_{j=k}^{k+i-1} \frac{\Lip_{\cZ(t_{j+1})}(\sP_*(\cdot, t_{k+i} - t_{j+1}))}{(\sfT - 1) {\rm Lip}(\sP_*)} \epsilon \leq \sum_{j=k}^{k+i-1} \frac{\epsilon}{\sfT} = \frac{i}{\sfT} \epsilon < \epsilon,
\end{align*}

We finish the proof by constructing a single neural network, denoted $\sP_\NN$, which will exactly represent each of the $p^j_\NN$, selecting the correct network based on a time parameter. For all $j \in [\sfT - 1]$, let $\delta_j = \min\{|t_{j+1} - t_j|, |t_j - t_{j-1}|\}$ (and $\delta_1 = |t_2-t_1|$) represent the distance between the time point $t_j$ and the closest other time point. 

We will use the notation $\bfv$ to represent the first $d$-dimensions of  latent code in $\cZ([a, b])$, and we define the function $\tilde{p}^j_\NN : [0,1]^d \times [0, T] \rightarrow [0,1]^d$ (where we recall that $\sigma$ is the ReLU function)

\[ \tilde{p}_\NN^j((\bfv, t)) := \sigma\left(\left( p^j_\NN((\bfv, t)), t_j\right) - \frac{\sqrt{d}}{\delta_j}  (t-t_j) \mathbf{1}_{d+1}  \right). \]

Here, $\mathbf{1}_{d+1} \in \bR^{d+1}$ is the vector of all $1$s. Note that $\tilde{p}_\NN^j \in \cF_\NN(d+1, d+1, 1+L_j, W_j)$, and for all $\bfv \in [0, 1]^d$ and $k \in [\sfT-1]$.

\[\tilde{p}^j_\NN((\bfv, t_k)) = \begin{cases} \left(p^k_\NN((\bfv, t_k)), t_{k+1} \right), & j = k, \\ \mathbf{0}_{d+1}, & j \neq k. \end{cases} \]

Finally, we define the propagator network for all $\bfz \in \bR^{d+1}$ as 
\[\sP_\NN(\bfz) = \sum_{j=1}^{\sfT} \tilde{p}^j_\NN(\bfz).\] 
Then for all $k \in [\sfT-1]$, we have that $\sP_\NN((\bfv, t_k)) = p^k_\NN(\bfz)$. In addition, we can see that for all $j \in [\sfT-1]$, we have
\begin{align*} 
\left(\bigcirc_{j=1}^k {\sP_\NN}\right) (\bfv, t_1) 
&= \left(\bigcirc_{j=1}^{k-1}{\sP_\NN}\right)(\tilde{p}_\NN^1(\bfv, t_1)) \\
&= \left(\bigcirc_{j=1}^{k-1}{\sP_\NN}\right)(p_\NN^1(\bfv), t_2) \\
&= \left(\bigcirc_{j=1}^{k-2}{\sP_\NN}\right)(\tilde{p}_\NN^2(p_\NN^1(\bfv), t_2)) \\
&= \left(\bigcirc_{j=1}^{k-2}{\sP_\NN}\right)(p_\NN^2(p_\NN^1(\bfv)), t_3) \\
&= \cdots \\
&= p_\NN^k(p_\NN^{k-1}(\cdots(p_\NN^1(\bfv))), t_{k+1}).
\end{align*}

We finish by computing the size of $\sP_\NN$. According to Proposition \ref{prop:appxtools} Part 3, the number of layers of $\sP_\NN$ is the sum of the number of layers of each $\tilde{p}_\NN$, so it is $\sum_{t=j}^{\sfT} \left(1 + O\left(\log\left(\frac{\sfT}{\epsilon} \right)\right)\right) = O\left(\sfT \log\left(\frac{ \sfT}{\epsilon}\right)\right)$, and the width is given by $d+1 + 2 \max_j O\left( \left(\frac{\sfT}{\epsilon}\right)^d \right) = O\left( \left(\frac{\sfT}{\epsilon}\right)^{d} \right)$.
\end{proof}

Now we prove Theorem \ref{thm:weldgeneral}. In the proof, we first derive a result analogous to Lemma \ref{lemma:weldlatentode}, and then apply the argument in Theorem \ref{thm:weldlatentode}.

\begin{proof}[Proof of Theorem \ref{thm:weldgeneral}]
We first assume that $\sfW = \sfS$, i.e. the windows are equal to the segments. This is done by following proof of Lemma \ref{lemma:weldlatentode}, but using Lemma \ref{lemma:propappx} instead of Lemma \ref{lemma:propdynamic} to construct $\sP_\NN$. We use the exact same construction as Lemma \ref{lemma:weldlatentode} for the autoencoders and transcoders.

For all $w \in [\sfW]$, consider the oracle propagator $\sP^w_*(x, t) = \sE^w_*(\sF(\sD_*^w(x), t))$. By Lemma \ref{lemma:propappx}, there is a neural network $\sP^w_\NN$ with $O\left(\mathbf{T}_w \log\left( \frac{\sfW \mathbf{T}_w }{\epsilon} \right)\right)$ layers and width $O\left(\left(\frac{\sfW \mathbf{T}_w }{\epsilon}\right)^d\right)$ such that 

\[ \| \bigcirc_{\ell=1}^{\bm{T}_w} \sP_\NN^w - \bigcirc_{\ell=1}^{\bm{T}_w} \sP^w_*(\cdot, t^w_{\ell+1} - t^w_{\ell}) \|_{L^\infty(\cZ( a_w))} < \frac{\epsilon}{(2\sfW + 1) {\rm Lip}(\sD) {\rm Lip}(\sF)}. \]

The composition above makes sense as the range of $\sP_\NN^w$ is contained in $\cZ([s_w, s_{w+1}])$ by construction. This establishes the result for the case that $\sfW = \sfS$. To handle the case that $\sfW > \sfS$, we can use the same argument as in Theorem \ref{thm:weldlatentode} to complete the proof.
\end{proof}

\section{Comparison Model Details}
\subsection{LDNet}
\label{app:ldnet}
We implemented LDNet which was proposed in \cite{regazzoni2024learning}, but we modified the architecture to be grid-dependent. Specifically, since WeldNet has inputs and outputs on a fixed grid, for a more direct comparison we implemented a grid-dependent LDNet such that the reconstruction network has to predict the values on the grid. The original LDNet implementation uses a reconstruction network that inputs the query location and outputs the value of the output field at that location. In other words, we would have to call the reconstruction network $D$ times to output values on a size $D$ grid. Training the grid based LDNet implementation is subsequently much faster and requires much lower computational resources (such as memory); the original LDNet implementation has significant memory requirements for our data. We found that the original model sizes used in \cite{regazzoni2024learning} were too small to perform well on our test problems, so we increased the network size to be comparable to the models in LDNet (e.g. width 500 networks), but this lead to a significant memory requirement.

We present a comparison between the original LDNet implementation and the grid based LDNet model. The original model was trained on an NVIDIA A100 GPU with 80 GBs of memory, while the grid based model was trained on an NVIDIA RTX6000 (i.e. we train with the same resources as WeldNet and other models). Due to memory limitations, we trained the original LDNet using data on a time grid that was \textbf{three} times as coarse as the time used for grid LDNet. With all of these changes, the total training time on the Burgers' scale dataset (with initial conditions from \eqref{eq:burgers}) for the original LDNet model is 1335.3s, while the grid LDNet took 1343.6s to train (on a dataset with 3x as many time steps). 

Table \ref{table:ldnetcompare} shows a comparison of the midway (i.e. middle of the time grid) and final time test errors between the original LDNet and the grid LDNet models. We used the tanh activation function for the reconstruction network for the Burgers' and KdV examples, but we used a ReLU activation for the transport examples. Superior performance of the original LDNet model is observed on the transport examples, but it is matched or outperformed with the grid LDNet. Recall that the original LDNet is trained to predict the evolution of 101 time steps but the grid LDNet is trained to predict 301 time steps. Clearly, a grid based LDNet performs roughly similarly or better to the original LDNet implementation for most of our problems using much less computational resources. We note that the original LDNet implementation has the advantage of being grid independent, so the output predictions can be queried at all locations regardless of the grid it was trained on. This is also advantageous for some datasets such as for the two-hat-shaped initial conditions we used for the transport equation. However, for our purpose of grid based surrogate modeling, a grid based LDNet is faster to train and performs better on lower resources (and on a more fine time grid), so we use that for comparison in the rest of the paper.

\begin{table}[]
\centering
\begin{tabular}{|c|c|c|}
\hline
       & Original-LDNet                            & Grid-LDNet      \\ \hline
bscale & 3.81\% / 5.12\%                           & 0.78\% / 1.36\% \\ \hline
bshift & 6.13\% / 6.87\%                           & 10.4\% / 6.38\% \\ \hline
tscale & 4.12\% / 4.90\%                           & 4.72\% / 8.29\% \\ \hline
tshift & 0.75\% / 0.89\%                           & 5.98\% / 10.3\% \\ \hline
kscale & \textgreater{}100\% / \textgreater{}100\% & 11.1\% / 22.7\% \\ \hline
kshift & 0.95\% / 1.89\%                           & 0.43\% / 0.72\% \\ \hline
\end{tabular}
\label{table:ldnetcompare}
\caption{Comparison of middle and final time relative test errors for original LDNet (with 101 time steps) and grid LDNet (with 301 time steps). Each cell is formatted as MiddleError / FinalError.}
\end{table}

\subsection{DeepONet}
\label{app:don}
Latent-DON consists of an autoencoder (which is identical to the autoencoder used in a one-window WeldNet model) and a DeepONet that predicts the evolution of a latent code from its initial time to a given future time. DeepONet is an operator learning architecture \cite{deeponet} that can be used in this case as time can be considered as the ``input'' to the operator. Specifically, a DeepONet is a function of the form (for some $p \in \bN$):
\begin{equation}
\label{eq:don}
\phi(\vec{z})(t) = \tau_\NN(t) \cdot B_\NN(\vec{z}) 
\end{equation}
Here, $B_\NN$ is a neural network with input dimension $k$ (i.e. the latent space dimension) and output dimension $pk$ which is then reshaped to be dimension $p \times k$. On the other hand, $\tau_\NN$ is a neural network with input dimension $1$ and output dimension $p$, so the product in \eqref{eq:don} outputs a vector of dimension $k$. 

This is the formulation used in \cite{kontolati2024learning}. We implement it in Pytorch such that all components of latent DeepONet (i.e. encoder, decoder, branch net, trunk net) are width 400 and depth 3 feedforward ReLU networks. We use $p=10$ for the Latent-DON comparison in this work.

\section{Error Tables}
\label{app:errortables}
For each dataset, we compute the relative test errors at different times and show them in Table \ref{table:bscale}-\ref{table:kshift}. The best model error in the final time is bolded. A dash ``--'' indicates a relative test over of 10 or higher.

We also display the total training time for each model on the Burgers' scale dataset in Table \ref{table:timesbscaleweld} for WeldNet models and Table \ref{table:timesbscaleall} for comparison models. The total training times for other models are similar. We use NVIDIA RTX 6000 GPUs. For WeldNet-2 and WeldNet-4, we use one GPU per window since the training of each window's models is completely independent from the other windows, but we use only one GPU for other models (there is no easy way to distribute the training among multiple GPUs for those models).

\begin{table}[]
\centering
\begin{tabular}{|c|c|c|c|c|c|}
\hline
       & Weld-1 & Weld-2 & Weld-2$^*$ & Weld-4 & Weld-4$^*$ \\ \hline
Time & 668.5s & 547.7s & 373.2s & 528.1s & 267.6s  \\ \hline
\end{tabular}
\label{table:timesbscaleweld}
\caption{Total training time for each model on $\cM_{bscale}$ data for WeldNet models. $*$ indicates training with multiple GPUs.}
\end{table}

\begin{table}[]
\centering
\begin{tabular}{|c|c|c|c|c|c|}
\hline
       & LDON & LDNet   & Time Input & HDP   & WLaSDI \\ \hline
Time & 65.2s & 1355.3s & 52.0s      & 56.1s & 2285s \\ \hline
\end{tabular}
\label{table:timesbscaleall}
\caption{Total training time for each model on $\cM_{bscale}$ data for comparison models.}
\end{table}

\begin{table}[]
\scriptsize\begin{tabular}{|c|c|c|c|c|c|c|c|c|c|c|}
\hline
Model  & 30 & 60 & 90 & 120 & 150 & 180 & 210 & 240 & 270 & 300 \\ \hline
FF-Weld-1 & 0.0091 & 0.0095 & 0.0111 & 0.0133 & 0.0145 & 0.0178 & 0.0220 & 0.0239 & 0.0241 & 0.0343 \\\hline
FF-Weld-2 & 0.0068 & 0.0045 & 0.0050 & 0.0053 & 0.0058 & 0.0122 & 0.0088 & 0.0089 & 0.0100 & 0.0101 \\ \hline
FF-Weld-4 & 0.0063 & 0.0071 & 0.0059 & 0.0071 & 0.0175 & 0.0073 & 0.0081 & 0.0082 & 0.0084 & 0.0137 \\ \hline
PCA-WeldNet-1 & 0.1875 & 0.1177 & 0.1078 & 0.1176 & 0.1307 & 0.1561 & 0.1798 & 0.1990 & 0.2174 & 0.2404 \\ \hline
PCA-WeldNet-2 & 0.1077 & 0.0420 & 0.0656 & 0.0578 & 0.0909 & 0.1675 & 0.1408 & 0.1204 & 0.1220 & 0.1551 \\ \hline
PCA-WeldNet-4 & 0.0445 & 0.0254 & 0.0292 & 0.0524 & 0.0440 & 0.1145 & 0.0885 & 0.1223 & 0.1213 & 0.1133 \\ \hline
Conv-Weld-1 & 0.0099 & 0.0070 & 0.0064 & 0.0065 & 0.0074 & 0.0072 & 0.0076 & 0.0076 & 0.0080 & 0.0091 \\ \hline
Conv-Weld-2 & 0.0090 & 0.0056 & 0.0055 & 0.0064 & 0.0066 & 0.0121 & 0.0082 & 0.0084 & 0.0090 & 0.0094 \\ \hline
Conv-Weld-4 & 0.0159 & 0.0119 & 0.0101 & 0.0067 & 0.0069 & 0.0089 & 0.0063 & 0.0074 & 0.0057 & \textbf{0.0060} \\ \hline
Latent-DON & 0.0059 & 0.0065 & 0.0069 & 0.0072 & 0.0078 & 0.0083 & 0.0092 & 0.0091 & 0.0092 & 0.0153 \\ \hline
LDNet & 0.0440 & 0.0373 & 0.0370 & 0.0375 & 0.0381 & 0.0374 & 0.0369 & 0.0360 & 0.0388 & 0.0512 \\ \hline
Grid-LDNet & 0.0063 & 0.0070 & 0.0072 & 0.0071 & 0.0078 & 0.0083 & 0.0097 & 0.0097 & 0.0095 & 0.0136 \\ \hline
TimeInput & 0.0258 & 0.0135 & 0.0116 & 0.0126 & 0.0123 & 0.0126 & 0.0145 & 0.0159 & 0.0178 & 0.0176 \\ \hline
HDP &  0.0209 & 0.0313 & 0.0403 & 0.0553 & 0.0825 & 0.0967 & 0.1116 & 0.1313 & 0.1323 & 0.1433 \\  \hline
WLaSDI & 1.045 & 1.042 & 1.055 & 1.072 & 1.097 & 1.137 & 1.192 & 1.263 & 1.359 & 1.486 \\ \hline
\end{tabular}
\caption{Test Errors for Burgers' Scale  about the trajectory manifold of the Burgers' equation \eqref{eq:burgers} with initial conditions in \eqref{eq:bscaleinitial}.}
\label{table:bscale}
\end{table}

\begin{table}[]
\scriptsize\begin{tabular}{|c|c|c|c|c|c|c|c|c|c|c|}
\hline
Model  & 30 & 60 & 90 & 120 & 150 & 180 & 210 & 240 & 270 & 300 \\ \hline
FF-Weld-2 & 0.0232 & 0.0181 & 0.0246 & 0.0427 & 0.0455 & 0.0484 & 0.0363 & 0.0287 & 0.0247 & 0.0242 \\ \hline
FF-Weld-4 & 0.0151 & 0.0364 & 0.0570 & 0.0669 & 0.0623 & 0.0557 & 0.0492 & 0.0536 & 0.0468 & 0.0472 \\ \hline
FF-AENet &  0.0218 & 0.0310 & 0.0468 & 0.0542 & 0.0444 & 0.0364 & 0.0312 & 0.0284 & 0.0289 & 0.0340 \\ \hline
Conv-Weld-2 & 0.0245 & 0.0147 & 0.0170 & 0.0202 & 0.0205 & 0.0252 & 0.0178 & 0.0155 & 0.0171 & \textbf{0.0201} \\ \hline
Conv-Weld-4 & 0.0199 & 0.0116 & 0.0153 & 0.0237 & 0.0278 & 0.0295 & 0.0229 & 0.0253 & 0.0227 & 0.0205 \\ \hline
Conv-AENet & 0.0328 & 0.0204 & 0.0291 & 0.0442 & 0.0554 & 0.0627 & 0.0738 & 0.0867 & 0.0998 & 0.1119 \\ \hline
PCA-WeldNet-2 & 0.1500 & 0.1628 & 0.3493 & 0.4804 & 0.4964 & 0.4950 & 0.4884 & 0.4806 & 0.4727 & 0.4651 \\ \hline
PCA-WeldNet-4 & 0.1484 & 0.1568 & 0.3466 & 0.4830 & 0.4980 & 0.4949 & 0.4882 & 0.4805 & 0.4736 & 0.4719 \\ \hline
PCA-AENet & 0.1521 & 0.2058 & 0.3729 & 0.4967 & 0.5127 & 0.5107 & 0.5060 & 0.5009 & 0.4983 & 0.5021 \\ \hline
Latent-DON & 0.0209 & 0.0291 & 0.0400 & 0.0460 & 0.0386 & 0.0320 & 0.0272 & 0.0248 & 0.0261 & 0.0318 \\ \hline
LDNet & 0.0254 & 0.0399 & 0.1043 & 0.1158 & 0.1041 & 0.0919 & 0.0811 & 0.0722 & 0.0657 & 0.0638 \\ \hline
TimeInput & 0.0291 & 0.0222 & 0.0323 & 0.0548 & 0.0550 & 0.0451 & 0.0366 & 0.0303 & 0.0267 & 0.0265 \\ \hline
HDP &  0.2155 & 0.4972 & 0.8060 & 1.2015 & 1.7128 & 2.3725 & 3.2132 & 4.2975 & 5.6602 & 7.3479 \\  \hline
WLaSDI & 0.782 & 0.839 & 0.854 & 0.835 & 0.844 & 0.910 & 0.978 & 1.045 & 1.199 & 1.380 \\ \hline
\end{tabular}
\caption{Test Errors for Burgers' Shift  about the trajectory manifold of the Burgers' equation \eqref{eq:burgers} with initial conditions in \eqref{eq:bshiftinitial}.}
\label{table:bshift}
\end{table}

\begin{table}[]
\scriptsize\begin{tabular}{|c|c|c|c|c|c|c|c|c|c|c|}
\hline
Model  & 30 & 60 & 90 & 120 & 150 & 180 & 210 & 240 & 270 & 300 \\ \hline
FF-Weld-2 & 0.0095 & 0.0125 & 0.0103 & 0.0139 & 0.0120 & 0.0085 & 0.0140 & 0.0118 & 0.0113 & 0.0142 \\ \hline
FF-Weld-4 & 0.0136 & 0.0112 & 0.0116 & 0.0141 & 0.0125 & 0.0160 & 0.0135 & 0.0094 & 0.0141 & \textbf{0.0102} \\ \hline
FF-AENet & 0.0215 & 0.0204 & 0.0223 & 0.0196 & 0.0195 & 0.0224 & 0.0245 & 0.0217 & 0.0249 & 0.0362 \\ \hline
Conv-Weld-2 & 0.0221 & 0.0234 & 0.0235 & 0.0248 & 0.0181 & 0.0198 & 0.0248 & 0.0250 & 0.0225 & 0.0268 \\ \hline
Conv-Weld-4 & 0.0168 & 0.0163 & 0.0155 & 0.0212 & 0.0202 & 0.0248 & 0.0311 & 0.0292 & 0.0190 & 0.0203 \\ \hline
Conv-AENet & 0.0298 & 0.0427 & 0.0346 & 0.0516 & 0.0445 & 0.0439 & 0.0271 & 0.0324 & 0.0300 & 0.0234 \\ \hline
PCA-WeldNet-2 & 0.5815 & 0.4339 & 0.4880 & 0.5128 & 0.5090 & 0.5813 & 0.4336 & 0.4875 & 0.5101 & 0.5010 \\ \hline
PCA-WeldNet-4 & 0.2867 & 0.1783 & 0.2379 & 0.2582 & 0.1697 & 0.2860 & 0.1790 & 0.2388 & 0.2571 & 0.1705 \\ \hline
PCA-AENet & 0.8214 & 0.5968 & 0.5885 & 0.6678 & 0.6158 & 0.6106 & 0.6205 & 0.6837 & 0.6580 & 0.5872 \\ \hline
Latent-DON & 0.0156 & 0.0155 & 0.0218 & 0.0217 & 0.0165 & 0.0193 & 0.0174 & 0.0204 & 0.0207 & 0.0318 \\ \hline
LDNet & 0.0359 & 0.0416 & 0.0492 & 0.0459 & 0.0472 & 0.0476 & 0.0453 & 0.0504 & 0.0463 & 0.0829 \\ \hline
TimeInput & 0.0580 & 0.0668 & 0.0591 & 0.0669 & 0.0679 & 0.0610 & 0.0640 & 0.0703 & 0.0661 & 0.0921 \\ \hline
HDP &  0.0189 & 0.0348 & 0.0767 & 0.1062 & 0.1366 & 0.1480 & 0.1629 & 0.1946 & 0.1881 & 0.0697 \\  \hline
WLaSDI & 1.210 & 1.671 & 1.521 & 3.540 & 9.388 & 5.028 & 7.136 & -- & -- & -- \\ \hline
\end{tabular}
\caption{Test Errors for Transport Scale  about the trajectory manifold of the transport equation \eqref{eq:transport1d} with initial conditions in \eqref{eq:tscale}.}
\label{table:tscale} 
\end{table}

\begin{table}[]
\scriptsize\begin{tabular}{|c|c|c|c|c|c|c|c|c|c|c|}
\hline
Model  & 30 & 60 & 90 & 120 & 150 & 180 & 210 & 240 & 270 & 300 \\ \hline
FF-Weld-2 & 0.0283 & 0.0295 & 0.0276 & 0.0308 & 0.0305 & 0.0405 & 0.1727 & 0.1273 & 0.0631 & 0.0599 \\ \hline
FF-Weld-4 & 0.0310 & 0.0297 & 0.0353 & 0.0372 & 0.0432 & 0.0357 & 0.0374 & 0.0457 & 0.0559 & 0.0587 \\ \hline
FF-AENet & 0.7640 & 0.8940 & 0.9187 & 0.9386 & 0.9492 & 0.9243 & 0.8009 & 0.8958 & 0.8003 & 1.1303 \\ \hline
Conv-Weld-2 & 0.0207 & 0.0186 & 0.0168 & 0.0189 & 0.0187 & 0.0208 & 0.0213 & 0.0221 & 0.0188 & 0.0210 \\ \hline
Conv-Weld-4 & 0.0173 & 0.0151 & 0.0154 & 0.0186 & 0.0183 & 0.0180 & 0.0213 & 0.0231 & 0.0200 & 0.0227 \\ \hline
Conv-AENet & 0.0312 & 0.0279 & 0.0330 & 0.0293 & 0.0276 & 0.0275 & 0.0254 & 0.0320 & 0.0321 & 0.0319 \\ \hline
PCA-WeldNet-2 & 0.7819 & 0.7177 & 0.6850 & 0.6988 & 0.7210 & 0.7889 & 0.7401 & 0.6998 & 0.7150 & 0.7326 \\ \hline
PCA-WeldNet-4 & 0.6773 & 0.6055 & 0.5905 & 0.6087 & 0.6159 & 0.6777 & 0.6055 & 0.5906 & 0.6095 & 0.6176 \\ \hline
PCA-AENet & 0.8650 & 0.8727 & 0.8775 & 0.9071 & 0.8685 & 0.8806 & 0.8619 & 0.8546 & 0.8281 & 0.8190 \\ \hline
Latent-DON & 0.1120 & 0.1088 & 0.1519 & 0.1396 & 0.1161 & 0.1438 & 0.1252 & 0.1412 & 0.1473 & 0.3016 \\ \hline
LDNet & 0.0503 & 0.0521 & 0.0565 & 0.0589 & 0.0598 & 0.0616 & 0.0609 & 0.0618 & 0.0664 & 0.1029 \\ \hline
TimeInput & 0.1435 & 0.1065 & 0.0966 & 0.1035 & 0.1085 & 0.1059 & 0.1188 & 0.1048 & 0.0955 & 0.1031 \\ \hline
HDP &  0.3034 & 2.3470 & 5.1519 & 8.4939 & -- & -- & -- & -- & -- & -- \\  \hline
WLaSDI & 6.216 & -- & -- & -- & -- & -- & -- & -- & -- & -- \\ \hline
\end{tabular}
\caption{Test Errors for Transport Shift about the trajectory manifold of the transport equation \eqref{eq:transport1d} with initial conditions in \eqref{eq:tshiftinitial}.}
\label{table:tshift}
\end{table}

\begin{table}[]
\scriptsize\begin{tabular}{|c|c|c|c|c|c|c|c|c|c|c|}
\hline
Model  & 30 & 60 & 90 & 120 & 150 & 180 & 210 & 240 & 270 & 300 \\ \hline
FF-Weld-2 & 0.0160 & 0.0249 & 0.0329 & 0.0273 & 0.0281 & 0.0277 & 0.0366 & 0.0332 & 0.0338 & 0.0371 \\ \hline
FF-Weld-4 & 0.0096 & 0.0101 & 0.0127 & 0.0128 & 0.0240 & 0.0166 & 0.0187 & 0.0216 & 0.0272 & 0.0395 \\ \hline
FF-AENet & 0.0795 & 0.1304 & 0.1255 & 0.1307 & 0.1696 & 0.1798 & 0.2025 & 0.2515 & 0.4120 & 0.6350  \\ \hline
Conv-Weld-2 & 0.0261 & 0.0231 & 0.0225 & 0.0235 & 0.0240 & 0.0232 & 0.0288 & 0.0283 & 0.0308 & 0.0316 \\ \hline
Conv-Weld-4 & 0.0106 & 0.0112 & 0.0128 & 0.0173 & 0.0186 & 0.0227 & 0.0206 & 0.0234 & 0.0283 & \textbf{0.0317} \\ \hline
Conv-AENet & 0.0187 & 0.0514 & 0.0864 & 0.1098 & 0.1423 & 0.1818 & 0.2392 & 0.2971 & 0.3628 & 0.3851 \\ \hline
PCA-WeldNet-2 & 0.5672 & 0.3454 & 0.3709 & 0.3762 & 0.4345 & 0.6621 & 0.6137 & 0.5912 & 0.5849 & 0.6024 \\ \hline
PCA-WeldNet-4 & 0.3723 & 0.2473 & 0.2717 & 0.5046 & 0.4217 & 0.5999 & 0.5589 & 0.5352 & 0.6763 & 0.6782 \\ \hline
PCA-AENet & 0.6709 & 0.4938 & 0.4819 & 0.4796 & 0.4769 & 0.4968 & 0.5399 & 0.5860 & 0.6234 & 0.6611 \\ \hline
Latent-DON & 0.2379 & 0.1528 & 0.1858 & 0.1949 & 0.1685 & 0.1378 & 0.1225 & 0.1438 & 0.1914 & 0.3889 \\ \hline
LDNet & 0.0735 & 0.0793 & 0.0890 & 0.1011 & 0.1105 & 0.1262 & 0.1393 & 0.1550 & 0.1734 & 0.2272 \\ \hline
TimeInput & 0.1673 & 0.1505 & 0.1362 & 0.1506 & 0.1570 & 0.1700 & 0.1767 & 0.1878 & 0.2060 & 0.2472 \\ \hline
HDP & 0.0770 & 0.1302 & 0.1688 & 0.2621 & 0.3905 & 0.6109 & 0.8120 & 0.9843 & 1.0881 & 1.1590 \\  \hline
WLaSDI &  0.993 & 1.417 & 2.495 & 5.950 & -- & -- & -- & -- & -- & -- \\ \hline
\end{tabular}
\caption{Test Errors for KdV Scale about the trajectory manifold of the KdV equation \eqref{eq:kdv} with initial conditions in \eqref{eq:kscaleinitial}.}
\label{table:kscale}
\end{table}

\begin{table}[]
\scriptsize\begin{tabular}{|c|c|c|c|c|c|c|c|c|c|c|}
\hline
Model  & 30 & 60 & 90 & 120 & 150 & 180 & 210 & 240 & 270 & 300 \\ \hline
FF-Weld-2 & 0.0028 & 0.0027 & 0.0028 & 0.0025 & 0.0027 & 0.0031 & 0.0030 & 0.0030 & 0.0030 & \textbf{0.0028} \\ \hline
FF-Weld-4 & 0.0026 & 0.0028 & 0.0028 & 0.0027 & 0.0041 & 0.0030 & 0.0032 & 0.0033 & 0.0032 & 0.0046 \\ \hline
FF-AENet & 0.0071 & 0.0069 & 0.0065 & 0.0069 & 0.0063 & 0.0061 & 0.0058 & 0.0051 & 0.0053 & 0.0087 \\ \hline
Conv-Weld-2 & 0.0044 & 0.0034 & 0.0036 & 0.0032 & 0.0036 & 0.0042 & 0.0034 & 0.0031 & 0.0031 & 0.0032 \\ \hline
Conv-Weld-4 & 0.0040 & 0.0030 & 0.0031 & 0.0034 & 0.0031 & 0.0038 & 0.0030 & 0.0030 & 0.0028 & \textbf{0.0028} \\ \hline
Conv-AENet & 0.0050 & 0.0053 & 0.0046 & 0.0042 & 0.0046 & 0.0042 & 0.0039 & 0.0046 & 0.0047 & 0.0053 \\ \hline
PCA-WeldNet-2 & 0.0444 & 0.0233 & 0.0192 & 0.0215 & 0.0230 & 0.0482 & 0.0210 & 0.0177 & 0.0233 & 0.0186 \\ \hline
PCA-WeldNet-4 & 0.0199 & 0.0119 & 0.0055 & 0.0051 & 0.0098 & 0.0168 & 0.0105 & 0.0057 & 0.0059 & 0.0090 \\ \hline
PCA-AENet & 0.1276 & 0.0832 & 0.0672 & 0.0551 & 0.0476 & 0.0542 & 0.0629 & 0.0587 & 0.0473 & 0.0628 \\ \hline
Latent-DON & 0.0047 & 0.0039 & 0.0044 & 0.0044 & 0.0035 & 0.0039 & 0.0038 & 0.0032 & 0.0032 & 0.0051 \\ \hline
LDNet & 0.0049 & 0.0041 & 0.0044 & 0.0045 & 0.0043 & 0.0043 & 0.0042 & 0.0038 & 0.0037 & 0.0072 \\ \hline
TimeInput & 0.0130 & 0.0071 & 0.0081 & 0.0099 & 0.0067 & 0.0074 & 0.0074 & 0.0064 & 0.0073 & 0.0071 \\ \hline
HDP & 0.0115 & 0.0226 & 0.0212 & 0.0253 & 0.0273 & 0.0321 & 0.0319 & 0.0332 & 0.0360 & 0.0426 \\  \hline
WLaSDI & 0.464 & 0.463 & 0.476 & 0.499 & 0.521 & 0.520 & 0.523 & 0.520 & 0.505 & 0.497  \\
\hline
\end{tabular}
\caption{Test Errors for KdV Shift about the trajectory manifold of the KdV equation \eqref{eq:kdv} with initial conditions in \eqref{eq:kshiftinitial}.}
\label{table:kshift}
\end{table}

\begin{table}[h]
\scriptsize\begin{tabular}{|c|c|c|c|c|c|c|c|c|c|c|}
\hline
Model  & 10 & 20 & 30 & 40 & 50 & 60 & 70 & 80 & 90 & 100 \\ \hline
FF-Weld-2 & 0.0051 & 0.0048 & 0.0043 & 0.0042 & 0.0041 & 0.0030 & 0.0030 & 0.0030 & 0.0032 & 0.0034 \\ \hline
FF-Weld-4 & 0.0051 & 0.0045 & 0.0037 & 0.0034 & 0.0031 & 0.0027 & 0.0026 & 0.0027 & 0.0026 & 0.0028 \\ \hline
FF-AENet & 0.0064 & 0.0056 & 0.0051 & 0.0051 & 0.0045 & 0.0047 & 0.0045 & 0.0042 & 0.0036 & 0.0043 \\ \hline
Latent-DON & 0.0080 & 0.0063 & 0.0052 & 0.0046 & 0.0044 & 0.0047 & 0.0047 & 0.0049 & 0.0045 & 0.0061 \\ \hline
LDNet & 0.0086 & 0.0074 & 0.0072 & 0.0072 & 0.0066 & 0.0065 & 0.0063 & 0.0063 & 0.0073 & 0.0095 \\ \hline
TimeInput & 0.0155 & 0.0129 & 0.0121 & 0.0120 & 0.0118 & 0.0119 & 0.0121 & 0.0115 & 0.0140 & 0.0194 \\ \hline
HDP & 0.0389 & 0.0312 & 0.0298 & 0.0294 & 0.0274 & 0.0262 & 0.0267 & 0.0269 & 0.0250 & 0.0316 \\ \hline
\end{tabular}
\caption{Test Errors for Shallow Water about the trajectory manifold of the KdV equation \eqref{eq:shallow} with initial conditions in \eqref{eq:shallowinitial}.}
\label{table:shallow}
\end{table}

\end{document}